%% file: arxiv-main.tex
\documentclass[11pt]{article}

\usepackage[margin=1in]{geometry}
\usepackage[T1]{fontenc}
\usepackage{amsmath,amsthm}
\usepackage{newtxtext}
\usepackage{newtxmath}
\usepackage{mathtools}
\usepackage{bm}
\usepackage{booktabs}
\usepackage{algorithm}
\usepackage{algpseudocode}
\usepackage{bbm}
\usepackage{accents}
\usepackage[shortlabels]{enumitem}
\usepackage{etoolbox}
\usepackage{graphicx}
\usepackage{tikz}
\usetikzlibrary{positioning}
\usepackage{url}
\usepackage{authblk}

\usepackage[round,authoryear]{natbib}
\usepackage[colorlinks=true,citecolor=blue,linkcolor=blue,urlcolor=blue]{hyperref}

\newbool{anonymous}
\boolfalse{anonymous}

\newcommand{\ECname}{later appendices}
\newcommand{\ofEC}{}
\newcommand{\inEC}{}

\theoremstyle{plain}
\newtheorem{theorem}{Theorem}
\newtheorem{lemma}{Lemma}
\newtheorem{corollary}{Corollary}
\newtheorem{proposition}{Proposition}

\newtheorem{assumption}{Assumption}
\theoremstyle{definition}
\newtheorem{definition}{Definition}
\newtheorem{example}{Example}
\theoremstyle{remark}
\newtheorem{remark}{Remark}
\theoremstyle{plain}

\newcommand{\indep}{{\perp}\!\!\! \perp}
\DeclareMathOperator{\E}{\mathbb{E}}
\DeclareMathOperator{\R}{\mathbb{R}}

\DeclareMathOperator{\N}{\mathcal{N}}
\DeclareMathOperator{\Prob}{\mathbb{P}}

\DeclareMathOperator{\C}{\mathcal{C}}

\DeclareMathOperator*{\argmin}{arg\,min}

\newcommand{\Var}{\mathrm{Var}}
\newcommand{\Cov}{\mathrm{Cov}}

\newcommand{\Hrow}{\mathcal{H}_{\mathrm{row}}}
\newcommand{\Hcol}{\mathcal{H}_{\mathrm{col}}}
\newcommand{\distThreshold}{\eta}
\newcommand{\latRow}[1]{x_{\mathrm{row}}^{(#1)}}
\newcommand{\latCol}[1]{x_{\mathrm{col}}^{(#1)}}
\newcommand{\nn}[2]{\N_{#2,#1}}
\newcommand{\sharedcol}[2]{\C_{#1#2}}
\newcommand{\rowdist}[2]{\rho_{#1#2}}

\newcommand{\distnn}{\textsc{Dist-NN}}
\newcommand{\numrows}{N}
\newcommand{\numcols}{M}

\newcommand{\Unif}{\mathrm{Unif}}

\NewDocumentCommand{\expect}{ e{_} s o >{\SplitArgument{1}{|}}m }{%
  \operatorname{\E}%
  \IfValueT{#1}{{\!}_{#1}}%
  \IfBooleanTF{#2}{%
    \expectarg*{\expectvar#4}%
  }{%
    \IfNoValueTF{#3}{%
      \expectarg{\expectvar#4}%
    }{%
      \expectarg[#3]{\expectvar#4}%
    }%
  }%
}
\NewDocumentCommand{\expectvar}{mm}{%
  #1\IfValueT{#2}{\nonscript\;\delimsize\vert\nonscript\;#2}%
}
\DeclarePairedDelimiterX{\expectarg}[1]{[}{]}{#1}

\NewDocumentCommand{\probc}{ e{_} s o >{\SplitArgument{1}{|}}m }{%
  \operatorname{\Prob}%
  \IfValueT{#1}{{\!}_{#1}}%
  \IfBooleanTF{#2}{%
    \probarg*{\probvar#4}%
  }{%
    \IfNoValueTF{#3}{%
      \probarg{\probvar#4}%
    }{%
      \probarg[#3]{\probvar#4}%
    }%
  }%
}
\NewDocumentCommand{\probvar}{mm}{%
  #1\IfValueT{#2}{\nonscript\;\delimsize\vert\nonscript\;#2}%
}
\DeclarePairedDelimiterX{\probarg}[1]{(}{)}{#1}

\makeatletter
\let\nn@origNewtheorem\newtheorem
\let\nn@origNewcommand\newcommand
\renewcommand{\newtheorem}[2]{\@ifundefined{#1}{\nn@origNewtheorem{#1}{#2}}{}}
\let\newcommand\providecommand
\makeatother

\input{jfeit-macros}
\fi\fi

\makeatletter
\let\newtheorem\nn@origNewtheorem
\let\newcommand\nn@origNewcommand
\makeatother

\newtheorem{customthm}{Theorem}
\newenvironment{namedthm}[1]{%
  \renewcommand\thecustomthm{#1}%
  \begin{customthm}}{\end{customthm}}

\usepackage{cleveref}
\usepackage{autonum}

\crefname{appendix}{App.}{Apps.}
\crefname{equation}{}{}
\crefname{lemma}{Lem.}{Lems.}
\crefname{claim}{Claim}{Claims}
\crefname{theorem}{Thm.}{Thms.}
\crefname{corollary}{Cor.}{Cors.}
\crefname{proposition}{Prop.}{Props.}
\crefname{algorithm}{Alg.}{Algs.}
\crefname{example}{Ex.}{Exs.}
\crefname{section}{Sec.}{Secs.}
\crefname{subsection}{Sec.}{Secs.}
\crefname{subsubsection}{Sec.}{Secs.}
\crefname{subappendix}{App.}{Apps.}
\crefname{subsubappendix}{App.}{Apps.}
\crefname{table}{Tab.}{Tabs.}
\crefname{remark}{Rem.}{Rems.}
\crefname{customthm}{Thm.}{Thms.}
\crefname{definition}{Def.}{Defs.}
\crefname{assumption}{Assum.}{Assums.}
\crefname{figure}{Fig.}{Figs.}

\newlist{assumptionconds}{enumerate}{1}
\setlist[assumptionconds]{label=(\alph*), ref=\theassumption(\alph*), leftmargin=2em, itemsep=3pt, topsep=3pt}
\crefname{assumptioncondsi}{Assum.}{Assums.}
\Crefname{assumptioncondsi}{Assum.}{Assums.}

\title{Distributional Matrix Completion via Nearest Neighbors\\ in the Wasserstein Space}

\author[1]{Jacob Feitelberg\thanks{Corresponding author: \texttt{jef2182@columbia.edu}}}
\author[2]{Kyuseong Choi}
\author[1]{Anish Agarwal}
\author[3]{Raaz Dwivedi}
\affil[1]{Department of Industrial Engineering and Operations Research, Columbia University}
\affil[2]{Department of Statistics and Data Science, Cornell Tech, Cornell University}
\affil[3]{Department of Operations Research and Information Engineering, Cornell Tech, Cornell University}

\date{}

\begin{document}

\maketitle

\begin{abstract}
Operational decisions often depend on the entire distribution of an uncertain quantity with variability and tails being frequently decisive. In intensive-care monitoring, for example, how long a patient's blood pressure stays below a safe threshold can matter more than its average. Matrix completion is a standard tool for filling in missing data, used both to predict unobserved values and to reconstruct partially observed records. Existing methods, however, handle only scalar entries, and so discard the variability and tails a decision needs. We study distributional matrix completion, in which each entry is a probability distribution observed through a small sample, and introduce Dist-NN, which imputes an entire distribution by averaging nearby distributions through tools from optimal transport. Because Dist-NN returns a full distribution, a single imputation yields any summary a decision requires, from a quantile to the frequency of crossing a threshold, along with confidence bands for that summary. We also prove convergence guarantees as well as limiting distributions. We validate Dist-NN on synthetic data and apply it to real intensive-care records, showing that Dist-NN estimates distributional quantities much more accurately than the imputation supplied with the dataset.
\end{abstract}


\input{introduction-hirid}
\input{setup}
\input{barycenters}
\input{nn}
\input{sim}
\input{empirical-hirid}
\input{discussion}

\section*{Acknowledgments}
Jacob Feitelberg and Anish Agarwal's work on this paper was supported by the Columbia Center for AI and Responsible Financial Innovation in collaboration with Capital One.

\bibliographystyle{plainnat}
\bibliography{references}

\appendix
\input{appendix-proof-techniques.tex}
\input{appendix-proofs.tex}

\input{appendix-estimators.tex}

\input{appendix-confidence-bands.tex}

\end{document}

%% file: jfeit-macros.tex
\newcommand{\order}{\mc{O}}

\newcommand{\sless}[1]{\stackrel{#1}{\leq}}

\newcommand{\seq}[1]{\stackrel{#1}{=}}

\newcommand{\x}{x}

\newcommand{\axi}[1][i]{\x_{#1}}

\newcommand{\lnorm}[1]{\bigg\Vert{#1}\bigg\Vert_{L^2(0,1)}}
\newcommand{\quantile}[1]{F^{-1}_{#1}}

\newcommand{\pseqxn}[1][n]{(\axi[i])_{i\geq 1}} % sequence with parentheses
\newcommand{\pseqxnn}[1][n]{(\axi[i])_{i=1}^n} % sequence with parentheses

\newcommand{\parenth}[1]{\left( #1 \right)}

\newcommand{\sbraces}[1]{\left\{ #1  \right\}}

\newcommand{\abss}[1]{\left\lvert #1 \right\rvert}

\newcommand{\real}{\ensuremath{\mathbb{R}}}

\def\balign#1\ealign{\begin{align}#1\end{align}}
\def\baligns#1\ealigns{\begin{align*}#1\end{align*}}
\def\balignat#1\ealign{\begin{alignat}#1\end{alignat}}
\def\balignats#1\ealigns{\begin{alignat*}#1\end{alignat*}}
\def\bitemize#1\eitemize{\begin{itemize}#1\end{itemize}}
\def\benumerate#1\eenumerate{\begin{enumerate}#1\end{enumerate}}

\newenvironment{talign*}
 {\csname align*\endcsname}
 {\endalign}
\newenvironment{talign}
 {\csname align\endcsname}
 {\endalign}

\def\balignst#1\ealignst{\begin{talign*}#1\end{talign*}}
\def\balignt#1\ealignt{\begin{talign}#1\end{talign}}
\newcommand{\qtext}[1]{\quad\text{#1}\quad}

\let\originalleft\left
\let\originalright\right
\renewcommand{\left}{\mathopen{}\mathclose\bgroup\originalleft}
\renewcommand{\right}{\aftergroup\egroup\originalright}

\def\tinycitep*#1{{\tiny\citep*{#1}}}
\def\tinycitealt*#1{{\tiny\citealt*{#1}}}
\def\tinycite*#1{{\tiny\cite*{#1}}}
\def\smallcitep*#1{{\scriptsize\citep*{#1}}}
\def\smallcitealt*#1{{\scriptsize\citealt*{#1}}}
\def\smallcite*#1{{\scriptsize\cite*{#1}}}

\def\mbb#1{\mathbb{#1}}
\def\mc#1{\mathcal{#1}}
\def\mrm#1{\mathrm{#1}}

\def\tbf#1{\textbf{#1}}

\def\R{\mathbb{R}}
\def\<{\left\langle} % Angle brackets
\def\>{\right\rangle}

\def\defeq{\triangleq} % defined equal to
\def\norm#1{\left\|{#1}\right\|} % A norm with 1 argument
\def\what#1{\widehat{#1}}

\def\E{\mbb{E}} % Expectation symbol

\def\Var{\mrm{Var}} % Variance symbol

\def\Cov{\mrm{Cov}} % Covariance symbol

\newcommand{\Unif}{\textnormal{Unif}}

\def\indep{\perp\!\!\!\perp} % conditional independence

\providecommand{\argmin}{\mathop\mathrm{arg min}}

\ifdefined\nonewproofenvironments\else
\ifdefined\ispres\else
\newtheorem{example}{Example}
\newenvironment{proof-sketch}{\noindent\textbf{Proof Sketch}
  \hspace*{1em}}{\qed\bigskip\\}
\newenvironment{proof-idea}{\noindent\textbf{Proof Idea}
  \hspace*{1em}}{\qed\bigskip\\}
\newenvironment{proof-of-lemma}[1][{}]{\noindent\textbf{Proof of Lemma {#1}}
  \hspace*{1em}}{\qed\\}
\newenvironment{proof-of-theorem}[1][{}]{\noindent\textbf{Proof of Theorem {#1}}
  \hspace*{1em}}{\qed\\}
\newenvironment{proof-attempt}{\noindent\textbf{Proof Attempt}
  \hspace*{1em}}{\qed\bigskip\\}

%% file: introduction-hirid.tex
\section{Introduction}\label{sec:intro}

Operational decisions often depend on the distributions of uncertain quantities. For instance, staffing levels are set by the quantiles of demand, and a clinical alarm by how often a signal crosses a threshold. Increasingly, the data behind such decisions arrive as distributions, each observed through finitely many samples and arranged in a partially observed matrix. An intensive-care monitor, for example, records a stream of heart-rate or blood-pressure readings for each patient over a window of time, and similar distributional data arise in clinical trials, with repeated measurements per patient-condition pair. Standard matrix completion, developed for scalar entries \citep{candes2010, cai2013matrix, chi2018low}, forces the analyst to collapse each entry to a single number, typically the mean, before imputing, and so discards the variance, quantiles, and tail behavior that the downstream decision needs. Consider blood pressure readings over time in an intensive-care unit (ICU). On top of the average blood pressure, clinicians often measure the \emph{hypotension burden}, the fraction of time that a patient's blood pressure stays below a critical threshold \citep{bakhsh2026impact}, because it better reflects how long and to what extent certain organs are not receiving enough blood \citep{marino2012icu}. This fraction is a probability under the window's distribution of readings, and so a mean-only imputation discards it.

In this paper, we study \emph{distributional matrix completion} (DMC): the problem of imputing missing entries when each entry is a probability distribution over $\mathbb{R}$ and the analyst observes only finitely many samples from a subset of entries. DMC involves two sources of information loss. The first is that for each observed entry we have only an empirical approximation to the underlying distribution, an object that lives in an infinite-dimensional space where distance and averaging must be defined with care. The second is the familiar matrix-completion loss: only a fraction of matrix entries are observed.

\subsection{Our approach: Dist-NN}

We propose \textbf{Dist-NN}, a nearest-neighbor estimator for DMC that operates in the 2-Wasserstein space of probability distributions. Given a target entry $(i,j)$, Dist-NN identifies rows whose entries with row $i$ are close in average 2-Wasserstein distance and then aggregates the corresponding distributions in column $j$ via their 2-Wasserstein barycenter. In one dimension, the 2-Wasserstein geometry has a particularly clean structure: every distribution corresponds uniquely to its quantile function, and the barycenter of a set of distributions is the distribution whose quantile function is the pointwise average of theirs. This quantile representation makes Dist-NN both efficient to compute and amenable to theoretical analysis. Beyond the basic single-sided estimator, the same geometry yields a \emph{two-sided} variant that borrows strength across both rows and columns, and a \emph{doubly robust} variant whose bias decays at the product of the row and column rates. Because the Wasserstein barycenter preserves shape (e.g., the barycenter of Gaussians is again Gaussian rather than a Gaussian mixture), Dist-NN inherits the geometric character of the observed data.

\cref{fig:random-samples} previews the gain. In two simulated matrices, one with Gaussian entries and one with continuous-uniform entries, each with one held-out entry, Dist-NN's imputed distribution (thick blue) tracks the true CDF (dotted black) far more closely than the empirical distribution formed from that entry's own samples (the \emph{Observed sample} curves). Dist-NN achieves this by pooling information across similar entries, so it effectively draws on far more samples than the single entry provides. This gain carries over to estimates of the variance, quantiles, value-at-risk, and threshold-crossing probabilities that operational decisions depend on, none of which a mean-only imputation can recover. In addition, because \distnn\ outputs a full distribution, each metric is computed from a single imputed object, with no need to re-run the imputation for each metric.

\begin{figure}[!h]
    \begin{center}
    \resizebox{\linewidth}{!}{
        \begin{tabular}{cc}
            \includegraphics[width=0.45\linewidth]{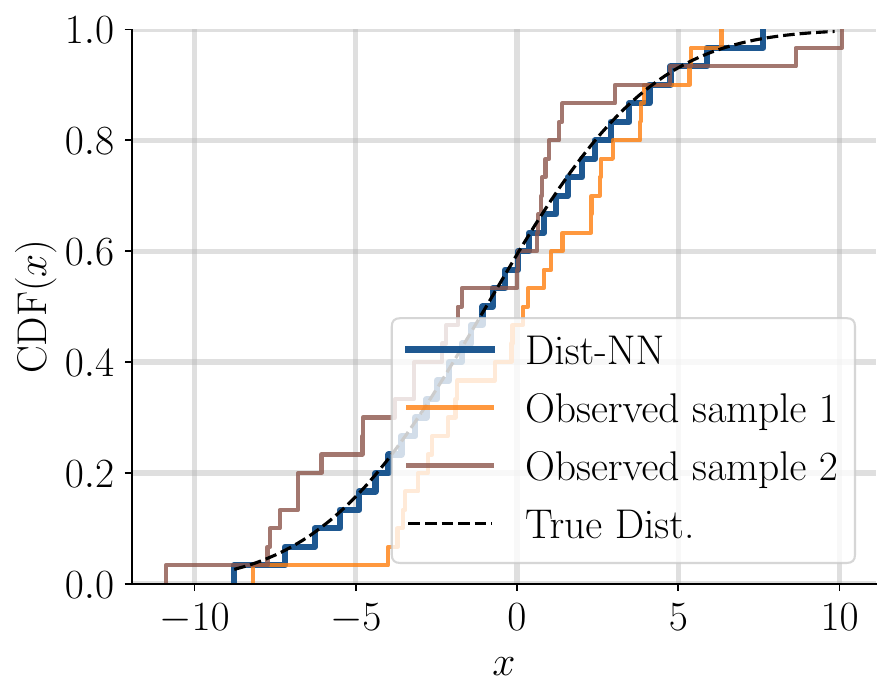} &
            \includegraphics[width=0.45\linewidth]{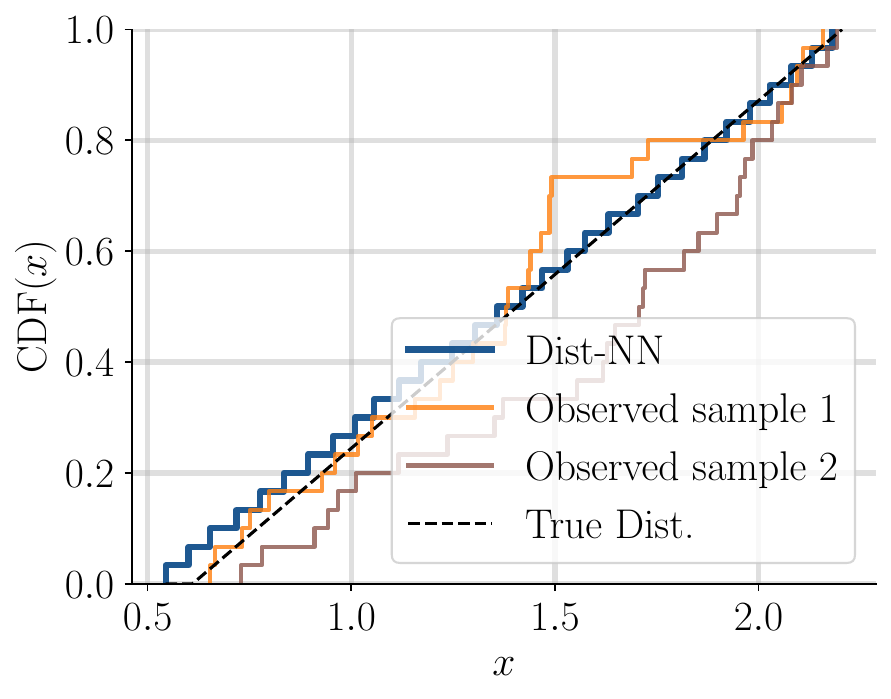}
            \\
            \parbox[c]{0.33\textwidth}{\centering (a) Gaussian} &
            \parbox[c]{0.33\textwidth}{\centering (b) Continuous uniform}
        \end{tabular}
    }
    \end{center}

    \caption{\tbf{\distnn\ recovers an entry's distribution far better than its raw observed samples.} We simulate two distributional matrix completion examples, one with Gaussian entries and one with continuous-uniform entries; in each, only the target matrix entry is unobserved. The \distnn\ imputed CDF (thick blue) closely matches the true CDF (dotted black), whereas the empirical CDFs built from the target entry's own samples (\emph{Observed sample} 1 and 2, thin lines) are noticeably farther off. \distnn\ attains this by pooling information across similar entries, giving it an effectively larger sample than the single entry provides. See \cref{sec:sim} for more details on our simulated tests.}\label{fig:random-samples}
\end{figure}

\subsection{Motivating example: intensive-care vital-sign monitoring}\label{sec:motivating-example}

Patients in an intensive-care unit are continuously monitored by bedside instruments that record vital signs such as heart rate, blood pressure, and oxygen saturation at high frequencies, often once every couple of minutes. Over any fixed window of time, the readings for a given patient and vital sign form a \emph{distribution} whose shape carries clinical meaning beyond the mean: the spread reflects physiologic instability, the lower tail of blood pressure quantifies hypotension burden, and the upper tail of heart rate flags tachycardic episodes. Summarizing a monitoring window by its average alone discards much of the information a clinician acts on.

Datasets of monitoring data often have missing entries due to random reasons unrelated to the patient's underlying physiology. Instruments malfunction, leads and probes detach as patients move or are repositioned, sensors are temporarily removed during transport or procedures, and telemetry drops out. The result is that whole windows of a vital sign are absent for some patients and present for others. This induces a distributional matrix indexed by patient (rows) and monitoring window (columns), where entry $(i,j)$ is the distribution of a chosen vital sign for patient $i$ during window $j$, observed through the readings collected in that window. 
% Rows sharing latent physiological state, and columns at the same phase of an ICU stay, carry similar distributions, exactly the structure Dist-NN exploits. 

In \cref{sec:empirical}, we apply Dist-NN to the HiRID high-resolution ICU database \citep{faltys2021hirid, hyland2020early}, comprising almost 34,000 patient admissions to the ICU at Bern University Hospital in Switzerland with vital signs recorded about every two minutes. To mimic missing entries, we mask windows of heart rate and blood pressure completely at random and impute their distributions, recovering the mean, the variance, and clinically salient tail functionals.

\subsection{Contributions}

In this paper, we make three main contributions:
\begin{enumerate}
    \item \textbf{New theory for empirical Wasserstein barycenters (\cref{sec:barycenters}).} We prove a multiplicative rate (\cref{prop:barycenters}), a functional CLT (\cref{thm:functional-clt}), and simultaneous confidence bands (\cref{cor:confidence-band}), stated generically for collections of distributions under certain regularity conditions.
    \item \textbf{Distributional nearest-neighbor matrix completion (\cref{sec:nn}).} We develop the Dist-NN estimators (single-sided, two-sided, and doubly robust), together with consistency and convergence-rate guarantees obtained by plugging our new barycenter theory of \cref{sec:barycenters} into a bias-variance decomposition. 
    % For the single-sided estimator $\hat\mu_{ij}$ of the true distribution $\mu_{ij}$ at entry $(i,j)$, we prove the rate
    % \begin{align}\label{eq:intro-main-bound}
    % W_2^2(\hat{\mu}_{ij}, \mu_{ij}) = \mathcal{O}_p\!\left(\eta + \frac{1}{p\sqrt{M}} + \frac{1}{n_j |\nn{\eta}{i}|} + \frac{\log^2 n_j}{n_j^2}\right)
    % \end{align}
    % as the number of columns, $M$, and the target-column sample count per entry, $n_j$, grow, where $\eta$ is the neighbor-distance threshold, $p$ the observation probability, and $|\nn{\eta}{i}|$ the number of $\eta$-neighbors. We additionally give entrywise inference via the functional CLT, and characterize how the two-sided and doubly robust variants improve the variance and bias, respectively.
    \item \textbf{Empirical validation (\cref{sec:sim,sec:empirical}).} We verify the theoretical rates and coverage through simulations, and, on real ICU monitoring data (HiRID), demonstrate that \distnn\ recovers decision-relevant distributional quantities, in particular hypotension-relevant tails, substantially more accurately than the imputation shipped with the dataset.
\end{enumerate}

\subsubsection{New Wasserstein barycenter theory}\label{sec:intro-barycenter}

Dist-NN's accuracy depends heavily on how well an empirical Wasserstein barycenter (the average of finitely-sampled neighbor distributions) approximates the population barycenter of the true neighbor distributions. From the prior literature, the sharpest general-class result \citep{bigot2018upper} attains only an \emph{additive} rate $\mathcal{O}(\frac{1}{k} + \frac{1}{n})$, with $k$ the number of distributions and $n$ the number of samples per distribution, while faster \emph{multiplicative} behavior has only been established for individual distributions such as the uniform distribution \citep{bobkov2019one} (\cref{sec:prior-barycenter} gives the detailed comparison).

We add to this literature by proving new Wasserstein barycenter theory, under several regularity conditions, which we believe is of independent interest:
\begin{itemize}
    \item \textbf{A multiplicative rate, uniformly over regular probability distributions.} We prove that the empirical barycenter of a collection of \emph{regular} probability distributions (defined in \cref{sec:regular-measures}) attains an error rate that is multiplicative in the number of distributions and the per-distribution sample size, recovering for a general class the scaling previously available only for specific distributions.
    \item \textbf{A functional central limit theorem.} We show the barycenter quantile process is, up to a vanishing error, a centered Gaussian process, with a covariance kernel built from the per-distribution quantile densities. To our knowledge this is the first functional CLT for the empirical barycenter quantile process.
    \item \textbf{Simultaneous confidence bands.} Studentizing the process makes its limit pivotal (i.e., does not depend on unknown parameters), so a single universal value yields confidence bands valid simultaneously over quantile levels. A decision-maker can thus attach calibrated uncertainty to any quantile-based metric of interest.
\end{itemize}

\subsection{Related work}

\paragraph{Matrix completion.} Algorithms for scalar matrix completion fall into two broad families. Empirical risk minimization (ERM) methods recover a low-rank or otherwise regularized matrix by minimizing reconstruction error on observed entries \citep{cai2013matrix, chi2018low, agarwal2023causal}. ERM methods rely on linear-algebraic structure (rank, singular values, nuclear norms) that does not transfer cleanly to the space of probability distributions, where the notions of singular values and rank are not easily translated. We therefore build on the second family: nearest-neighbor methods, which estimate a missing entry by aggregating ``similar'' rows or columns. Beyond their scalability \citep{dou2016survey, khan2017collaborative}, nearest-neighbor methods come with strong guarantees under latent-factor models \citep{li2019, dwivedi2022doubly}, and they extend naturally to the distributional setting once one supplies a notion of distance and average, both of which the 2-Wasserstein space provides. The doubly robust nearest-neighbor estimator of \citet{dwivedi2022doubly} inspires our two-sided and doubly robust variants. Additional matrix completion work has extended the scalar setting to noisy observations \citep{candes2010, li2019} and panel data for causal inference \citep{dwivedi2022doubly, agarwal2023causal}, but in all cases the entries remain scalar.

% \paragraph{Distribution-valued and metric-space-valued statistics.} A parallel line of work studies regression and inference when the response is a distribution or a general metric-space object: Fréchet regression \citep{lin2023causal}, Wasserstein regression on quantile functions \citep{chen2023Wasserstein}, and related distribution-on-distribution methods. These works target a regression map from covariates to a distributional target; our problem is instead one of completion under partial observation with latent row and column structure, and our estimator is a local average (a barycenter of neighbors) rather than a fitted global map. The 2-Wasserstein geometry we use is shared with this literature, but our contribution, convergence and inference for the empirical barycenter itself, is orthogonal to it.

\paragraph{Optimal transport and Wasserstein barycenters.} Optimal transport (OT) supplies the geometric machinery we need to define distance and averaging between distributions \citep{monge1781memoire, kolouri2017optimal, thorpe2018introduction}. We work in the 2-Wasserstein metric, under which the barycenter of a set of distributions \citep{cuturi14, bigot2013consistent, Bigot2013, bigot2020} preserves the shape of its inputs. In one dimension the barycenter is the inverse of the average quantile function; in higher dimensions it is NP-hard to compute \citep{altschuler2022wasserstein}, which is one reason we restrict to one-dimensional distributions. 
We choose to use the Wasserstein geometry rather than a kernel-based metric such as maximum mean discrepancy because the Wasserstein barycenter has several nice geometric properties which we discuss in detail in \cref{sec:ot-background}. 
% The Wasserstein barycenter is itself a distribution with a quantile function, so Dist-NN returns an imputed distribution (and hence estimates of quantiles and tail functionals) and admits a functional CLT for the quantile process, none of which a reproducing-kernel mean embedding provides. 
The statistics of empirical barycenters has been studied extensively by \citet{bigot2018upper,le2022fast,ahidarcoutrix2020} among others; our rate and functional CLT (\cref{sec:barycenters}) extend this line of work, drawing on strong-approximation theory for empirical quantile processes \citep{csorgo1978strong,bobkov2019one}. 
% Prior central limit theorems in this area target the scalar Wasserstein-variation functional \citep{delbarrio2019clt}, a different object from the barycenter quantile process we study.

\paragraph{Distributional estimation in matrix and panel settings.} The closest prior work is~\cite{gunsilius2023distributional}, who study distributional estimation in a \emph{synthetic control} setting via quantile-function regression. We differ mainly in three respects: we use nearest-neighbor matching rather than synthetic-control regression, we allow nonlinear maps between latent factors and distributions, and under stronger regularity we obtain faster rates.

\subsection{Organization}

In \cref{sec:setup}, we introduce the distributional matrix model, the latent factor assumptions, and the class of regular probability distributions. In \cref{sec:barycenters}, we develop the empirical Wasserstein barycenter theory. In \cref{sec:nn}, we present the Dist-NN estimators and their guarantees. In \cref{sec:sim}, we verify the rates in simulation, and in \cref{sec:empirical} we apply the method to ICU vital-sign monitoring.

%% file: setup.tex
\section{Setup and assumptions}\label{sec:setup}

Now, we formally describe the distributional matrix completion setup, review the optimal-transport tools used throughout, and state the data-generating process and regularity conditions under which our guarantees hold.

\subsection{Problem setup}\label{sec:problem-setup}
We analyze a partially observed $\numrows+1$ by $\numcols+1$ matrix, denoted $Y$, where each observed entry contains an array of scalars. (We add 1 to each dimension to simplify the notation in our main theorem.) Within each entry $Y_{ij}$, the scalars are drawn independently and identically (i.i.d.) from some law $\mu_{ij}$. Every entry in a column $j$ has $n_j$ samples, though our algorithm and guarantees carry over directly to unequally sized arrays. If an entry $(i,j)$ is observed, we write its samples as ${\{y_{ij,k}\}}_{k=1}^{n_j} \sim \mu_{ij}$. For observed entries we see only this empirical data, not the true distribution. Our goal is to use the observed empirical distributions to estimate the true distributions $\mu_{ij}$ at all entries, observed and unobserved. Let $A$ be an $\numrows+1$ by $\numcols+1$ binary matrix recording which entries are observed. Then,
\begin{align}
    \qtext{for} i \in [\numrows+1], j\in[\numcols+1]:\qquad
    Y_{ij} = \begin{cases}
        \{y_{ij,1},\dots,y_{ij,n_j}\} &\quad \text{if} \, A_{ij} = 1 \\
        \text{missing} &\quad \text{if} \, A_{ij} = 0
    \end{cases}
\end{align}
where ${\{y_{ij,k}\}}_{k=1}^{n_j} \sim \mu_{ij}$.

\subsection{Optimal transport: distance and averaging}\label{sec:ot-background}

Generalizing nearest neighbors to distributions requires a notion of \emph{distance} and \emph{average} on probability distributions. We use the 2-Wasserstein distance and the 2-Wasserstein barycenter, both of which have closed forms in one dimension.

The Wasserstein distance is a natural choice because (i) it is a metric, (ii) it has a closed-form one-dimensional solution, and (iii) it behaves well when distributions do not share support. For instance, total variation has $\mathrm{TV}(\Unif(0,1), \Unif(2,3))=\mathrm{TV}(\Unif(0,1),\Unif(4,5)) = 0$, whereas the 2-Wasserstein distance $W_2$ satisfies $W_2(\Unif(0,1), \Unif(2,3)) < W_2(\Unif(0,1),\Unif(4,5))$; see \citep{wasserman2017optimal} for further geometric properties.

In one dimension, the 2-Wasserstein distance is an $L^2$ norm between quantile functions: for probability distributions $\mu,\nu$ on $\R$ with finite second moment \citep[Eq.~2]{bigot2020},
\begin{align}\label{eq:wasserstein-1dim}
    W_2(\mu,\nu) = \bigg(\int_0^1 |F_\mu^{-1}(x) - F_\nu^{-1}(x)|^2 dx\bigg)^{1/2} = \left\lVert F_\mu^{-1} - F_\nu^{-1}\right\rVert_{L^2(0,1)},
\end{align}
where $F_\mu^{-1}$ and $F_\nu^{-1}$ are the respective quantile functions. For empirical distributions $\mu_n,\nu_n$ with the same number of samples $n$, generated from $\{X_i\}_{i=1}^n$ and $\{Y_i\}_{i=1}^n$ \citep[Lemma~4.2]{bobkov2019one},
\begin{align}\label{eq:wasserstein-1dim-emp}
    W_2(\mu_n,\nu_n) = \bigg(\frac{1}{n}\sum_{i=1}^n \big(X^{(i)} - Y^{(i)}\big)^2\bigg)^{1/2},
\end{align}
where $X^{(i)}, Y^{(i)}$ are the $i$-th order statistics. Thus $W_2$ is computable in $\mathcal{O}(n\log n)$ time, dominated by sorting. In this paper, we use ``Wasserstein'' and ``2-Wasserstein'' interchangeably because we only consider the 2-Wasserstein distance. We also denote $W_2(\R)$ as the space of one-dimensional probability distributions on $\R$ with finite second moment equipped with $W_2$.

The Wasserstein barycenter of $\mu_1,\dots,\mu_N \in W_2(\R)$ is the distribution $\mu$ minimizing $\sum_{i=1}^N W_2^2(\mu,\mu_i)$, a distributional analogue of an average. Its quantile function has the closed form \citep[Eq.~8]{bigot2020}
\begin{align}\label{eq:barycenter-quantile}
    F_\mu^{-1} = \frac{1}{N}\sum_{i = 1}^N F_{\mu_i}^{-1}.
\end{align}
When each $\mu_j$ is empirical with order statistics $\{X_{\mu_j}^{(i)}\}_{i=1}^n$, the barycenter is the empirical distribution with order statistics \citep[Section 2.4]{bigot2018upper}
\begin{align}\label{eq:barycenter-order-stats}
    X_\mu^{(i)} = \frac{1}{N} \sum_{j=1}^N X_{\mu_j}^{(i)},
\end{align}
computable in $\mathcal{O}(N n\log n)$ time. The barycenter preserves geometric structure: the barycenter of Gaussians is Gaussian \citep{bigot2020}, unlike a mixture.

\subsection{Factor model and assumptions}\label{sec:factor-model}

Without structure, matrix completion is ill-posed. We assume the matrix entries share latent structure, following the general framework of \citet{chatterjee2015}, which covers scalar-valued low-rank models as a special case \citep[Sec.~D]{li2019}.

\begin{assumption}[Lipschitz latent-factor model]\label{assum:latent-factors}
    Assume the following latent structure: (i) there exist latent bounded metric spaces $\parenth{\Hrow, d_\mathrm{row}}$ and $\parenth{\Hcol, d_\mathrm{col}}$ for rows and columns, (ii) each row $i$ has a latent vector $\latRow{i} \in \Hrow$ and each column $j$ a latent vector $\latCol{j} \in \Hcol$, and (iii) there exists $f: \Hrow \times \Hcol \to W_2(\R)$ such that for all $i,j$, $\mu_{ij} = f\big(\latRow{i},\latCol{j}\big)$. We impose the following conditions on $f$, each used by the estimators indicated:
    \begin{assumptionconds}
        \item\label{assum:latent-factors-lipschitz} \emph{Row-Lipschitzness} (single-sided, two-sided, and doubly robust estimators): $f$ is $L$-Lipschitz in its row argument, so that for all $x_1,x_2 \in \Hrow$ and $y\in \Hcol$,
        \begin{align}\label{eq:lipschitz-latent-factors}
            W_2\parenth{f(x_1, y), f(x_2, y)} \leq L\, d_\mathrm{row}(x_1, x_2).
        \end{align}
        \item\label{assum:latent-factors-lipschitz-col} \emph{Column-Lipschitzness} (two-sided and doubly robust estimators): $f$ is $L$-Lipschitz in its column argument, so that for all $x \in \Hrow$ and $y_1, y_2 \in \Hcol$,
        \begin{align}\label{eq:lipschitz-latent-factors-col}
            W_2\parenth{f(x, y_1), f(x, y_2)} \leq L\, d_\mathrm{col}(y_1, y_2).
        \end{align}
        \item\label{assum:latent-factors-bilinear-lipschitz} \emph{Bounded interaction} (doubly robust estimator only): for a target entry $(i,j)$, define the second-order mixed difference of the quantile functions,
        \begin{align}\label{eq:mixed-difference}
            \gamma_{uv} \defeq \quantile{\mu_{uv}} - \quantile{\mu_{uj}} - \quantile{\mu_{iv}} + \quantile{\mu_{ij}}.
        \end{align}
        There exists a constant $L_h \geq 0$ such that for all rows $i, u$ and columns $j, v$,
        \begin{align}\label{eq:bilinear-interaction}
            \norm{\gamma_{uv}}_{L^2(0,1)} \leq L_h \cdot W_2(\mu_{uj}, \mu_{ij}) \cdot W_2(\mu_{iv}, \mu_{ij}).
        \end{align}
    \end{assumptionconds}
\end{assumption}

\Cref{assum:latent-factors-lipschitz} means the true distributions vary smoothly in Wasserstein space as the latent row vectors vary, so rows close in the latent space have similar distributions within a column; \cref{assum:latent-factors-lipschitz-col} is the same statement, just across columns. The interaction bound \cref{eq:bilinear-interaction} vanishes whenever $u = i$ or $v = j$, giving $\gamma_{iv} = \gamma_{uj} = 0$, which is what lets the doubly robust estimator cancel the main effects exactly. It is the distributional analogue, in the $L^2$ geometry of quantile functions, of the bilinear factor conditions used for doubly robust scalar nearest neighbors \citep{dwivedi2022doubly}.

For a target entry $(i,j)$, our bias analysis compares the distance between two rows at the target column with their average distance across columns. Define
\begin{align}\label{eq:cf-def}
    c_f \defeq \sup_{x \neq x' \in \Hrow} \frac{W_2^2\parenth{f(x, \latCol{j}),\, f(x', \latCol{j})}}{\expect_{x'' \sim \Hcol}*{W_2^2(f(x,x''), f(x',x''))}},
\end{align}
where the supremum runs over pairs with a positive denominator, and let $c_f^{\mathrm{col}}$ denote the symmetric quantity with the roles of rows and columns exchanged. The estimated distances used by our estimators concentrate around average distances across columns, so $c_f$ converts a small estimated distance into a small distance at the target column itself. When the number of rows is fixed, our results only need the maximum of the ratio over the realized row pairs, which is almost surely finite. The results in which $\numrows \to \infty$ (\cref{thm:two-sided-rate,thm:dr-rate,cor:nn-functional-clt}) assume $c_f$ (and, where columns are pooled, $c_f^{\mathrm{col}}$) to be finite. This rules out functions $f$ whose disagreement between two rows concentrates on a vanishingly small share of columns. Both constants are explicit in our examples: $c_f = c_f^{\mathrm{col}} = 1$ for \cref{ex:location-scale-hetero}, and $c_f, c_f^{\mathrm{col}} \leq 12d$ for \cref{ex:location-scale-homo}, since the smallest eigenvalue of $\expect*{x x^\top}$ for $x$ uniform on $[0,1]^d$ is $1/12$.

We next discuss \emph{location-scale families}, which we use in our examples. A location-scale family is parameterized by a location $\alpha \in \R$ and scale $\sigma \in \R_{\geq 0}$ (e.g., Gaussian, continuous and discrete uniform). If $\mu$ is location-scale and $X\sim\mu$, then $Y \defeq \sigma X + \alpha$ has $F_Y^{-1} = \sigma F_X^{-1} + \alpha$: the quantile function is linear in that of $X$, which simplifies the analysis, since quantile functions drive both \cref{eq:wasserstein-1dim} and \cref{eq:barycenter-quantile}. Observing one sample per cell recovers noisy scalar matrix completion with i.i.d.\ location-scale noise.

\begin{example}[Homoscedastic location-scale]\label{ex:location-scale-homo}
    Let the latent row and column factors be uniform over $[0,1]^d$. For ${\latRow{i}, \latCol{j} \in [0,1]^d}$, let $\quantile{\mu_{ij}} = \sigma^2 F^{-1} + \langle \latRow{i}, \latCol{j} \rangle$ where $F^{-1}$ is a quantile function of the location-scale family.
\end{example}
\begin{example}[Heteroscedastic location-scale]\label{ex:location-scale-hetero}
    Let the latent row and column factors be uniform over $[0,1]$. For ${\latRow{i}, \latCol{j} \in [0,1]}$, let $F^{-1}_{\mu_{ij}} = \latCol{j} F^{-1}_\mu + \latRow{i}$ for a quantile function $F^{-1}_\mu$ of the location-scale family.
\end{example}
\cref{ex:location-scale-homo} induces low-rank structure on the first moment with constant scale. \cref{ex:location-scale-hetero} gives each row a different location and each column a different scale. The former satisfies \cref{assum:latent-factors-lipschitz} with $L = \sqrt{d}$, the latter with $L = 1$.

\begin{assumption}[MCAR]\label{assum:mcar}
    Each entry's missingness $A_{ij}\sim \mathrm{Bernoulli}(p)$ is i.i.d.\ across entries and independent of the row and column latent factors.
\end{assumption}
MCAR is standard in the matrix completion literature. We believe it can be relaxed to the missing-not-at-random class of missingness mechanisms, but leave this as future work.

\paragraph{Equal sample size.} For the symmetric (two-sided and doubly robust) estimators, which pool samples across columns, we assume for simplicity that every pooled entry contributes the same number of samples, $n_v \equiv n$. This loses no generality. The same results hold with $n \defeq \min_v n_v$ in place of the shared sample size: drawing $\min_v n_v$ samples from each pooled entry reduces the heterogeneous case to the equal-size one, at the cost only of a larger constant (the bound is then governed by the smallest per-entry sample size). The equal-size assumption lets us invoke the empirical-barycenter rate (\cref{prop:barycenters}), which is stated for collections observed through an equal number of samples each. The single-sided estimator does not need this because it averages only within column $j$.

\subsection{Regular probability distributions and the estimand}\label{sec:regular-measures}

Our barycenter rate and inference results hold over a class of \emph{regular} probability distributions, standard in empirical-quantile-process theory \citep{csorgo1978strong,bobkov2019one}.

\begin{definition}[Regular probability distribution]\label{def:regular-measure}
    A distribution $\nu$ with cumulative distribution function $F$ and density $f$ is regular if (i) $F$ is twice-differentiable and continuous on $(a,b)$ with $-\infty < a < b < \infty$, (ii) there is a universal $C > 0$ with $f(x) \geq C$ for all $x \in (a,b)$, (iii) $\nu$ has a finite second moment, (iv) $(F^{-1})'$ is $L'$-Lipschitz, (v) $f$ is non-decreasing in a right-neighborhood of $a$ and non-increasing in a left-neighborhood of $b$, and (vi)
    \begin{align}\label{eq:regularity-F}
        \sup_{x \in (a,b)} \frac{F(x)\parenth{1-F(x)}}{f^2(x)} \abss{f'(x)} < 2.
    \end{align}
\end{definition}

\begin{assumption}[Regularity conditions]\label{assum:regular}
    For $(i,j) \in [\numrows + 1] \times [\numcols + 1]$, the distribution $\mu_{ij}$ is regular.
\end{assumption}

All continuous uniform distributions satisfy \cref{assum:regular} (constant density on their support). Because we require a uniformly lower-bounded density on compact support, Gaussians do not satisfy it, which is a common issue in the asymptotic analysis of the Wasserstein distance \citep[Remark~1]{martinet2022variance}. Truncated Gaussians do satisfy it, and we use them to verify the barycenter rate in \cref{sec:sim-barycenter}; our simulations in \cref{sec:sim} further show that on Gaussian-valued matrices, the empirical rates still match our predictions. The constant $2$ in \cref{eq:regularity-F} merely simplifies the analysis and may be raised to any $\gamma < \infty$.

%% file: barycenters.tex
\section{Convergence and inference for empirical Wasserstein barycenters}\label{sec:barycenters}

This section develops our theory for empirical Wasserstein barycenters, stated generically for a collection of regular probability distributions and independent of the matrix-completion application. \cref{sec:nn} then uses these results to obtain the \distnn\ guarantees. We consider a collection of $k$ probability distributions $\{\nu_j\}_{j=1}^k$ on $\R$ and let $\hat\nu_{j,n}$ denote the empirical distribution of $n$ i.i.d.\ samples from $\nu_j$. We write $\Bar{\nu}$ and $\what{\overline{\nu}}_n$ for the Wasserstein barycenters of $\{\nu_j\}$ and $\{\hat\nu_{j,n}\}$, respectively, and study the squared Wasserstein risk $W_2^2(\what{\overline{\nu}}_n, \Bar{\nu})$ as a function of both the number of distributions $k$ and the per-distribution sample size $n$.

\subsection{Prior rates: additive in general, multiplicative case by case}\label{sec:prior-barycenter}

\paragraph{The additive rate for general distribution classes.}
The most directly comparable prior work is \citet{bigot2018upper}, who study this setting. For a general class of distributions with a finite $J_2$ functional (see Eq. 3.10 in \citet{bigot2018upper}), their Corollary~3.1 establishes
\begin{align}\label{eq:prior_result}
    \expect*{W_2^2(\what{\overline{\nu}}_n, \Bar{\nu})}
    = \order\parenth{\frac{1}{k} + \frac{1}{n}}.
\end{align}
The two terms are \emph{additive}: the error is governed by the slower of the $k^{-1}$ and $n^{-1}$ rates, so increasing the per-distribution sample size $n$ beyond $n \asymp k$ yields no further improvement.
%  Moreover, \citet{bigot2018upper} show via a minimax lower bound (their Theorem~3.4) that the $n^{-1}$ scaling in the leading term is optimal in their setting, where the distributions $\{\nu_j\}$ are themselves an i.i.d.\ sample from a population of random distributions and the irreducible population variance $\mathrm{Var}(\nu)$ enters the risk.

\paragraph{Multiplicative behavior for specific distributions.}
While the general bound \cref{eq:prior_result} is additive, \citet{bigot2018upper} also derive an \emph{exact} closed-form risk for samples of equal size (their Theorem~3.1). Evaluating it for specific barycenters reveals a faster \emph{multiplicative} term; their Remark~3.7 records
\begin{align}\label{eq:bigot-examples}
    \expect*{W_2^2(\what{\overline{\nu}}_n, \Bar{\nu})}
    \asymp
    \begin{cases}
        \dfrac{1}{k} + \dfrac{1}{nk} + \dfrac{1}{n^2},
            & \Bar{\nu} \text{ uniform},\\[1.2ex]
        \dfrac{1}{k} + \log n \parenth{\dfrac{1}{nk} + \dfrac{1}{n}},
            & \Bar{\nu} \text{ exponential},\\[1.2ex]
        \dfrac{1}{k} + \log\log n \parenth{\dfrac{1}{nk} + \dfrac{1}{n}},
            & \Bar{\nu} \text{ Gaussian}.
    \end{cases}
\end{align}
Each is strictly faster than \cref{eq:prior_result}, and each contains a multiplicative cross-term in $(nk)^{-1}$. The uniform case shows that, depending on $k$, the scaling in $n$ can range from $n^{-1}$ to $n^{-2}$. But these sharp rates are obtained \emph{one distribution at a time}, each from an exact order-statistic variance computation. \citet{bigot2018upper} note that a sharp estimate ``remains a difficult task except for specific distributions.''

\subsection{A multiplicative rate, uniformly over regular probability distributions}\label{sec:our-barycenter}

Our first result advances this prior work by proving a multiplicative rate that holds uniformly over the entire class of regular probability distributions (\cref{def:regular-measure}).

\begin{proposition}[Convergence of the barycenter of empirical distributions]\label{prop:barycenters}
    Consider a collection of regular (\cref{def:regular-measure}) probability distributions $\sbraces{\nu_j}_{j=1}^{k}$ and let $\what \nu_{j, n}$ denote the empirical distribution obtained from $n$ i.i.d.\ samples from $\nu_j$ for each $j \in [k]$. Define the two barycenters
    \begin{align}\label{eq:barrycenters}
        \overline{\nu}
        \defeq \argmin_{\nu} \sum_{j =1}^k W_2^2(\nu, \nu_{j})
        \qtext{and}
        \what{\overline{\nu}}_n
        \defeq \argmin_{\nu} \sum_{j =1}^k W_2^2(\nu, \hat\nu_{j,n}).
    \end{align}
    Suppose $k = k(n)$ satisfies $\log k = \order(\log n)$ as $n \to \infty$ (i.e., $k$ grows at most polynomially in $n$). Then
    \begin{align}\label{eq:our_w2_rate_prop1}
        W_2^2(\what{\overline{\nu}}_n,\overline{\nu})
        = \order_p\parenth{\frac{1}{n k} + \frac{\log^2 n}{n^2}}
        \qtext{as} n\to\infty,
    \end{align}
    uniformly with constants in $\order_p(\cdot)$ that depend neither on the set $\{\nu_j\}_{j=1}^k$ nor on its size $k$.
\end{proposition}

\Cref{prop:barycenters} advances the prior work of \cref{sec:prior-barycenter} in three ways. First, the rate $\tfrac{1}{nk} + \tfrac{\log^2 n}{n^2}$ is \emph{multiplicative}: it combines $k$ and $n$ in a single $(nk)^{-1}$ term, rather than the additive $\tfrac1k + \tfrac1n$ of \cref{eq:prior_result}. As a result, the scaling in $n$ improves from $n^{-1}$ toward $n^{-2}$ as $k$ grows, recovering for a general class the phenomenon previously seen only in the distribution-specific calculations of \cref{eq:bigot-examples}. Second, it holds \emph{uniformly over all regular probability distributions}, requiring only \cref{def:regular-measure} rather than an exact computation per distribution. Third, our setting differs from that of \citet{bigot2018upper} in a way that explains why the multiplicative rate is attainable. There, the $\{\nu_j\}$ are an i.i.d.\ sample from a population, so the risk carries an irreducible $\tfrac1k\mathrm{Var}(\nu)$ term whose $\tfrac1k$ rate is minimax-optimal. Here, the collection is fixed (or, in the matrix-completion application, conditioned upon), so there is no population-sampling term and the residual is pure estimation error. Our result therefore does not contradict their lower bound. Rather, it quantifies the complementary fixed-collection setting, the one relevant to nearest-neighbor matrix completion. We prove \cref{prop:barycenters} in \cref{sec:proof-techniques}, with the supporting lemmas proven in the \ECname, by bridging two threads that have developed independently: convergence rates for empirical Wasserstein barycenters, and strong-approximation theory for empirical quantile functions \citep{csorgo1978strong}. Conditions (iv) and (v) of \cref{def:regular-measure} are exactly what this bridge requires.

\subsection{A functional central limit theorem}\label{sec:functional-clt}

\Cref{prop:barycenters} controls the magnitude of the error. We now expand on this by proving a distributional statement: the entire rescaled barycenter quantile process is, up to a vanishing error, an explicit centered Gaussian process. 
% To our knowledge this is the first functional CLT for the empirical barycenter quantile process. 
Define the \emph{barycenter quantile process}
\begin{align}\label{eq:barycenter-quantile-process}
    \mbb{Q}_{n,k}(t) \defeq \sqrt{nk}\,\big(\quantile{\what{\overline{\nu}}_n}(t)
        - \quantile{\overline{\nu}}(t)\big), \qquad t \in (0,1).
\end{align}

\begin{theorem}[Functional CLT for the empirical barycenter quantile process]\label{thm:functional-clt}
    Consider a collection of regular (\cref{def:regular-measure}) probability distributions $\sbraces{\nu_j}_{j=1}^k$, and let $\what\nu_{j,n}$ denote the empirical distribution of $n$ i.i.d.\ samples from $\nu_j$. Suppose $k = k(n)$ satisfies $\log k = \order(\log n)$. For each $j \in [k]$, write $w_j \defeq 1/(f_{\nu_j}\circ\quantile{\nu_j})$ for the quantile-density function of $\nu_j$. Then there exist independent standard Brownian bridges $\mbb{B}_{1,n},\dots,\mbb{B}_{k,n}$, defined on a common probability space with the samples, such that the centered Gaussian process
    \begin{align}\label{eq:clt-gaussian-process}
        \mbb{G}_k(t) \defeq \frac{1}{\sqrt k}\sum_{j=1}^k w_j(t)\,\mbb{B}_{j,n}(t),
        \qquad t \in (0,1),
    \end{align}
    with covariance kernel
    \begin{align}\label{eq:clt-covariance}
        \Sigma_k(s,t)
        \defeq \Cov\big(\mbb{G}_k(s),\mbb{G}_k(t)\big)
        = \big(\min(s,t)-st\big)\cdot\frac1k\sum_{j=1}^k w_j(s)\,w_j(t),
    \end{align}
    satisfies
    \begin{align}\label{eq:functional-clt}
        \sup_{0<t<1}\abss{\mbb{Q}_{n,k}(t) - \mbb{G}_k(t)}
        \overset{a.s.}{=} \order\!\parenth{\sqrt k\,\frac{\log n}{\sqrt n}}
        \qquad \text{as } n \to \infty,
    \end{align}
    where the constant in $\order(\cdot)$ depends neither on the collection $\sbraces{\nu_j}_{j=1}^k$ nor on its size $k$. In particular, in $L^2(0,1)$,
    \begin{align}\label{eq:functional-clt-l2}
        \lnorm{\mbb{Q}_{n,k} - \mbb{G}_k}
        = \order_p\!\parenth{\sqrt k\,\frac{\log n}{\sqrt n}},
    \end{align}
    so that if $\sqrt k\,\log n/\sqrt n \to 0$, then $\mbb{Q}_{n,k}$ is asymptotically equal, uniformly in $t$, to the centered Gaussian process $\mbb{G}_k$ with covariance \cref{eq:clt-covariance}.
\end{theorem}

The covariance \cref{eq:clt-covariance} is a weighted Brownian-bridge kernel, in which the per-distribution quantile densities $w_j$ enter as deterministic weights. It depends only on the collection $\{\nu_j\}$ and is free of $n$. Two features distinguish it from existing barycenter CLTs. First, prior CLTs target the \emph{scalar} Wasserstein-variation functional $\sum_j W_2^2(\nu_j, \overline\nu)$, used for evaluating the distance between distributions \citep{delbarrio2019clt}, whereas \cref{thm:functional-clt} characterizes the law of the quantile \emph{process} itself, with an explicit covariance that accounts for per-distribution sample complexity. Second, the coupling is uniform over $t$, which is what enables \emph{simultaneous} inference (\cref{sec:confidence-bands}).

\begin{remark}[Necessity of $n\to\infty$ and the discretization bias]\label{rmk:functional-clt-regime}
The error term in \cref{eq:functional-clt} vanishes only when $\sqrt k\,\log n/\sqrt n \to 0$. Thus \cref{thm:functional-clt} concerns the regime $n \to \infty$ with $k = o(n/\log^2 n)$; it is \emph{not} a fixed-$n$, large-$k$ statement, and this is intrinsic. The $\sqrt k$ factor is discretization bias: the per-distribution approximation remainders are independent across $j$, so their centered average is $k$-uniform $\order_p(\log n/\sqrt n)$, and the entire $\sqrt k$ comes from the rescaled discretization bias of the empirical barycenter quantile function, the same irreducible bias producing the $\log^2 n/n^2$ floor in \cref{prop:barycenters}. 
% For fixed $n$, $\what{\overline{\nu}}_n$ does not converge to $\overline\nu$ as $k\to\infty$, so no CLT at the $\sqrt{nk}$ rate centered at $\overline\nu$ can hold there; a fixed-$n$, large-$k$ limit becomes available only after recentering at the mean empirical quantile, a different estimand we do not pursue.
\end{remark}

\subsection{Simultaneous confidence bands}\label{sec:confidence-bands}

The functional limit yields confidence bands valid \emph{simultaneously} over a range of quantile levels. The key simplification is that the studentized limit is \emph{pivotal} (its distribution does not depend on the unknown distributions $\{\nu_j\}$), so a single critical value applies to every entry \citep{casella2002statistical}. Dividing the limit $\mbb{G}_k$ by its pointwise standard deviation $\sqrt{\Sigma_k(t,t)}$ cancels the unknown quantile densities and leaves a standardized Brownian bridge $t\mapsto \mbb{B}(t)/\sqrt{t(1-t)}$, whose law depends only on $t$. Thus, the $(1-\alpha)$-quantile of $\sup_{t\in[\delta,1-\delta]}\abss{\mbb{B}(t)}/\sqrt{t(1-t)}$ is a \emph{universal constant} $c_\alpha(\delta)$, computed once by simulating a standardized Brownian bridge, and a band of half-width $c_\alpha(\delta)\,\widehat\sigma_k(t)/\sqrt{nk}$, where $\widehat\sigma_k(t)$ estimates $\sqrt{\Sigma_k(t,t)}$, achieves simultaneous coverage. The bands are stated on an interior interval $[\delta,1-\delta]$ because the studentized supremum diverges as $t\to 0$ or $t\to 1$.

This construction is tighter than a Bonferroni band: $c_\alpha(\delta) = \order(1)$ in the grid resolution, whereas a Bonferroni band with simultaneous coverage over $m$ quantile levels inflates each pointwise interval by $z_{\alpha/(2m)} \asymp \sqrt{2\log m}$, which diverges as the grid is refined. When $n$ is moderate, we can instead simulate the supremum of $\mbb{G}_k/\widehat\sigma_k$ under a plug-in estimate of the finite-$k$ kernel \cref{eq:clt-covariance}, or bypass the scale entirely with a bootstrap. We compare these in \cref{sec:appendix-bands}\inEC. We give the formal statement in the matrix-completion specialization (\cref{cor:confidence-band}), where the difference-quotient scale estimator and the transformed-Brownian-bridge limit for a single entry make the construction concrete. 
% The generic statement above is identical with $\quantile{\what{\overline{\nu}}_n}$, $\overline\nu$, and $nk$ in place of the entrywise quantities.

%% file: nn.tex
\section{Distributional nearest-neighbor matrix completion}\label{sec:nn}

We now present the Dist-NN estimators and their guarantees. The estimators are nearest-neighbor averages in Wasserstein space. Their analysis reduces, via a bias-variance decomposition, to two pieces: the empirical-barycenter theory of \cref{sec:barycenters} for the variance, and neighborhood concentration for the bias. Nearest neighbors scales to the large datasets common in recommendation and panel-data settings, and requires only a notion of \emph{distance} and \emph{average}, supplied here by the 2-Wasserstein distance and barycenter.

\subsection{The single-sided estimator}\label{sec:single-sided}

The inputs are a data matrix $Y$, a masking matrix $A$, and a distance threshold $\distThreshold \geq 0$. For each entry $(i,j)$ we compute an estimate $\Hat{\mu}_{ij}$ of $\mu_{ij}$.

\tbf{{\distnn}$(Y,A,i,j,\distThreshold)$:}
\begin{enumerate}
    \item[\tbf{Step 1:}] \tbf{Find the set of nearest neighbors for row $i$.}\\
    For each row $u \neq i$ with $A_{uj} = 1$, define the columns observed in both rows $i$ and $u$,
    \begin{align}\label{eq:shared-col-def}
        \sharedcol{i}{u} \triangleq \{v \in [\numcols+1] \setminus\{j\}: A_{iv} = 1, A_{uv} = 1\},
    \end{align}
    and the average distance
    \begin{align}\label{eq:avg-distance-def}
        \rowdist{i}{u} \defeq \begin{cases}
             \frac{1}{\lvert\sharedcol{i}{u}\rvert} \sum_{v \in \sharedcol{i}{u}} W_2^2(Y_{iv}, Y_{uv})&\quad \text{if $\lvert \sharedcol{i}{u}\rvert \geq 1$} \\
            \infty &\quad \text{if $\lvert \sharedcol{i}{u}\rvert = 0$}.
        \end{cases}
    \end{align}
    Define row $i$'s $\eta$-nearest neighbors
    \begin{align}\label{eq:neighbors-def}
        \nn{\eta}{i} \triangleq \{u \in [\numrows+1] : A_{uj} = 1, \rowdist{i}{u} \leq \eta\}.
    \end{align}

    \item[\tbf{Step 2:}] \tbf{Average the neighbors in column $j$.}\\
    If $\nn{\eta}{i} \neq \emptyset$, estimate the quantile function of $\mu_{ij}$ by the barycenter
    \begin{align}\label{eq:dist-nn-barycenter}
        \quantile{\hat{\mu}_{ij}} = \frac{1}{\abss{\nn{\eta}{i}}}\sum_{u \in \nn{\eta}{i}} \quantile{Y_{uj}};
    \end{align}
    otherwise, report that no neighborhood was found.
\end{enumerate}
Step 1 computes pairwise Wasserstein distances between row $i$ and every row observed in column $j$. Step 2 forms the Wasserstein barycenter of the neighbors' column-$j$ distributions, via the closed form \cref{eq:barycenter-quantile}.

\subsection{Error rate and consistency}\label{sec:nn-rate}

We decompose the error of $\hat\mu_{ij}$ against $\mu_{ij}$ through the intermediate \emph{population} neighbor barycenter $\Bar{\mu}_{ij}$, with quantile function $\quantile{\Bar{\mu}_{ij}} = \abss{\nn{\eta}{i}}^{-1} \sum_{u\in\nn{\eta}{i}} \quantile{\mu_{uj}}$. By the triangle and Cauchy-Schwarz inequalities,
\begin{align}\label{eq:basic_decomp}
    W_2^2\parenth{\Hat{\mu}_{ij}, \mu_{ij}}
    \leq 2\big[\underbrace{ W_2^2\parenth{\Hat{\mu}_{ij}, \Bar{\mu}_{ij}}}_{\defeq \mathcal{V}\text{ (variance)}} + \underbrace{ W_2^2\parenth{\Bar{\mu}_{ij}, \mu_{ij}}}_{\defeq \mathcal{B}\text{ (bias)}}\big].
\end{align}
The variance $\mathcal{V}$ is exactly the empirical-vs-population barycenter error of \cref{prop:barycenters} (conditioning on the latent factors makes the neighborhood a fixed collection), and the bias $\mathcal{B}$ is controlled by neighborhood concentration. This yields our first guarantee.

\begin{theorem}[Rate of error decay for $\Hat{\mu}_{ij}$]\label{thm:main-asymptotic}
    Let \cref{assum:latent-factors-lipschitz}, \cref{assum:mcar}, and \cref{assum:regular} hold. Let $\vert\nn{\eta}{i}\vert$ be the number of neighbors for row $i$ with threshold $\eta$, let $\numrows$ and $\{n_v\}_{v \neq j}$ be fixed, and (without loss of generality) $\abss{\nn{\eta}{i}} \geq 1$. Then, as ${n_j, \numcols \to \infty}$,
    \begin{align}\label{eq:main-bound}
        W_2^2(\Hat{\mu}_{ij},\mu_{ij}) = \mathcal{O}_p\parenth{\eta + \frac{1}{p \sqrt{\numcols}} + \frac{1}{n_j \abss{\nn{\eta}{i}}} + \frac{\log^2 n_j}{n_j^2} }.
    \end{align}
\end{theorem}

The first two terms are the bias from estimating the latent-factor distance between rows. The $p\sqrt{\numcols}$ factor is a high-probability lower bound on the number of co-observed columns. The last two terms are the variance, inherited from \cref{prop:barycenters}. Unlike scalar nearest-neighbor methods \citep{li2019,dwivedi2022counterfactual,dwivedi2022doubly}, which require the number of neighbors to grow to infinity for consistency, here the error vanishes as $n_j\to\infty$ even at fixed neighborhood size, because more samples sharpen the column-$j$ distributions directly. The residual $\log^2 n_j/n_j^2$ floor at finite $n_j$ is the barycenter discretization bias discussed in \cref{rmk:functional-clt-regime}.

Specializing to the location-scale examples gives consistency in the data parameters alone. Let $\Tilde{\order}_p$ suppress logarithmic factors and constants besides $(\numcols,\numrows,p,n_j,d)$.

\begin{corollary}[Location-scale consistency]\label{cor:location-scale}
    Let \cref{assum:latent-factors-lipschitz}, \cref{assum:mcar}, and \cref{assum:regular} hold and $p$ be fixed. For $v\neq j$, let the per-cell sample sizes be fixed with minimum $n_{-j}$, let the latent dimension be $d$, and let $\distThreshold \asymp \tfrac{1}{n_{-j}} + \tfrac{1}{{(\numrows p)}^{2/(d+2)}} + \sqrt{\tfrac{\log \numrows}{\numcols p^2}}$, up to constants depending only on $d$ and $L$. Assume in addition $\log \numrows = \order(\log n_j)$. Then:
    \begin{enumerate}[(a)]
        \item For \cref{ex:location-scale-homo}, as $n_j, \numrows,\numcols \to \infty$,
        \begin{align}\label{eq:w2_bound_cor_1}
            W_2^2(\Hat{\mu}_{ij},\mu_{ij})
            &= \Tilde{\order}_p\parenth{\frac{1}{n_{-j}} + \frac{1}{{(\numrows p)}^{\frac{2}{d+2}}} + \frac{1}{p\sqrt{\numcols}} + \frac{1}{n_j{(\numrows p)}^{\frac{2}{d+2}}} + \frac{\log^2 n_j}{n_j^2}}.
        \end{align}
        \item For \cref{ex:location-scale-hetero}, the same bound holds with $d = 1$.
    \end{enumerate}
\end{corollary}

The term ${(\numrows p)}^{-2/(d+2)}$ is the neighborhood-radius bias, and ${(\numrows p)}^{2/(d+2)}$ is a high-probability lower bound on the number of neighbors. Like other nonparametric nearest-neighbor results, both suffer the curse of dimensionality in $d$. The added condition $\log\numrows = \order(\log n_j)$ is mild, and is vacuous in \cref{thm:main-asymptotic} where $\numrows$ is fixed. It ensures \cref{prop:barycenters} holds with constants uniform over the growing neighborhood. We prove \cref{thm:main-asymptotic} in \cref{sec:proof-techniques} and \cref{cor:location-scale} in \cref{sec:corr-proofs}.

\subsection{Two-sided and doubly robust estimators}\label{subsec:two-sided-dr}

The single-sided estimator utilizes information only across rows, within column $j$. When $f$ is also smooth in its column argument (\cref{assum:latent-factors-lipschitz-col}), we can utilize information across columns as well. We mirror Step~1 along columns, using rows other than $i$ that are observed in both columns $j$ and $v$ to form a column distance and an $\eta_c$-nearest column set $\nn{\eta_c}{j}^c$; excluding row $i$ mirrors the exclusion of column $j$ in \cref{eq:shared-col-def}. This gives three population barycenters of the target-anchored distributions,
\begin{align}\label{eq:three-barycenters}
    \quantile{\bar\mu^{\textnormal{row}}_{ij}} &= \frac{1}{|\nn{\eta_r}{i}^r|}\sum_{u\in\nn{\eta_r}{i}^r}\!\!\quantile{\mu_{uj}},
    \qquad
    \quantile{\bar\mu^{\textnormal{col}}_{ij}} = \frac{1}{|\nn{\eta_c}{j}^c|}\sum_{v\in\nn{\eta_c}{j}^c}\!\!\quantile{\mu_{iv}}, \text{ and} \\
    \quantile{\bar\mu^{\textnormal{both}}_{ij}} &= \frac{1}{|\nn{\eta_r}{i}^r||\nn{\eta_c}{j}^c|}\!\!\sum_{\substack{u\in\nn{\eta_r}{i}^r\\ v\in\nn{\eta_c}{j}^c}}\!\!\!\quantile{\mu_{uv}},
\end{align}
where $\nn{\eta_r}{i}^r \defeq \nn{\eta_r}{i}$ is the row neighborhood, and empirical counterparts $\hat\mu^{\textnormal{row}}_{ij}, \hat\mu^{\textnormal{col}}_{ij}, \hat\mu^{\textnormal{both}}_{ij}$ replace each $\quantile{\mu_{\cdot}}$ with $\quantile{Y_{\cdot}}$.

\paragraph{Two-sided estimator.} Combine the row and column neighborhoods into a single barycenter,
\begin{align}\label{eq:two-sided-est}
    \quantile{\hat\mu^{\textnormal{2s}}_{ij}}
    \defeq \frac{1}{|\nn{\eta_r}{i}^r| + |\nn{\eta_c}{j}^c|}
    \bigg(\sum_{u\in\nn{\eta_r}{i}^r}\!\!\quantile{Y_{uj}} + \sum_{v\in\nn{\eta_c}{j}^c}\!\!\quantile{Y_{iv}}\bigg).
\end{align}
Being a \emph{linear} average of quantile functions, this is itself a Wasserstein barycenter, so the functional CLT of \cref{sec:functional-clt} applies to it as well. Pooling enlarges the effective neighbor count, reducing variance. Next, we discuss a doubly robust estimator that combines the three barycenters in a different way.

\paragraph{Doubly robust estimator.} Combine the three barycenters and project back onto valid quantile functions. With the raw combination
\begin{align}\label{eq:dr-raw}
    \quantile{\hat\mu^{\textnormal{raw}}_{ij}}
    \defeq \quantile{\hat\mu^{\textnormal{row}}_{ij}} + \quantile{\hat\mu^{\textnormal{col}}_{ij}} - \quantile{\hat\mu^{\textnormal{both}}_{ij}},
\end{align}
which need not be monotone, apply the \emph{monotone rearrangement} operator $\mathcal{R}$ \citep{hardy1952inequalities,chernozhukov2010quantile} and set
\begin{align}\label{eq:dr-est}
    \quantile{\hat\mu^{\textnormal{DR}}_{ij}} \defeq \mathcal{R}\big(\quantile{\hat\mu^{\textnormal{raw}}_{ij}}\big).
\end{align}
Subtracting the ``both'' term cancels the row and column main effects, leaving only the second-order interaction \cref{eq:mixed-difference}, which yields a \emph{product}-form bias. Thus, if either the row or column neighborhood is close, then the doubly robust estimator is close to the target.

\begin{theorem}[Rate for the two-sided estimator]\label{thm:two-sided-rate}
    Let \cref{assum:latent-factors-lipschitz,assum:latent-factors-lipschitz-col}, \cref{assum:mcar}, and \cref{assum:regular} hold, assume for simplicity that every pooled entry contributes $n$ samples (\cref{sec:setup}; the statement with $\min_v n_v$ also holds), assume $\log(\numrows\numcols) = \order(\log n)$, and assume the ratio constants $c_f$ and $c_f^{\mathrm{col}}$ of \cref{eq:cf-def} are finite. Then, as $n, \numrows, \numcols \to \infty$,
    \begin{align}\label{eq:two-sided-rate}
        W_2^2\big(\hat\mu^{\textnormal{2s}}_{ij}, \mu_{ij}\big)
        = \order_p\!\bigg(
        \underbrace{\eta_r + \eta_c + \sqrt{\tfrac{\log(2\numrows)}{\numcols p^2}} + \sqrt{\tfrac{\log(2\numcols)}{\numrows p^2}}}_{\textnormal{bias}}
        + \underbrace{\tfrac{1}{n\,(|\nn{\eta_r}{i}^r| + |\nn{\eta_c}{j}^c|)} + \tfrac{\log^2 n}{n^2}}_{\textnormal{variance}}\bigg).
    \end{align}
\end{theorem}

\begin{theorem}[Rate for the doubly robust estimator]\label{thm:dr-rate}
    Let \cref{assum:latent-factors-lipschitz,assum:latent-factors-lipschitz-col,assum:latent-factors-bilinear-lipschitz}, \cref{assum:mcar}, and \cref{assum:regular} hold, assume for simplicity that every pooled entry contributes $n$ samples (\cref{sec:setup}; the statement with $\min_v n_v$ also holds), assume $\log(\numrows\numcols) = \order(\log n)$, assume the ratio constants $c_f$ and $c_f^{\mathrm{col}}$ of \cref{eq:cf-def} are finite, and let the row and column neighborhoods be computed on an independent fold (cross-fitting). Then, as $n, \numrows, \numcols \to \infty$,
    \begin{align}\label{eq:dr-rate}
        W_2^2\big(\hat\mu^{\textnormal{DR}}_{ij}, \mu_{ij}\big)
        = \order_p\!\bigg(
        \underbrace{\Big(\eta_r + \sqrt{\tfrac{\log(2\numrows)}{\numcols p^2}}\Big)\Big(\eta_c + \sqrt{\tfrac{\log(2\numcols)}{\numrows p^2}}\Big)}_{\textnormal{product bias}}
        + \underbrace{\tfrac{1}{n\,|\nn{\eta_r}{i}^r|} + \tfrac{1}{n\,|\nn{\eta_c}{j}^c|} + \tfrac{\log^2 n}{n^2}}_{\textnormal{variance}}\bigg).
    \end{align}
\end{theorem}

The two estimators improve on single-sided \distnn\ in complementary ways. The two-sided estimator (\cref{thm:two-sided-rate}) leaves the bias order unchanged (the larger of the row and column biases) but lowers the variance by pooling, so its effective neighbor count is $|\nn{\eta_r}{i}^r| + |\nn{\eta_c}{j}^c|$. The doubly robust estimator (\cref{thm:dr-rate}) instead addresses the bias: main-effect cancellation makes the residual bias the \emph{product} of the two one-sided bias terms, the distributional analogue of the double robustness of \citet{dwivedi2022doubly}, accurate if \emph{either} neighborhood is close to the target entry. The monotone rearrangement \cref{eq:dr-est} is a nonlinear projection, so the functional CLT does not transfer to the doubly robust estimator. We discuss bootstrap inference for it in \cref{sec:discussion} and prove both theorems in \cref{sec:proof-two-sided-dr}\inEC.

\subsection{Confidence bands}\label{sec:nn-inference}

For a single entry, the functional CLT of \cref{sec:functional-clt} specializes to a clean limit. As the neighborhood shrinks, $\eta_\numcols\to 0$, all neighbor distributions converge to the target and their quantile densities converge uniformly (\cref{lem:quantile-density-continuity}), collapsing the kernel \cref{eq:clt-covariance} to a single transformed Brownian bridge.

\begin{theorem}[Functional CLT for the nearest-neighbor estimator]\label{cor:nn-functional-clt}
    Let \cref{assum:latent-factors-lipschitz}, \cref{assum:mcar}, and \cref{assum:regular} hold, assume the ratio constant $c_f$ of \cref{eq:cf-def} is finite, and let the sequence ${\{n_{j,\numcols},\numrows_\numcols,\eta_\numcols,\nn{\eta_\numcols}{i}\}}_{\numcols=1}^\infty$ satisfy $\eta_\numcols \to 0$ and
    \begin{align}\label{eq:bb-assum1}
        \sqrt{\abss{\nn{\eta_\numcols}{i}}}\frac{\log n_{j,\numcols}}{\sqrt{n_{j,\numcols}}} = o_p(1), \quad \text{and}\quad \sqrt{n_{j,\numcols}\, \abss{\nn{\eta_\numcols}{i}}}\, \bigg(\eta_\numcols + \sqrt{\tfrac{\log(2\numrows_\numcols)}{\numcols p^2}}\bigg)^{\!1/4} = o_p(1).
    \end{align}
    Write $w_\infty \defeq 1/(f_{\mu_{ij}}\circ\quantile{\mu_{ij}})$. Then there exists a standard Brownian bridge $\mbb{B}$, on a common probability space with the samples, such that
    \begin{align}\label{eq:nn-functional-clt}
        \lnorm{\sqrt{n_{j,\numcols}\abss{\nn{\eta_\numcols}{i}}}\,
            \big(\quantile{\Hat{\mu}_{ij}} - \quantile{\mu_{ij}}\big) - w_\infty\,\mbb{B}}
        = o_p(1) \qquad \text{as } \numcols \to \infty.
    \end{align}
\end{theorem}

The limit \cref{eq:nn-functional-clt} is scaled by the unknown population quantile density $w_\infty$; for feasible inference we replace it with an estimate. The relevant per-level scale averages the empirical quantile fluctuations over the selected neighbors,
\begin{align}\label{eq:neighbor-sigma}
    \sigma^2_{\nn{\eta}{i}}(t) &\defeq \frac{1}{\abss{\nn{\eta}{i}}} \sum_{u \in \nn{\eta}{i}} \frac{t-t^2}{f_{\mu_{uj}}^2 (\quantile{\mu_{uj}}(t))},
\end{align}
which is not directly observable, so we estimate it by a finite difference of the neighbors' observed empirical quantile functions $\quantile{Y_{uj}}$: given a bandwidth $h_\numcols > 0$,
\begin{align}\label{eq:sigma-hat}
    \widehat{\sigma}^2_{\nn{\eta_\numcols}{i}}(t) \defeq \frac{t(1-t)}{\abss{\nn{\eta_\numcols}{i}}} \sum_{u \in \nn{\eta_\numcols}{i}} \bigg(\frac{\quantile{Y_{uj}}(t + h_\numcols) - \quantile{Y_{uj}}(t - h_\numcols)}{2 h_\numcols}\bigg)^{\!2}.
\end{align}
Studentizing by $\widehat\sigma$ gives a feasible pointwise limit. At a single quantile level this yields a sharper interval than the simultaneous band below.

\begin{corollary}[Feasible pointwise normality of $\quantile{\Hat{\mu}_{ij}}(t)$]\label{thm:main-brownian-bridge-feasible}
    Assume the conditions of \cref{cor:nn-functional-clt}, and let the bandwidth $h_\numcols$ additionally satisfy
    \begin{align}\label{eq:bb-assum2}
        h_\numcols \to 0 \quad \text{and}\quad
        \frac{\log n_{j,\numcols}}{n_{j,\numcols}\, h_\numcols^2} = o_p(1).
    \end{align}
    Then, for almost all $t \in (0,1)$,
    \begin{align}\label{eq:normality-feasible}
        \frac{\sqrt{n_{j,\numcols}\, \abss{\nn{\eta_\numcols}{i}}}}{\widehat{\sigma}_{\nn{\eta_\numcols}{i}}(t)}\,
        \big(\quantile{\Hat{\mu}_{ij}}(t) - \quantile{\mu_{ij}}(t)\big)
        \overset{d}{\to} \mathcal{N}(0,1) \quad \text{as} \quad \numcols\to\infty.
    \end{align}
\end{corollary}

To turn this process limit into a confidence band, we \emph{studentize} it: we divide the centered estimate by the estimated scale $\widehat\sigma$. This removes the unknown quantile density, so the limit becomes \emph{pivotal}, meaning its distribution no longer depends on the target distribution \citep{casella2002statistical}. A single critical value, computed once from this universal limit, then produces a band that holds \emph{simultaneously} across quantile levels. We work on an interior interval $[\delta,1-\delta]$, since the studentized limit blows up at the endpoints $t\to 0,1$.

\begin{corollary}[Simultaneous confidence band]\label{cor:confidence-band}
    Fix $\delta \in (0,\tfrac12)$ and $\alpha \in (0,1)$. Assume the conditions of \cref{cor:nn-functional-clt} and that the bandwidth $h_\numcols$ satisfies \cref{eq:bb-assum2}, under which the difference-quotient estimator $\widehat{\sigma}_{\nn{\eta_\numcols}{i}}(t)$ of \cref{eq:sigma-hat} is uniformly consistent for $\sigma_{\nn{\eta_\numcols}{i}}(t)$ on $[\delta, 1-\delta]$ (established in the proof of \cref{thm:main-brownian-bridge-feasible}):
    \begin{align}\label{eq:uniform-sigma-consistency}
        \sup_{t\in[\delta,1-\delta]}
        \abss{\frac{\widehat{\sigma}_{\nn{\eta_\numcols}{i}}(t)}{\sigma_{\nn{\eta_\numcols}{i}}(t)} - 1}
        = o_p(1).
    \end{align}
    Let $c_\alpha(\delta)$ be the $(1-\alpha)$-quantile of $\sup_{t\in[\delta,1-\delta]} \abss{\mbb{B}(t)}/\sqrt{t(1-t)}$, a universal constant depending only on $\delta,\alpha$. Then
    \begin{align}\label{eq:confidence-band}
        \mc{C}_\numcols(t) \defeq \sbraces{\,\quantile{\Hat{\mu}_{ij}}(t)
            \pm c_\alpha(\delta)\,
            \frac{\widehat{\sigma}_{\nn{\eta_\numcols}{i}}(t)}
                 {\sqrt{n_{j,\numcols}\abss{\nn{\eta_\numcols}{i}}}}\,}
    \end{align}
    satisfies $\probc*{\quantile{\mu_{ij}}(t) \in \mc{C}_\numcols(t) \text{ for all } t\in[\delta,1-\delta]} \to 1-\alpha$ as $\numcols\to\infty$.
\end{corollary}
The band \cref{eq:confidence-band} provides uncertainty estimates to the \emph{entire} imputed quantile function at once, not just to a single level. 
% Its critical value $c_\alpha(\delta)$ is $\order(1)$ in the grid resolution, so the band stays tight however finely the quantile grid is sampled. A Bonferroni correction, on the other hand, widens with the number of levels (\cref{rmk:band-sharpness}). 
From one imputed distributional entry, a practitioner can thus make valid statements about quantities including a quantile, a value-at-risk, or, by inverting the band, a threshold-crossing probability such as the hypotension burden of \cref{sec:empirical}. We confirm the coverage and compare the pivotal, plug-in, and bootstrap critical values in \cref{sec:appendix-bands}. We prove \cref{cor:nn-functional-clt} in \cref{sec:proof-nn-functional-clt}, \cref{thm:main-brownian-bridge-feasible} in \cref{sec:proof-main-brownian-bridge-feasible}, and \cref{cor:confidence-band} in \cref{sec:proof-confidence-band}.

%% file: sim.tex
\section{Numerical results}\label{sec:sim}
Here we show through simulations that \distnn\ can recover missing distributions, reproducing error decay rates close to those proven in \cref{sec:nn}. The Python code to run our tests is available on GitHub at \ifbool{anonymous}{\url{https://anonymous.4open.science/r/Dist-NN-D7D7/README.md}}{\url{https://github.com/jacobf18/Dist-NN}}.

\subsection{The multiplicative barycenter rate}\label{sec:sim-barycenter}

We first verify the multiplicative convergence rate of \cref{prop:barycenters} directly, in isolation of the matrix-completion application. For a collection of $k$ regular probability distributions $\{\nu_j\}$ observed through $n$ samples each (the generic barycenter setting of \cref{sec:barycenters}, where $\nu$ denotes arbitrary distributions rather than the matrix entries $\mu_{ij}$), the proposition predicts $W_2^2(\what{\overline\nu}_n, \overline\nu) = \order_p(\tfrac{1}{nk} + \tfrac{\log^2 n}{n^2})$ for the empirical and population barycenters $\what{\overline\nu}_n$ and $\overline\nu$. So, the error decay in the per-distribution sample size $n$ should steepen from roughly $n^{-1}$ when $k$ is small toward $n^{-2}$ as $k$ grows. This contrasts with the additive general-class rate $\order(\tfrac1k + \tfrac1n)$ of \citet{bigot2018upper}, under which the $n$-exponent would remain $-1$ regardless of $k$.

We test on \emph{truncated normal} distributions, taking the standard normal truncated to $[-2,2]$ as the base of a location-scale family. \citet{bigot2018upper} already compute the uniform, exponential, and Gaussian barycenter errors in closed form; the truncated normal is not among them, but it is regular in the sense of \cref{def:regular-measure}, satisfying the curvature bound \cref{eq:regularity-F} with a supremum of $0.52$, well inside the required $2$. (Untruncated Gaussians are excluded because their density is not bounded below on a compact support.) Because quantile functions average pointwise, the true barycenter is again a truncated normal with the averaged location and scale, giving an exact target.

For each $k \in \{1, 5, 25, 100\}$ we draw $k$ such distributions with locations from $\Unif(-5,5)$ and scales from $\Unif(1,5)$, form the empirical barycenter by averaging order statistics \cref{eq:barycenter-order-stats}, and measure its squared 2-Wasserstein distance to the true barycenter, averaging over 50 runs.\ \cref{fig:barycenter-rate} plots the result. As predicted, the fitted exponent in $n$ steepens monotonically with $k$, from $n^{-0.97}$ at $k=1$ to $n^{-1.27}$ at $k=100$, confirming the multiplicative phenomenon on a regular probability distribution for which no distribution-specific rate was available.

\begin{figure}[ht!]
    \centering
    \includegraphics[width=0.35\linewidth]{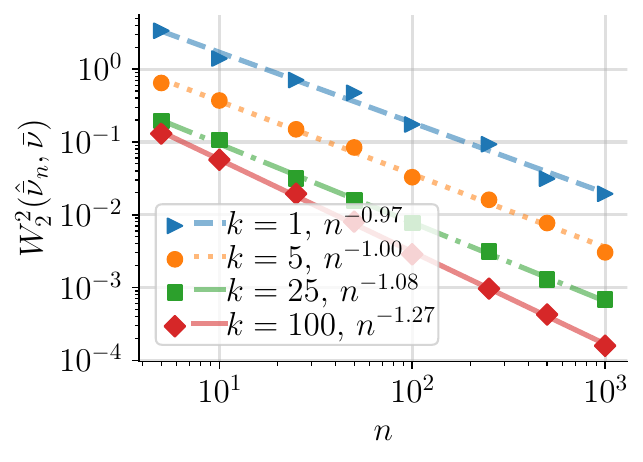}
    \caption{\tbf{Empirical Wasserstein-barycenter error versus per-distribution sample size $n$, for several numbers of distributions $k$.} Each curve fixes $k$ and plots the squared 2-Wasserstein distance between the empirical barycenter (of $k$ truncated-normal distributions sampled $n$ times each) and the true barycenter, averaged over 50 runs, with a least-squares power-law fit $f(n) = a\,n^b$. The fitted exponent $b$ (shown in the legend) steepens from $-0.97$ at $k=1$ toward $-1.27$ at $k=100$, the signature of the multiplicative rate $\tfrac{1}{nk} + \tfrac{\log^2 n}{n^2}$ of \cref{prop:barycenters}: increasing $k$ accelerates the decay in $n$, which the additive bound of \citet{bigot2018upper} does not capture.}\label{fig:barycenter-rate}
\end{figure}

\subsection{Matrix completion simulations}
Here, we simulate the location-scale setting described in \cref{ex:location-scale-hetero}. In these simulations, we let the sample size per matrix entry be $n$ and thus equal across columns. In \cref{fig:error-samples-rows}, we show error decay rates with respect to the number of samples $n$, the number of rows $\numrows+1$, and the product of $n$ and the number of neighbors in the Gaussian location-scale case. We test the Gaussian case to connect with previous matrix completion literature, which primarily considers sub-Gaussian noise. We also test our method on continuous uniform distributions, in \cref{fig:random-samples} and in the confidence-band study of \cref{sec:appendix-bands}.

Each matrix entry is drawn from the same distribution family, with only the location or scale differing between entries, as in \cref{ex:location-scale-hetero}. We draw locations from $\Unif(-5,5)$ and scales from $\Unif(1,5)$. \emph{Without loss of generality}, our experiments have a single missing entry: nearest neighbors estimates one entry at a time, so estimating one entry is enough to estimate them all.

We choose the threshold $\eta$ by cross-validation on observed entries: to estimate $\mu_{ij}$ we hold out each observed entry of row $i$ in turn, re-estimate it, and keep the $\eta$ minimizing the squared Wasserstein distance to the held-out entry. Since the objective is nonconvex, we minimize it with the Tree of Parzen Estimators \citep{bergstra2011algorithms} as implemented in \texttt{hyperopt}, using its default settings and at most $50$ iterations.

\paragraph{Error with respect to number of samples, $n$} In \cref{fig:error-samples-rows}, the estimation error drops rapidly as the number of samples $n$ and the number of rows $\numrows+1$ increase. Against the number of samples, the error decay rate steepens from about $\mathcal{O}(n^{-0.16})$ with few neighbors to $\mathcal{O}(n^{-1.02})$ with many. This matches \cref{thm:main-asymptotic}, where the dominant rate in $n$ improves toward $\mathcal{O}(\log^2 n / n^2)$ as the number of neighbors increases. The observed error rate with respect to $n \lvert\nn{\eta}{i}\rvert$ is around $\mathcal{O}(\parenth{n\abss{\nn{\eta}{i}}}^{-0.74})$, close to the asymptotic bound of $\mathcal{O}((n\abss{\nn{\eta}{i}})^{-1})$. 

\paragraph{Error with respect to number of rows, $\numrows$} Against the number of rows, the error again drops rapidly, and the decay rate beats what our theory predicts. The dotted line shows the expected error of a single observed sample, which Dist-NN improves on by a wide margin. Even at just 20 rows, our error is already well below the expected error of an observed empirical distribution. So Dist-NN gives a much better estimate of the true distribution than an entry's own observed samples. We call this \emph{denoising}, since it mirrors the denoising ability of scalar nearest neighbors in noisy matrix completion.

\begin{figure}[ht!]
    \centering
    \resizebox{\linewidth}{!}{
        \begin{tabular}{ccc}
            \includegraphics[width=0.33\linewidth]{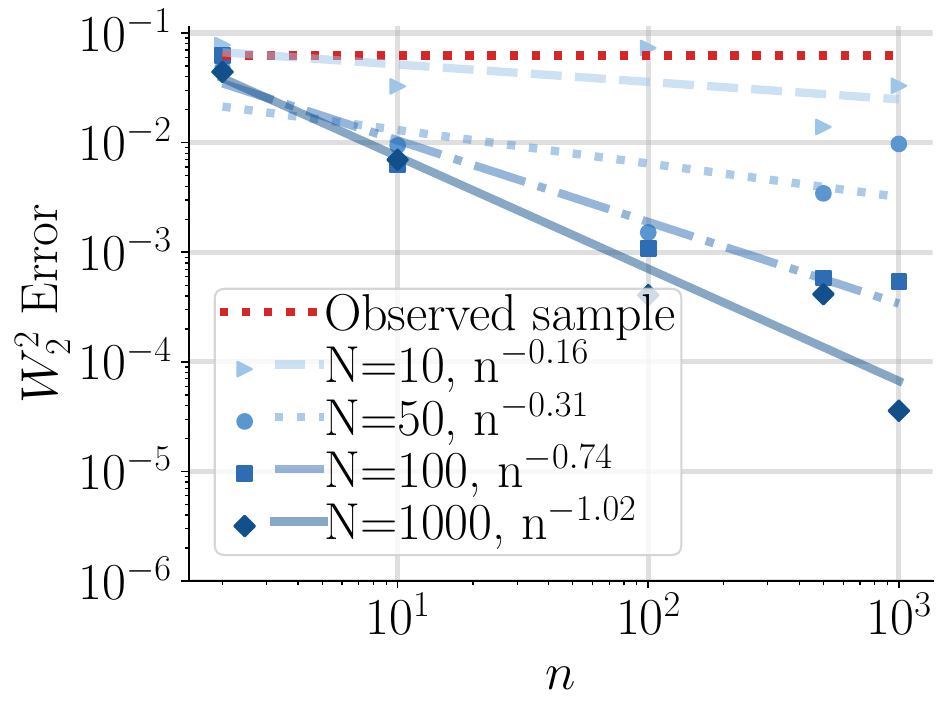} &
            \includegraphics[width=0.33\linewidth]{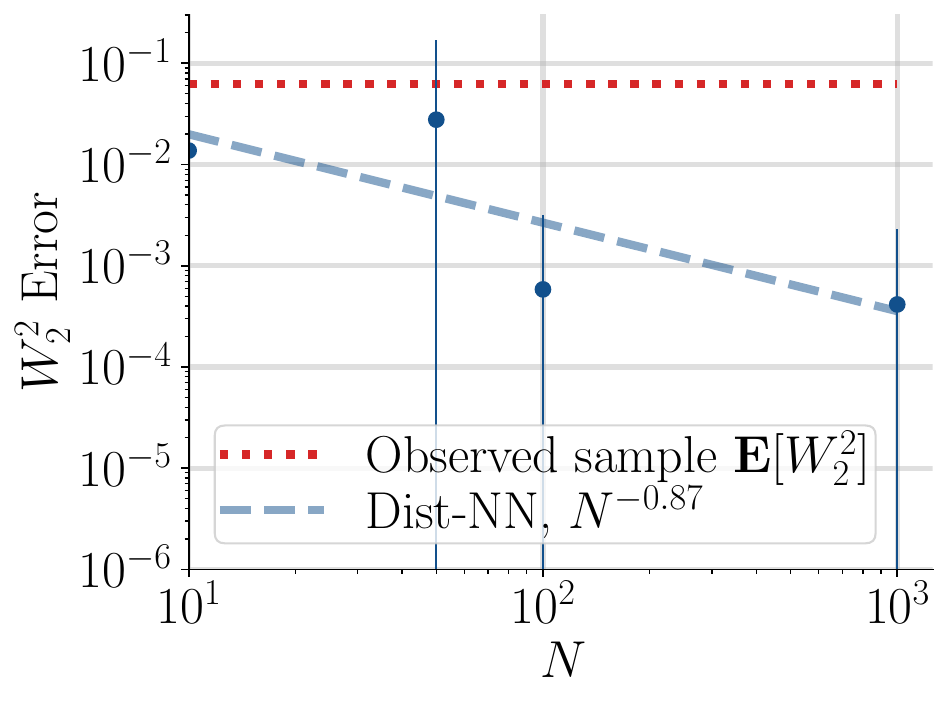} &
            \includegraphics[width=0.33\linewidth]{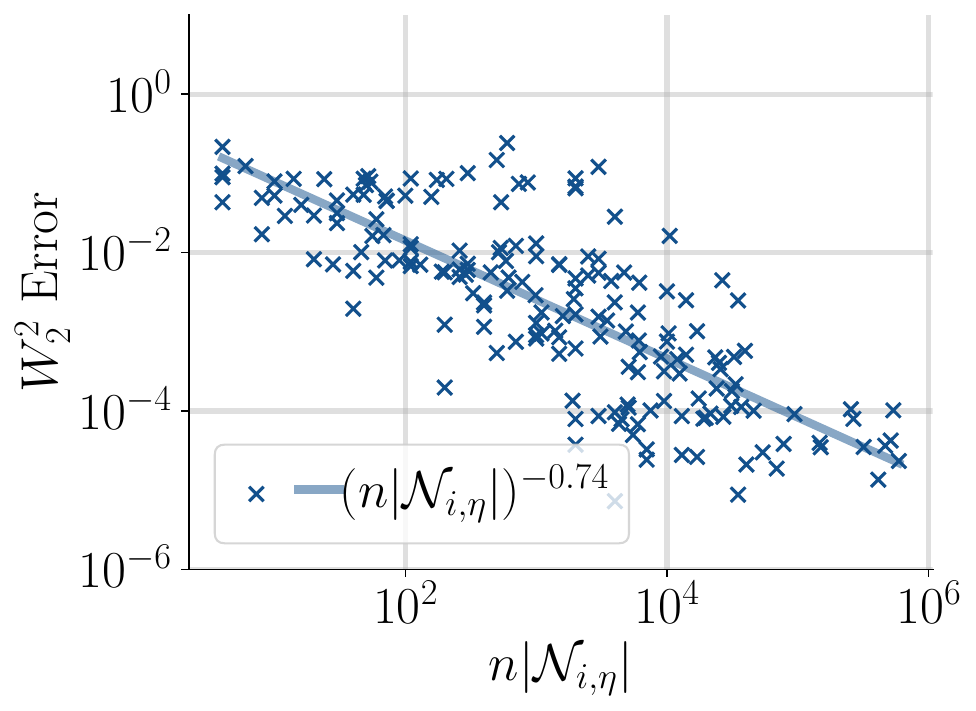}
            \\
            \parbox[c]{0.33\linewidth}{\centering (a) Scaling with \# samples $n$} & 
            \parbox[c]{0.33\linewidth}{\centering (b) Scaling with \# rows $\numrows$} & 
            \parbox[c]{0.33\linewidth}{\centering (c) Scaling with $n\abss{\nn{\eta}{i}}$}
        \end{tabular}
    }
    \caption{\tbf{Scaling of error with sample size $n$, number of rows $\numrows$, and effective sample size $n\abss{\nn{\eta}{i}}$.} Every distribution in the matrix is Gaussian: each row's mean is drawn from $\Unif(-5,5)$ and each column's standard deviation from $\Unif(1,5)$, with $30$ columns throughout. Plot (a) varies the per-entry sample size $n$ at several row counts $\numrows$ (requiring at least $2$ nearest neighbors). Plot (b) varies $\numrows$ with $n=500$ samples per entry. Plot (c) pools all settings and plots the error against the effective sample size $n\abss{\nn{\eta}{i}}$. In plots (a) and (b), the dotted line is the expected error of a single observed sample in one matrix entry, estimated over $100$ draws. We simulate each setting $10$ times, and fit each curve by least squares to the power law $f(x)=a x^b$.}\label{fig:error-samples-rows}
\end{figure}

\paragraph{Simultaneous confidence bands} We validate the simultaneous confidence band of \cref{cor:confidence-band} empirically in \cref{sec:appendix-bands}. Across the Gaussian and continuous-uniform location-scale cases, the pivotal, plug-in, and bootstrap critical values are all smaller than the Bonferroni factor $z_{\alpha/(2m)} \asymp \sqrt{2\log m}$, so the bands are uniformly narrower. They attain close to nominal simultaneous coverage once the sample size is moderate, under-covering only at the smallest samples, due to empirical-quantile boundary bias.

\paragraph{Denoising} \cref{fig:random-samples} shows the method's denoising ability through empirical CDFs. By sharing information across rows, nearest neighbors produces empirical distributions much closer to their true distributions than a single observed sample does. This helps downstream analysis: distributional quantities such as the mean, variance, and value-at-risk can then be estimated far more accurately than from an isolated set of samples. Value-at-risk (VaR) is a representative tail metric, used in financial risk analysis and, as a threshold-crossing burden, in clinical monitoring. It is defined for $\alpha \in (0,1)$ as $\mathrm{VaR}_X(\alpha) = F^{-1}_{-X}(1 - \alpha)$ where $X$ is a random variable and $F^{-1}_{-X}$ is the quantile function of $-X$.

\cref{fig:mean-std-var} shows that our method estimates distributional quantities well. Dist-NN produces estimates of the mean, median, standard deviation, and value-at-risk much closer to their true values than the observed samples in a single cell do. So, running \distnn\ on observed distributions can give much better estimates of their true distributions than using their observed samples in isolation. New distributional quantities can be computed from our method's output without re-running the estimation procedure, because the output is a full distribution. Thus, our method estimates all of these quantities at once.

\begin{figure}[ht!]
    \centering
    \includegraphics[width=1.0\linewidth]{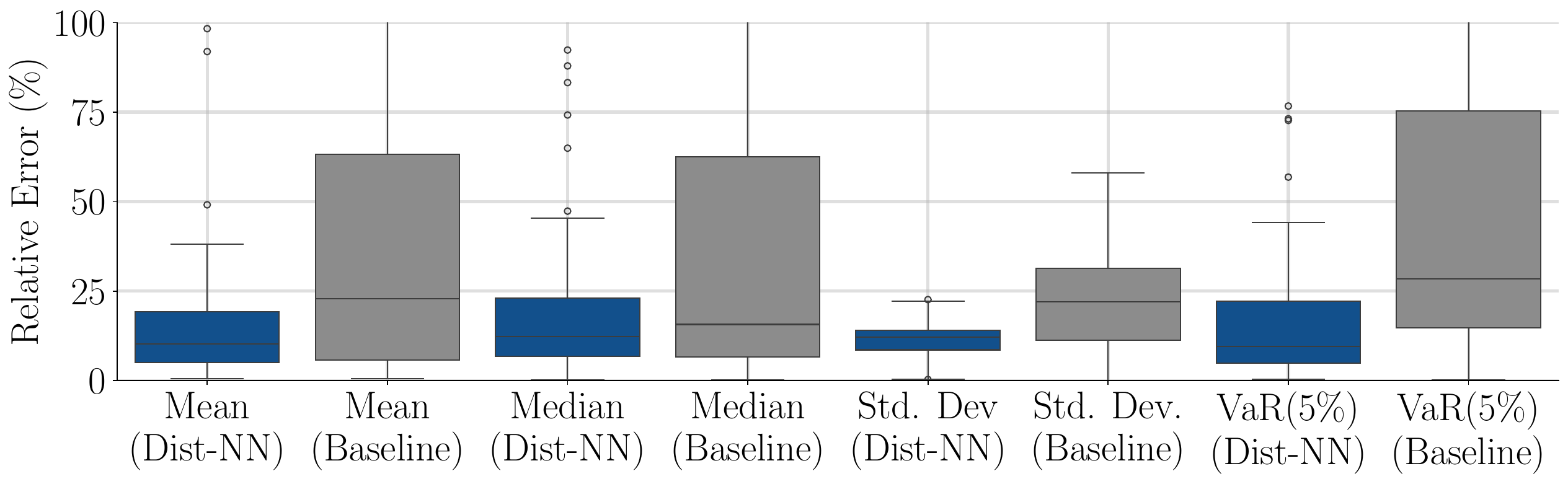}
    \caption{\tbf{Relative error for \distnn\ and a baseline of only using the distribution within a single matrix cell for estimating means, medians, standard deviations, and value-at-risk (VaR(5\%)).}  We use the same Gaussian location-scale setup as in \cref{fig:error-samples-rows}. \distnn\ is able to estimate all of these distributional quantities noticeably better than the baseline. Thus, utilizing shared information across the matrix through \distnn\ helps to estimate all of these distributional quantities simultaneously. The box and whisker plots show the three quartiles in the shaded boxes, the range of the data, and any outliers. Note that we cut off the $y$-axis at 100\% for readability.}\label{fig:mean-std-var}
\end{figure}

%% file: empirical-hirid.tex
\section{Real-world application: intensive-care vital-sign monitoring}\label{sec:empirical}

We next apply \distnn\ to the imputation of intensive-care monitoring records, the setting from \cref{sec:motivating-example}. Instruments record a patient's vital signs at high frequency throughout an ICU stay, but the record has frequent gaps: leads and probes detach, sensors are removed during transport or procedures, and telemetry drops out, all for reasons unrelated to the patient's physiology. For retrospective analysis, we want to fill each gap with the vital sign's \emph{distribution} over the missing window, not just its mean, because the quantity of interest is often a tail. For example, the lower tail of mean arterial pressure gives the hypotension burden, which is associated with worse patient outcomes \citep{bakhsh2026impact}. Because the task is retrospective, we can use the full record, including later windows, to reconstruct an earlier gap.

\subsection{Data and matrix construction}\label{sec:hirid-data}

We use the HiRID high-resolution ICU database \citep{faltys2021hirid, hyland2020early}, which contains almost $34{,}000$ ICU admissions with vital signs recorded about every two minutes. HiRID is distributed by PhysioNet at \url{https://physionet.org/content/hirid/} under credentialed access. We work throughout with the merged stage of the release. We study two continuously monitored signals: heart rate (HR) and mean arterial pressure (MAP). For each signal we build a distributional matrix whose rows are patients and whose columns are consecutive four-hour windows measured from admission, keeping the first twelve windows (the first two days of the stay). Entry $(i,j)$ is the empirical distribution of patient $i$'s readings during window $j$. High-frequency readings are strongly autocorrelated within a window, so we subsample to about one reading per ten minutes. Then, each cell approximates an independent sample from the window's occupation measure. We treat a cell as observed when it has at least fifteen readings. We report results on a cohort of $2{,}000$ patients per vital, whose matrices are $64\%$ (MAP) and $67\%$ (HR) observed, since patients stay in the unit for different lengths of time.

\subsection{Task, metrics, and baselines}\label{sec:hirid-task}

To evaluate retrospective completion, we mask a fraction of the observed cells at random, impute each masked cell from the rest of the record, and score the imputed distribution against the held-out truth. The windows have equal length, so they contribute comparable per-cell sample counts, which we reduce to a single size by taking the minimum as in \cref{sec:setup}. The ICU task therefore meets the conditions behind our theory. We vary the masked fraction over $\{0.15, 0.30, 0.50, 0.70\}$ to see how the methods hold up as the record grows sparser. We score each imputed distribution by its squared $2$-Wasserstein error $W_2^2$ to the held-out window, along with two clinical metrics: the lower-tail quantile $q_{0.05}$ for HR, and the hypotension burden $\mathrm{P}(\mathrm{MAP} < 65)$ for MAP.

We compare the three \distnn\ estimators against two baselines. The \emph{population} baseline imputes the Wasserstein barycenter of the target window over all observed patients. It is a valid distribution, but it ignores patient similarity. \emph{Forward-fill} (last observation carried forward) is the official imputation provided with HiRID \citep{faltys2021hirid} and a common baseline for clinical time series. Applied to a missing window, forward-fill carries a single value across the entire gap, so it cannot represent spread or tails by construction. This is the information that distributional completion recovers.

These two baselines are non-learned, so we add a third: a Gaussian process (GP) time-series imputer, a natural competitor since a monitored vital is a time series. For a masked cell $(i,j)$ we fit a GP to patient $i$'s readings in every \emph{other} observed window, regressing value on reading time. Its kernel combines a short-range component for minute-to-minute fluctuation, with a slowly varying one for long-term drift, together with an observation-noise term. 
% We fit the hyperparameters by maximizing the marginal likelihood. 
We predict at the masked window's sampling times. The GP returns a Gaussian marginal $\N(\hat m(t), \hat s^2(t))$ at each time $t$ in the gap.
% , so the distribution it implies for the window is the average of those marginals over the window's sampling times, which we draw from and summarize on the same quantile grid as every other method. 
% We give the GP this full predictive distribution rather than its posterior mean alone, which would understate the spread. 
The GP uses the temporal structure of the patient's own trajectory but none of the cross-patient information.

\subsection{HiRID results}\label{sec:hirid-results}

We begin with a single masked window to illustrate the methods, then turn to the full sweep. Panel (a) of \cref{fig:hirid-collapse} shows a masked MAP window in which the patient was hypotensive for much of the window ($\mathrm{P}(\mathrm{MAP}<65)=0.52$). Forward-fill collapses it to a point mass at the last carried value, in the normotensive range, and the population barycenter is centered on the far higher pressures of the average patient, so each reports \emph{no} hypotension, with squared errors of $176.9$ and $159.8$. The GP does far better, at $9.1$, locating the episode but smoothing it out, and the two-sided \distnn\ estimate tracks the truth most closely, at $1.64$. Both distributional methods overstate the hypotensive mass here ($0.69$ for \distnn\ and $0.66$ for the GP, against a true $0.52$). Panel (b) shows aggregated errors at $50\%$ missingness. Both distributional methods are far more accurate than forward-fill and the population baseline on all three quantities. Between the two, \distnn\ is the more accurate on the lower quantile, while on the mean and the hypotension burden the difference is within noise.

\begin{figure}[ht!]
    \centering
    \begin{tabular}{cc}
        \includegraphics[width=0.42\linewidth]{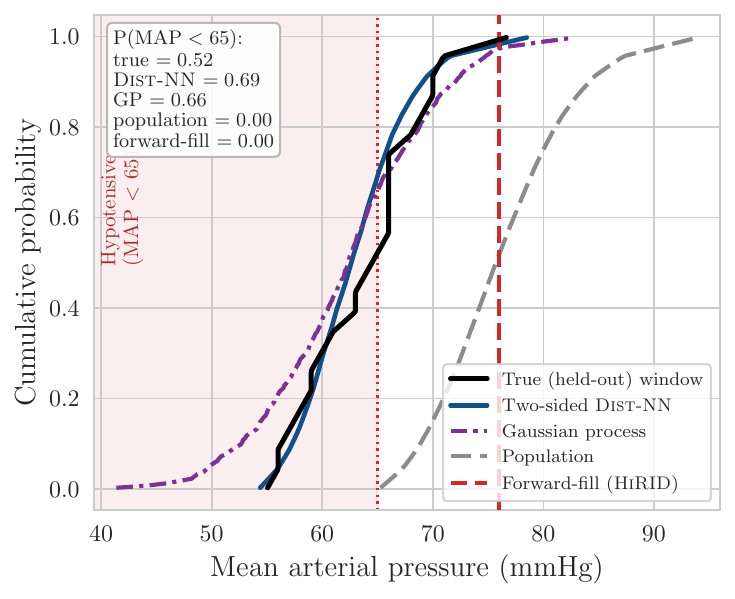} &
        \includegraphics[width=0.42\linewidth]{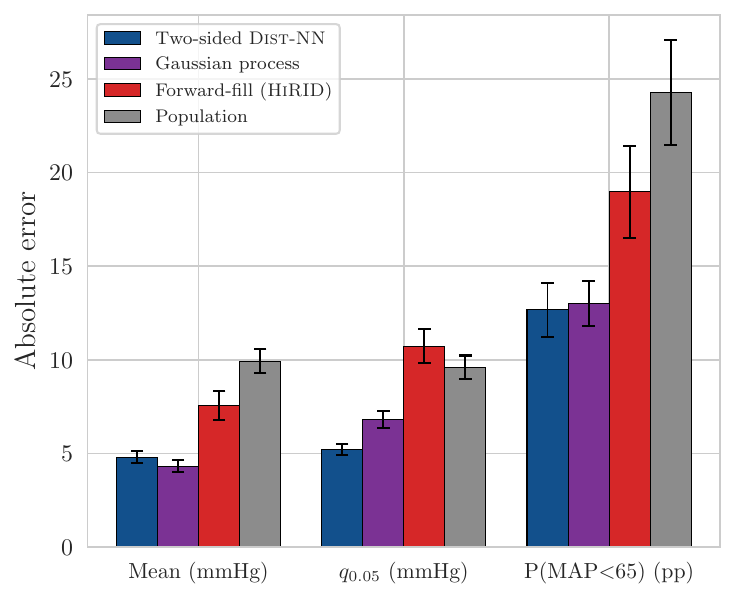}
        \\
        \parbox[c]{0.44\linewidth}{\centering (a) A masked MAP window} &
        \parbox[c]{0.44\linewidth}{\centering (b) Functional error at $50\%$ missing}
    \end{tabular}
    \caption{\tbf{\distnn\ recovers the shape of the distribution more closely than a time-series imputer or the non-learned baselines.} (a) A masked mean-arterial-pressure window: the true (held-out) distribution places substantial mass below the hypotension threshold of $65$ mmHg (shaded). The two-sided \distnn\ estimate recovers it closely (squared error $1.64$); the Gaussian process locates the episode but smooths it out ($9.1$); the population baseline (centered on the higher pressures of the average patient) and forward-fill (a point mass at the last carried value) miss it entirely ($159.8$ and $176.9$). The inset reports $\mathrm{P}(\mathrm{MAP}<65)$ under each; both distributional methods overstate it on this window. (b) Absolute error of the imputed mean, lower quantile $q_{0.05}$, and hypotension burden $\mathrm{P}(\mathrm{MAP}<65)$ (in percentage points) at $50\%$ missingness, over $150$ held-out cells with $\pm 1$ standard-error bars. \distnn\ and the GP are far more accurate than forward-fill and the population baseline throughout. \distnn\ is the more accurate of the two on the lower quantile; on the mean and the hypotension burden the two are within noise of each other. Forward-fill is especially poor on the tail and the hypotension burden, the quantities that a point mass cannot capture.}\label{fig:hirid-collapse}
\end{figure}

\cref{fig:hirid-missingness} reports the same comparison across the full range of missingness, for both vitals and with all six methods on the same masked cells. Forward-fill and the population baseline are the least accurate everywhere: two-sided \distnn\ improves on forward-fill by $32$-$73\%$ in mean $W_2^2$, and the gap on the clinical metrics (right column) is larger still, since a point mass cannot express spread or tails. The doubly robust estimator trails both, as expected in this variance-dominated regime (\cref{sec:sim-estimators}), so we recommend the two-sided estimator as the default for this task.

The two-sided estimator and the GP, by contrast, are close throughout, and which one leads depends on how much of the record survives. On heart rate the GP is slightly ahead while the trajectory is still dense, such as at $15\%$ and $30\%$ missingness, and \distnn\ overtakes GP once $70\%$ of the record is gone. For mean arterial pressure, the ordering reverses in the middle of the range, with \distnn\ ahead at $30\%$ and $50\%$. Both methods are scored on the same held-out cells, and the standard-error bars show that most of these differences are within noise. A within-patient GP has a lot of local signal to interpolate from when few windows are missing, and progressively less as the trajectory itself thins, whereas \distnn\ draws on other patients throughout and so degrades more gracefully.

% Two points are worth noting. First, a well-specified time-series model does much of this task well, and forward-fill, the imputation shipped with the dataset, is a weak yardstick: the GP roughly halves its error. Second, \distnn\ matches that model without using the time index at all. It never sees when a reading was taken, only which distributions sit in neighboring rows and columns, and it is therefore not restricted to entries that form a time series. That is what makes the comparison informative rather than merely favorable.

\begin{figure}[ht!]
    \centering
    \includegraphics[width=\linewidth]{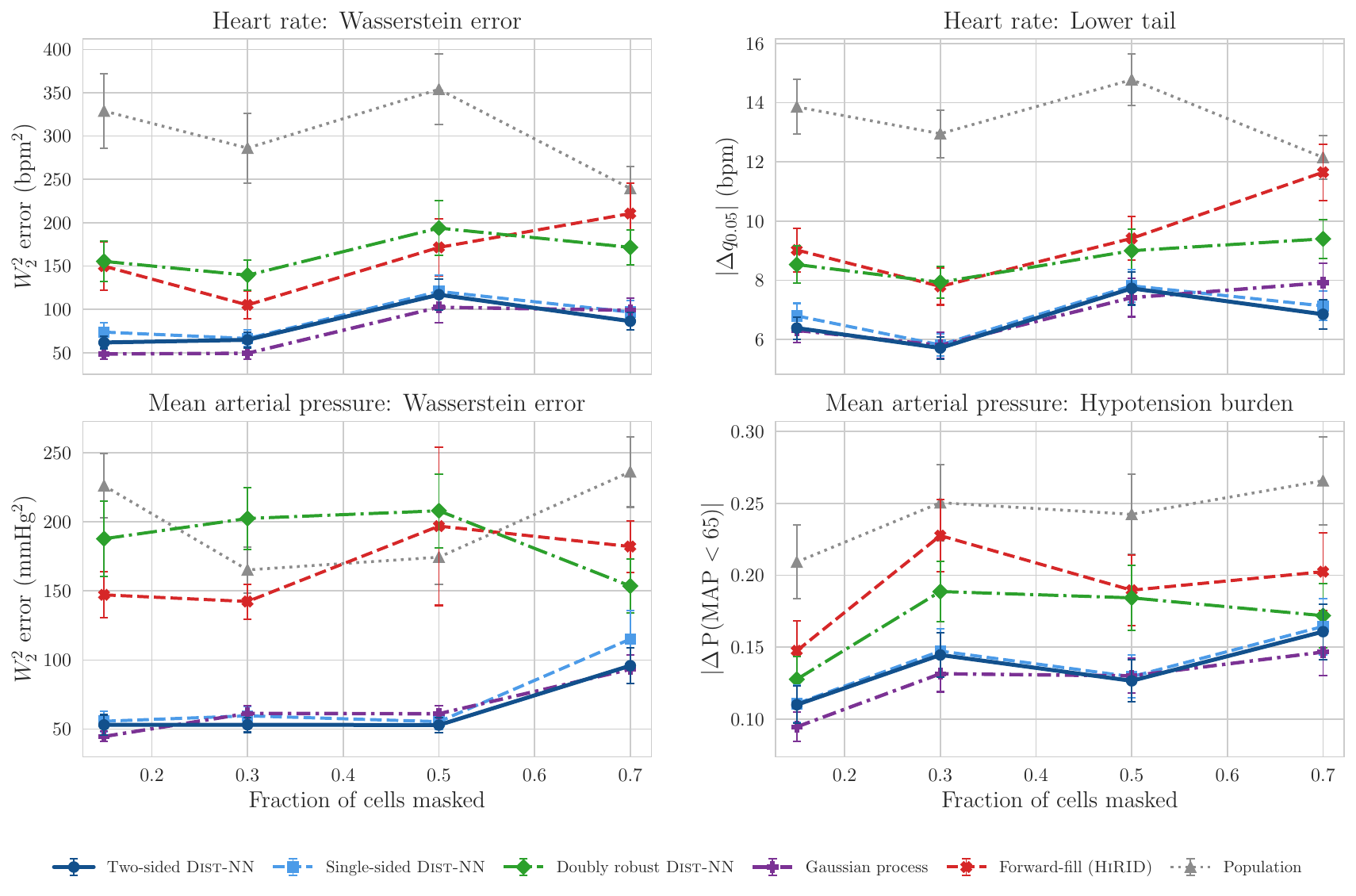}
    \caption{\tbf{Imputation error versus missingness on HiRID ICU data.} ($2{,}000$ patients. Error bars are $\pm 1$ standard error over $150$ held-out cells). Left: squared $2$-Wasserstein error to the held-out window. Right: the clinically salient functional for each vital, the heart-rate lower-tail quantile $q_{0.05}$ (top) and the mean-arterial-pressure hypotension burden $\mathrm{P}(\mathrm{MAP}<65)$ (bottom). The two-sided \distnn\ estimator and the Gaussian process are the two most accurate methods at every setting and are close to each other throughout, with the GP ahead on the denser records and \distnn\ ahead once the record thins. \distnn\ improves on forward-fill, HiRID's own imputation, by $32$-$73\%$ in $W_2^2$. The doubly robust estimator trails both, as expected in this variance-dominated regime of short, noisy windows, in which pooling helps and product-bias cancellation does not (cf.\ \cref{sec:sim-estimators}).}\label{fig:hirid-missingness}
\end{figure}

In summary, \distnn\ recovers the full monitoring distribution across sensor-dropout gaps, and in particular the hypotension-relevant tail, substantially more accurately than the imputation shipped with the dataset. And because \distnn\ returns a distribution, we can read a range of downstream quantities, including the mean, the variability, quantiles, and threshold-crossing burdens, from a single imputed object, with no need to re-run the procedure.

%% file: discussion.tex
\section{Discussion}\label{sec:discussion}

In this paper, we studied distributional matrix completion, the imputation of missing matrix entries when each entry is a probability distribution observed through finitely many samples. We proposed the Dist-NN family of estimators (single-sided, two-sided, and doubly robust) for imputing entire distributions by nearest-neighbor averaging in the 2-Wasserstein space, with consistency and convergence-rate guarantees under a Lipschitz latent-factor model. We utilized and contributed new results for empirical Wasserstein barycenters: a multiplicative convergence rate that holds uniformly over a general class of regular probability distributions, a functional central limit theorem for the barycenter quantile process, and simultaneous confidence bands. This barycenter theory underpins the matrix-completion guarantees and applies to any setting in which empirical Wasserstein barycenters of one-dimensional distributions are estimated.

Several research directions remain open. First, the doubly robust estimator attains a product-form bias rate, but the monotone rearrangement is nonlinear, so the functional central limit theorem does not carry over immediately. 
% Inference there proceeds via the bootstrap, resampling both the per-neighbor observations and the neighbor set. 
% A closed-form theory for the rearranged estimator would require either functional convergence through the rearrangement map or derivative-level concentration to control monotonicity. 
% Second, our bands are simultaneous over quantile levels for a single entry. Joint inference across entries introduces a distinct multiplicity and remains open. 
Second, extending the analysis to the missing-not-at-random setting would broaden the method's applicability. Finally, Wasserstein barycenters are NP-hard beyond one dimension \citep{altschuler2022wasserstein}, so extending \distnn\ to higher-dimensional distributions likely requires restricting the class of distributions considered \citep{le2022fast} or using sliced Wasserstein distances \citep{kolouri2019generalized}.

\section*{Code and data availability}

All code for this paper, along with a README describing how to reproduce our results, is available at \ifbool{anonymous}{\url{https://anonymous.4open.science/r/Dist-NN-D7D7/README.md}}{\url{https://github.com/jacobf18/Dist-NN}}. The HiRID data is available at \url{https://physionet.org/content/hirid/} under credentialed access.

%% file: appendix-proof-techniques.tex
\section{Proof of \protect\cref{thm:main-asymptotic}}\label{sec:proof-techniques}
Here we prove \cref{thm:main-asymptotic}, and we prove the lemmas and propositions it relies on in \cref{sec:proof-thm-main-asymptotic}\ofEC. A notable feature of the proof is a new result (\cref{prop:barycenters}) on the distance between the barycenter of a collection of distributions and the barycenter of its empirical counterparts, which may be of independent interest.

\subsection{Overview of the proof of \protect\cref{thm:main-asymptotic}}
We start with a basic decomposition of the error, and identify quantities that are akin to ``bias'' and ``variance'' in a typical decomposition of squared error. To do so, we first define $\Bar{\mu}_{ij}$,  the Wasserstein barycenter for the \emph{true} distributions in the neighboring rows in column $j$:
\begin{align}\label{eq:mu_bar_def}
    \Bar{\mu}_{ij} 
    &\defeq \argmin_{\mu} \sum_{u \in \nn{\eta}{i}} W_2^2(\mu, \mu_{uj}).
    \label{eq:true_mean}
\end{align}
From \citep[Eq.~8]{bigot2020}, this minimum has a closed form solution in 1-dimension where $\Bar{\mu}_{ij}$ is the distribution with quantile function given by
\begin{align}
    \quantile{\Bar{\mu}_{ij}} = \frac{1}{\lvert \nn{\eta}{i} \rvert} \sum_{u \in\nn{\eta}{i}} \quantile{\mu_{uj}}.
\end{align}
Next, we find that
\begin{align}\label{eq:bias-variance}
    W_2^2\parenth{\Hat{\mu}_{ij}, \mu_{ij}}
    &\seq{\cref{eq:wasserstein-1dim}} \lnorm{\quantile{\Hat{\mu}_{ij}} - \quantile{\mu_{ij}}}^2 \\
    &= \lnorm{\quantile{\Hat{\mu}_{ij}} - \quantile{\Bar{\mu}_{ij}} + \quantile{\Bar{\mu}_{ij}} - \quantile{\mu_{ij}}}^2 \\
    &\overset{(a)}{\leq} \parenth{\lnorm{\quantile{\Hat{\mu}_{ij}} - \quantile{\Bar{\mu}_{ij}}} + \lnorm{\quantile{\Bar{\mu}_{ij}} - \quantile{\mu_{ij}}}}^2\\
    &\overset{(b)}{\leq} 2\bigg[\underbrace{ W_2^2\parenth{\Hat{\mu}_{ij}, \Bar{\mu}_{ij}}}_{\defeq \mathcal{V}} + \underbrace{ W_2^2\parenth{\Bar{\mu}_{ij}, \mu_{ij}}}_{\defeq \mathcal{B}}\bigg],
\end{align}
where $(a)$ follows from the Minkowski inequality \citep[Thm. 198]{hardy1952inequalities} and $(b)$ follows from the Cauchy-Schwarz inequality \citep[Thm. 7]{hardy1952inequalities}. We refer to  $\mathcal{B}$ as the ``bias'' because it captures the error from estimating the distance between rows. Similarly, we refer to $\mathcal{V}$ as the ``variance'' because it captures the error from having a finite number of samples per cell. Next, we bound the terms $\mathcal{B}$ and $\mathcal{V}$ with two intermediate claims: 
% First, we claim the following bound on the variance term $\mathcal{V}$:
\begin{align}\label{claim:asym-bias-bound}
    \mc B = W_2^2\parenth{\Bar{\mu}_{ij}, \mu_{ij}} &= \order_p \big(\eta + (\numcols p^2)^{-1/2}\big) \quad \text{as} \quad \numcols \to \infty, \qtext{and}
 \\ 
    % Fix $(i,j)$ and let $\Bar{\mu}_{ij}$ be defined as in \cref{eq:mu_bar_def}, $\hat{\mu}_{ij}$ be the distribution with quantile function defined in \cref{eq:dist-nn-barycenter}, and $\nn{\eta}{i}$ defined as in \cref{eq:neighbors-def}. Then, the distributions in column $j$, $\sbraces{\mu_{s, j}}_{s=1}^{N}$, satisfy the conditions of \cref{prop:barycenters}, and thus
    % \begin{align}\label{eq:variance-bound}
        \mc V = W_2^2\parenth{\Hat{\mu}_{ij}, \Bar{\mu}_{ij}}
        &= 
        \order_p\parenth{\frac{1}{n_j |\nn{\eta}{i}|} + \frac{\log^2 n_j}{n_j^2} }
        \qtext{as} n_j \to\infty.
        \label{claim:barycenters-ok}
    % \end{align}.
\end{align}
which when put together with \cref{eq:basic_decomp} yields the desired claim in \cref{thm:main-asymptotic}.

The proof of \cref{claim:asym-bias-bound} is provided in \cref{sec:proof-asym-bias-bound}. It uses Minkowski's inequality and Hoeffding concentration to generalize the scalar-setting proof to the distributional setting. The proof of \cref{claim:barycenters-ok} is more involved and relies on new results for Wasserstein barycenters, so we dedicate the next two subsections to placing our derivation in the context of prior work.

We also state here an auxiliary lemma and a remark referenced from \cref{sec:nn}; their proofs, and the proofs of the remaining results of \cref{sec:barycenters,sec:nn}, appear in the \ECname.

\begin{lemma}[Uniform continuity of the quantile density]\label{lem:quantile-density-continuity}
    Let $\sbraces{\nu_m}_{m\geq 1}$ and $\nu$ be regular probability distributions (\cref{def:regular-measure}),
    with quantile densities $w_{\nu_m} \defeq 1/(f_{\nu_m}\circ\quantile{\nu_m})$ and
    $w_{\nu} \defeq 1/(f_{\nu}\circ\quantile{\nu})$, and suppose the regularity constants in
    conditions (i), (ii), (iv) are uniform over the family. If $W_2(\nu_m,\nu)\to 0$, then
    \begin{align}
        \norm{w_{\nu_m} - w_{\nu}}_{L^\infty(0,1)} \to 0.
    \end{align}
\end{lemma}

\begin{remark}[Comparison with Bonferroni, and alternatives]\label{rmk:band-sharpness}
The critical value $c_\alpha(\delta)$ is a universal constant, obtained once by simulating a
standardized Brownian bridge, and does not grow with the resolution of any quantile grid.
The band \cref{eq:confidence-band} is therefore narrower than a Bonferroni band, which with
simultaneous coverage over $m$ quantile levels inflates each pointwise interval by the factor
$z_{\alpha/(2m)} \asymp \sqrt{2\log m}$, which diverges as the grid is refined, whereas
$c_\alpha(\delta) = \order(1)$. 
% Two variants relax the bridge limit. When $\numcols$ is
% moderate, one may replace $c_\alpha(\delta)$ by the data-driven value obtained by simulating
% $\sup_{t\in[\delta,1-\delta]} \abss{G^*(t)}/\widehat{\sigma}_{\nn{\eta_\numcols}{i}}(t)$ for
% $G^*\sim\mc{N}(0,\widehat{\Sigma})$, with $\widehat{\Sigma}$ the plug-in estimate of the
% finite-neighborhood kernel \cref{eq:clt-covariance}; this adapts to departures from the bridge
% limit before $\eta_\numcols$ is small. Alternatively, both the scale and the critical value can
% be bypassed by a bootstrap that resamples the per-neighbor empirical distributions and the
% neighbor set, which we use as a cross-check in the simulations of \cref{sec:sim}.
\end{remark}

\subsubsection{Proof of \cref{claim:barycenters-ok}}
Our key step is to first condition on the latent row and column factors, so that the distributions in the neighborhood set $\nn{\eta}{i}$ are fixed objects, and \cref{prop:barycenters} can be applied. In the remainder of the proof, we abuse the notation and let $Y_{uv}$ denote the samples in the matrix entry $(u,v)$ regardless of whether $(u,v)$ is observed or not.
    
    For $u \in [\numrows + 1]$, let $I_u$ be the indicator random variable that $\rowdist{i}{u} \leq \eta$. Recall the definition of the row-wise distance from \cref{eq:avg-distance-def}.
    % : 
    % \begin{align}
    %     \rowdist{i}{u} &\defeq \begin{cases}
    %         \lvert \sharedcol{i}{u}\rvert^{-1} \sum_{v \in \sharedcol{i}{u}} W_2^2(Y_{iv}, Y_{uv})&\quad \text{if $\lvert \sharedcol{i}{u}\rvert \geq 1$} \\
    %         \infty &\quad \text{if $\lvert \sharedcol{i}{u}\rvert = 0$}
    %     \end{cases}, \quad \text{where} \\
    % % \end{align}
    % % \begin{align}
    % \label{eq:shared-col-def}
    %     \sharedcol{i}{u} &\defeq \{v \in [\numcols+1] \setminus\{j\}: A_{iv} = 1, A_{uv} = 1\}.
    % \end{align}
    We have $Y_{uj} \indep I_u$ since (i) the samples in column $j$ are not used to calculate $\rowdist{i}{u}$ and (ii), from \cref{assum:mcar}, the missingness of an entry is generated independently of the samples in that entry and all the latent factors. Next, we condition on the latent row and column factors $\mc{U}_{\mathrm{row}}$ and $\mc{U}_{\mathrm{col}}$, making each distribution corresponding to the rows in the neighborhood set fixed. 
    % $A_{uj}$ is dependent on $\rowdist{i}{u}$ because $\sharedcol{i}{u}$ is only defined if $A_{uj}=1$. But, we already showed by MCAR that ${Y_{uj} \indep A_{uj}}$.
    
    In column $j$, each observed matrix entry has $n_j$ samples drawn i.i.d., as described in \cref{sec:setup}. Thus, we can apply \cref{prop:barycenters} to the empirical distribution set $\{Y_{uj}\}_{u \in \nn{\eta}{i}}$ with $k = \abss{\nn{\eta}{i}}$. The polynomial-growth side condition $\log k = \order(\log n_j)$ required by \cref{prop:barycenters} is inherited from the result that invokes this claim. It holds automatically whenever the number of rows $\numrows$ is held fixed (as in \cref{thm:main-asymptotic}), since then $k = \abss{\nn{\eta}{i}} \leq \numrows + 1$ is bounded; it is implied by the rate condition $\sqrt{\abss{\nn{\eta}{i}}}\,\log n_j/\sqrt{n_j} = o_p(1)$ used in our asymptotic-normality results (\cref{thm:main-brownian-bridge-feasible}); and it is imposed explicitly as a hypothesis in the location-scale corollaries (\cref{cor:location-scale}), where $\numrows \to \infty$ and the neighborhood grows. Finally, since the bound in \cref{prop:barycenters} is uniform with universal constants that do not depend on the distributions or $\abss{\nn{\eta}{i}}$, we can remove the conditioning on the latent row and column factors $\mc{U}_{\mathrm{row}}$ and $\mc{U}_{\mathrm{col}}$.

\subsubsection{Proof of \protect\cref{prop:barycenters}: Convergence of the barycenter of empirical distributions}
\label{proof_of_empirical_convergence_barycenter}
From \citep[Eq. 8]{bigot2020}, the quantile function of each barycenter has an explicit formula:
\begin{align}
    \quantile{\overline{\nu}} = \frac1k \sum_{j=1}^k \quantile{\nu_{j}} \quad \text{and} \quad 
    \quantile{\what{\overline{\nu}}} = \frac1k \sum_{j=1}^k \quantile{\hat\nu_{j,n}}.
\end{align}
So, we can write the Wasserstein distance between the two barycenters as:
\begin{align}\label{eq:wass-barycenter}
    W_2^2(\what{\overline{\nu}},\overline{\nu})
    \seq{\cref{eq:wasserstein-1dim}} \big\Vert\quantile{\what{\overline{\nu}}} - \quantile{\overline{\nu}}\big\Vert^2 
    &\seq{\cref{eq:barycenter-quantile}} \lnorm{\frac1k \sum_{j=1}^k \parenth{\quantile{\hat\nu_{j,n}} - \quantile{\nu_{j}}}}^2 \\
    &= \frac1n \lnorm{\frac1k \sum_{j=1}^k \sqrt{n}\parenth{\quantile{\hat\nu_{j,n}} - \quantile{\nu_{j}}}}^2 \\ 
    &\defeq \frac1n \bigg\Vert{\frac1k \sum_{j=1}^k \hat{q}_{\nu_j,n} } \bigg\Vert^2_{L^2(0,1)},
\end{align}
where we have defined $\hat{q}_{\nu_j,n} \defeq \sqrt{n}(\quantile{\Hat{\nu}_{j,n}} - \quantile{\nu_j})$ in the last step.
% Next, for a distribution $\nu_j$ and its respective empirical distribution on $n$ samples, $\Hat{\nu}_{j,n}$, define
% \begin{align}\label{eq:emp-quantile}
%     \hat{q}_{\nu_j,n} \defeq \sqrt{n}\parenth{\quantile{\Hat{\nu}_{j,n}} - \quantile{\nu_j}}.
% \end{align}
To complete the proof, we next derive the asymptotic distribution of $\hat{q}_{\nu_j,n}$. While it is well known that $\hat{q}_{\nu_j,n}$ converges to a weighted Brownian bridge in distribution \citep[Ch. 18]{shorack2009empirical}, the next lemma establishes a stronger result, namely a convergence in probability for the barycenters. Here, a stochastic process $\mathbb{B}$ is a standard Brownian bridge if it is a Gaussian process where for $s,t \in (0,1)$, $\expect*{\mathbb{B}(t)} = 0$ and $\Cov\parenth{\mathbb{B}(s),\mathbb{B}(t)}=\min(s,t) - st$ \citep[Prop. 8.1.1]{ross1995stochastic}. Its proof  builds on a strong uniform bound for empirical quantile functions provided in \cite{csorgo1978strong} and is provided in \cref{sec:proof-l2-strong-approx-Op1}.
% require convergence in probability, and hence make use of a stronger result via \emph{Hungarian embeddings} \citep[Pg. 269]{van2000asymptotic}. In particular, using results from \citep[Pgs. 268-269]{van2000asymptotic}, we can construct an approximation of $\hat{q}_{\nu_j,n}$ via a sequence of standard Brownian bridges, our next lemma states.
\begin{lemma}[Approximation of barycenter]\label{lemma:l2-strong-approx-Op1}
    Consider a collection of distributions $\sbraces{\nu_j}_{j=1}^{k}$ each of which satisfies \cref{assum:regular}. For each $j\in [k]$, let $\what \nu_{j, n}$ denote the empirical distribution obtained from $n$ i.i.d. samples from $\nu_j$. Suppose that $k = k(n)$ satisfies $\log k = \order(\log n)$ as $n \to \infty$.
    Then, for each $j \in [k]$, there exists a sequence of standard Brownian bridges $\{\mbb B_{j,n}\}_{n=1}^\infty$ and a universal constant $c$ such that, almost surely for all $n$ large enough,
    \begin{align}\label{eq:lemma-l2-strong-approx}
        \frac{1}{n}\lnorm{\frac{1}{k} \sum_{j=1}^k \bigg(\sqrt{n}(\quantile{\Hat{\nu}_{j,n}} - \quantile{\nu_j}) \!-\! \frac{\mbb{B}_{j,n}}{f_{\nu_j} \circ \quantile{\nu_j}}\bigg)}^2
        \!\leq\! \frac{c\log^2 n}{n^2}
        \!=\! \order\bigg(\frac{\log^2 n}{n^2}\bigg).
    \end{align}
\end{lemma}
Using the Brownian bridges appearing in \cref{lemma:l2-strong-approx-Op1},  we obtain that
\begin{align}
    W_2^2(\what{\overline{\nu}},\overline{\nu})
    &\seq{\cref{eq:wass-barycenter}} \frac1n \lnorm{\frac1k \sum_{j=1}^k \hat{q}_{\nu_j,n}}^2 \\
    &= \frac1n \lnorm{\frac1k \sum_{j=1}^k \bigg(\hat{q}_{\nu_j,n} - \frac{\mathbb{B}_{j,n}}{f_{\nu_j} \circ \quantile{\nu_j}} + \frac{\mathbb{B}_{j,n}}{f_{\nu_j} \circ \quantile{\nu_j}}\bigg)}^2 \\
    &\sless{(i)} \frac2n \lnorm{\frac1k \sum_{j=1}^k \bigg(\hat{q}_{\nu_j,n} - \frac{\mathbb{B}_{j,n}}{f_{\nu_j} \circ \quantile{\nu_j}}\bigg)}^2 + \frac2n \lnorm{\frac1k \sum_{j=1}^k \frac{\mathbb{B}_{j,n}}{f_{\nu_j} \circ \quantile{\nu_j}}}^2,
    \label{eq:w2_decomp}
\end{align}
where $(i)$ follows from applying Minkowski's inequality \citep[Thm. 198]{hardy1952inequalities} followed by the Cauchy-Schwarz inequality \citep[Thm. 7]{hardy1952inequalities}. While \cref{lemma:l2-strong-approx-Op1} establishes that the first term on the RHS of the above display is bounded almost surely, the next result provides tight control on the second term in probability. 
\begin{lemma}[Norm of weighted average of Brownian bridges]\label{lemma:l2-strong-approx-bounded}
    For each $j \in [k]$, let $\{\mbb B_{j,n}\}_{n=1}^\infty$ be a sequence of independent standard Brownian bridges, and let $w_j:[0, 1] \to \real$ be $L$-Lipschitz and uniformly bounded, that is, $\sup_{j \in [k]} \sup_{t \in [0,1]} \abss{w_j(t)} \leq C'$, where $L$ and $C'$ do not depend on $k$.
    Then, we have
    \begin{align}
        \frac1n \lnorm{\frac1k \sum_{j=1}^k w_j \cdot \mathbb{B}_{j}}^2
        &= \order_p\parenth{\frac{1}{nk}}.
    \end{align}
\end{lemma}
The weights appearing in \cref{eq:w2_decomp} are $w_j = 1/(f_{\nu_j} \circ \quantile{\nu_j}) = (\quantile{\nu_j})'$, which satisfy both hypotheses uniformly in $j$ under \cref{assum:regular}: condition (ii) gives $f_{\nu_j} \geq C$ for a universal $C > 0$, so $\abss{w_j} \leq 1/C$, and condition (iv) states that $(\quantile{\nu_j})'$ is $L'$-Lipschitz. We prove \cref{lemma:l2-strong-approx-bounded} in  \cref{sec:proof-l2-strong-approx-bounded} by showing that the weighted sum of Brownian bridges is a Gaussian process and then proving that this sum is uniformly bounded over all values of $k$. Finally, putting together \cref{eq:w2_decomp,lemma:l2-strong-approx-Op1,lemma:l2-strong-approx-bounded}, we obtain \cref{eq:our_w2_rate_prop1} as claimed in \cref{prop:barycenters}.

%% file: appendix-proofs.tex
\section{Proof of Lemmas and Propositions for Thm. \protect\ref{thm:main-asymptotic}}\label{sec:proof-thm-main-asymptotic}

In this section, we provide detailed proofs for each supporting lemma and proposition used in the proof of \cref{thm:main-asymptotic} in \cref{sec:proof-techniques}.

\subsection{Remaining parts from the proof of \protect\cref{prop:barycenters}: Convergence of the barycenter of empirical distributions}
\label{remainder_proof_of_empirical_convergence_barycenter}

\subsubsection{Proof of \protect\cref{lemma:l2-strong-approx-Op1}: Approximation of barycenter}\label{sec:proof-l2-strong-approx-Op1}

For each $j \in [k]$ and $l \in [n]$, write $U_{j,l} \defeq F_{\nu_j}(X_{j,l})$ for the uniform random variables generated by the samples of $\nu_j$, let $V_{j,n}$ denote their empirical quantile function, and let $u_{j,n}(y) \defeq \sqrt{n}\parenth{V_{j,n}(y) - y}$ denote the corresponding uniform quantile process. Since $\quantile{\Hat{\nu}_{j,n}} = \quantile{\nu_j} \circ V_{j,n}$, the general quantile process can be recovered from the uniform one. Our proof builds on the following exponential bound:
\begin{namedthm}{A}[Thm. 1 in \citep{csorgo1978strong}]\label{thm:csorgo-unif-strong}
    If the uniform (0,1) random variables $U_1,U_2,\dots$ are defined on a rich enough probability space, then one can define, for each $n$, a Brownian bridge $\{\mbb{B}_n(y): 0 \leq y \leq 1\}$ on the same probability space such that, for all $z$, we have
    \begin{align}
        \mbb{P}\bigg(\sup_{0 \leq y \leq 1} \abss{u_n(y) - \mbb{B}_n(y)} > n^{-1/2}(A \log n + z)\bigg) \leq B e^{-Cz}
    \end{align}
    for positive absolute constants $A,B,$ and $C$.
\end{namedthm}
This result holds for each distribution separately. We must extend it, and the bounds derived from it, so that they hold uniformly over the collection with constants that do not depend on $k$. Our argument follows the architecture of the proof of Theorem 3 in \citet{csorgo1978strong}: a Taylor-based reduction of the general quantile process to the uniform one on an interior interval (their Eqs. 3.7-3.9), with the boundary strips handled by showing that both processes are individually small there (their Eqs. 3.10-3.14). We depart from their proof in one place. The weighted bound on the uniform quantile process, which they obtain from the law of the iterated logarithm of \citet{csaki1977law} via their Theorem 2, is a limit statement for a single sequence and does not extend to a maximum over a growing collection. We instead derive a version that holds at each finite $n$ with an explicit failure probability, so that it survives a union bound over the collection, from \cref{thm:csorgo-unif-strong} and Gaussian maximal inequalities, on the slightly narrower interior $[Dn^{-1}\log n, 1 - Dn^{-1}\log n]$ and with $\sqrt{\log n}$ in place of their constant, a loss that the coupling rate absorbs. The stronger regularity of \cref{def:regular-measure}, which bounds the density below on a compact support, also lets us replace their boundary analysis with a Lipschitz argument. Throughout, $c_1, c_2, \dots$ denote positive constants that depend only on the absolute constants of \cref{thm:csorgo-unif-strong}, the regularity constants of \cref{def:regular-measure}, and the constant in the side condition $\log k = \order(\log n)$. To avoid a clash with the constant $C$ of \cref{thm:csorgo-unif-strong}, we denote the density lower bound of condition (ii) by $C_f$ in this proof.

\paragraph{Coupling the uniform processes, uniformly over the collection.}
Applying \cref{thm:csorgo-unif-strong} separately to the samples of each $\nu_j$ (the underlying space is rich enough to carry the $k$ constructions independently) and taking a union bound,
\begin{align}\label{eq:csorgo-unif-strong-union}
    \mbb{P}\bigg(\sup_{j \in [k]}\sup_{0 \leq y \leq 1} \abss{u_{j,n}(y) - \mbb{B}_{j,n}(y)} > n^{-1/2}(A \log n + z)\bigg) \leq k B e^{-Cz}.
\end{align}
We convert this into an almost-sure rate via the Borel-Cantelli lemma. Writing $k = k(n)$ and choosing the threshold $z_n \defeq \tfrac1C\parenth{\log k(n) + 2 \log n}$, the right-hand side of \cref{eq:csorgo-unif-strong-union} becomes
\begin{align}
    k B e^{-C z_n} = k B \cdot \frac{1}{k\, n^2} = \frac{B}{n^2},
\end{align}
which is summable over $n$. By the Borel-Cantelli lemma, almost surely for all large $n$,
\begin{align}
    \sup_{j \in [k]} \sup_{0 \leq y \leq 1} \abss{u_{j,n}(y) - \mbb{B}_{j,n}(y)}
    &\leq n^{-1/2}\parenth{A \log n + z_n} \notag \\
    &= n^{-1/2}\parenth{A \log n + \tfrac1C \log k + \tfrac2C \log n} \notag \\
    &= \order\parenth{\frac{\log n + \log k}{\sqrt{n}}}. \label{eq:unif-strong-union-rate}
\end{align}
Denote the right-hand side of \cref{eq:unif-strong-union-rate} by $R_n \defeq c_1 (\log n + \log k)/\sqrt{n}$. Under the side condition $\log k = \order(\log n)$, we have $R_n = \order(n^{-1/2}\log n)$.

\paragraph{Maximal bounds for the Brownian bridges.}
Set $\delta_n \defeq D n^{-1}\log n$ for a constant $D \geq 1$ chosen below. We claim that, almost surely for all large $n$,
\begin{align}\label{eq:bridge-maximal-bounds}
    \sup_{j \in [k]} \sup_{0 \leq y \leq \delta_n} \abss{\mbb{B}_{j,n}(y)} \leq c_2 \sqrt{\delta_n \log n}
    \qtext{and}
    \sup_{j \in [k]} \sup_{\delta_n \leq y \leq 1-\delta_n} \frac{\abss{\mbb{B}_{j,n}(y)}}{\sqrt{y(1-y)}} \leq c_3 \sqrt{\log n},
\end{align}
and the first bound also holds on the symmetric strip $[1-\delta_n, 1]$. For the first bound, on $[0,\delta]$ the bridge satisfies $\sup_{y \leq \delta}\Var(\mbb{B}(y)) \leq \delta$ and $\expect*{\sup_{y\leq\delta}\abss{\mbb{B}(y)}} \leq c\sqrt{\delta}$, so the Borell-TIS inequality \citep[Thm. 2.1.1]{adler2009random} gives, for any $\lambda > 0$,
\begin{align}
    \mbb{P}\bigg(\sup_{0\leq y \leq \delta}\abss{\mbb{B}(y)} > c\sqrt{\delta} + \lambda\bigg) \leq 2\exp\parenth{-\frac{\lambda^2}{2\delta}}.
\end{align}
Taking $\delta = \delta_n$ and $\lambda = \sqrt{2\delta_n(\log k + 2\log n)}$, a union bound over $j \in [k]$ makes the failure probability at most $2/n^2$, which is summable, and the Borel-Cantelli lemma gives the first bound of \cref{eq:bridge-maximal-bounds}. For the second bound, partition $[\delta_n, 1/2]$ into the dyadic blocks $I_m \defeq [2^m \delta_n, 2^{m+1}\delta_n]$ for $m = 0, 1, \dots, M$ with $M \leq \log_2 n$. On $I_m$ we have $y(1-y) \geq 2^{m-1}\delta_n$, and the same Borell-TIS argument on $[0, 2^{m+1}\delta_n]$ with $\lambda_m \defeq \sqrt{2^{m+2}\delta_n(\log k + 2\log n + 2\log(m+2))}$ has failure probability at most $2/(n^2 (m+2)^2)$. The sum over $m$ and $n$ remains finite, and on the complement,
\begin{align}
    \sup_{y \in I_m} \frac{\abss{\mbb{B}_{j,n}(y)}}{\sqrt{y(1-y)}}
    \leq \frac{c\sqrt{2^{m+1}\delta_n} + \lambda_m}{\sqrt{2^{m-1}\delta_n}}
    \leq 2c + 4\sqrt{2\parenth{\log k + 2\log n + 2\log(m+2)}}
    \leq c_3\sqrt{\log n},
\end{align}
using $\log k = \order(\log n)$ and $\log(m+2) \leq \log \log_2 n + 2$. The interval $[1/2, 1-\delta_n]$ is symmetric.

\paragraph{Bounds for the uniform quantile processes.}
Combining \cref{eq:unif-strong-union-rate,eq:bridge-maximal-bounds} via $\abss{u_{j,n}} \leq \abss{u_{j,n} - \mbb{B}_{j,n}} + \abss{\mbb{B}_{j,n}}$, we have almost surely for all large $n$, uniformly in $j \in [k]$,
\begin{align}
    \sup_{0 \leq y \leq \delta_n} \abss{u_{j,n}(y)} &\leq R_n + c_2\sqrt{\delta_n \log n} \leq c_4\, \frac{\log n + \log k}{\sqrt n}, \qtext{and} \label{eq:unif-strip-bound}\\
    \sup_{\delta_n \leq y \leq 1-\delta_n} \frac{\abss{u_{j,n}(y)}}{\sqrt{y(1-y)}} &\leq \frac{R_n}{\sqrt{\delta_n/2}} + c_3\sqrt{\log n} \leq c_5 \sqrt{\log n}, \label{eq:unif-interior-bound}
\end{align}
where \cref{eq:unif-strip-bound} uses $\sqrt{\delta_n \log n} = \sqrt{D}\,n^{-1/2}\log n$, and \cref{eq:unif-interior-bound} uses $y(1-y) \geq \delta_n/2$ on the interior together with $R_n/\sqrt{\delta_n/2} = c_1\sqrt{2/D}\,(\log n + \log k)/\sqrt{\log n} \leq c\sqrt{\log n}$ by the side condition. The strip $[1-\delta_n,1]$ is symmetric.

\paragraph{From uniform to general on the boundary strips.}
Write $Q_j \defeq \quantile{\nu_j}$ and $f_j$ for the density of $\nu_j$. Since $\quantile{\Hat\nu_{j,n}} = Q_j \circ V_{j,n}$ and $Q_j$ is $(1/C_f)$-Lipschitz by condition (ii) of \cref{def:regular-measure}, we have $\abss{\hat q_{\nu_j,n}(y)} = \sqrt{n}\,\abss{Q_j(V_{j,n}(y)) - Q_j(y)} \leq \abss{u_{j,n}(y)}/C_f$ for every $y \in (0,1)$. Moreover, $f_j$ is bounded above by a universal constant $c_6$ on its support, by conditions (i) and (v) as in the proof of \cref{claim:gaussian-process-op1}. So, on the strips,
\begin{align}\label{eq:strip-coupling}
    \sup_{y \leq \delta_n \text{ or } y \geq 1-\delta_n} \abss{f_j(Q_j(y))\,\hat q_{\nu_j,n}(y) - u_{j,n}(y)}
    \leq \parenth{\frac{c_6}{C_f} + 1} \sup_{y \leq \delta_n \text{ or } y \geq 1-\delta_n} \abss{u_{j,n}(y)}
    \leq c_7\, \frac{\log n + \log k}{\sqrt n},
\end{align}
by \cref{eq:unif-strip-bound}: near the boundary we do not need the two processes to track each other, because both are individually small.

\paragraph{From uniform to general on the interior.}
For $\delta_n \leq y \leq 1-\delta_n$, the second-order Taylor identity of \citep[Eqs. 3.7-3.8]{csorgo1978strong} gives
\begin{align}\label{eq:csorgo-taylor}
    \abss{f_j(Q_j(y))\,\hat q_{\nu_j,n}(y) - u_{j,n}(y)}
    \leq \frac{1}{2\sqrt n}\, u_{j,n}^2(y)\, f_j(Q_j(y))\, \frac{\abss{f_j'(Q_j(\xi))}}{f_j^3(Q_j(\xi))},
\end{align}
where $\xi$ lies between $y$ and $V_{j,n}(y)$, so that $\abss{\xi - y} \leq n^{-1/2}\abss{u_{j,n}(y)}$. By \cref{eq:unif-interior-bound}, $\abss{\xi - y} \leq c_5\sqrt{y(1-y)\log n / n}$, and choosing $D \geq 16 c_5^2$ makes $\abss{\xi-y} \leq \tfrac{y}{2} \wedge \tfrac{1-y}{2}$ on the interior, so that each of the ratios $y/\xi$, $\xi/y$, $(1-y)/(1-\xi)$, and $(1-\xi)/(1-y)$ is at most $2$. We factor the right-hand side of \cref{eq:csorgo-taylor} as
\begin{align}
    \frac{1}{2\sqrt n} \cdot \underbrace{\frac{u_{j,n}^2(y)}{y(1-y)}}_{\leq\, c_5^2 \log n} \cdot \underbrace{\frac{y(1-y)}{\xi(1-\xi)}}_{\leq\, 4} \cdot \underbrace{\xi(1-\xi)\frac{\abss{f_j'(Q_j(\xi))}}{f_j^2(Q_j(\xi))}}_{\leq\, \gamma} \cdot \underbrace{\frac{f_j(Q_j(y))}{f_j(Q_j(\xi))}}_{\leq\, 4^\gamma},
\end{align}
where the first factor is \cref{eq:unif-interior-bound}, the third is condition (vi) of \cref{def:regular-measure} with $\gamma < 2$, and the fourth follows from Lemma 1 of \citep{csorgo1978strong} (their Eq. 3.6), which bounds the density ratio by $\parenth{\tfrac{y\vee\xi}{y\wedge\xi}\cdot\tfrac{1-y\wedge\xi}{1-y\vee\xi}}^{\gamma} \leq 4^\gamma$. Thus, almost surely for all large $n$, uniformly in $j \in [k]$,
\begin{align}\label{eq:interior-coupling}
    \sup_{\delta_n \leq y \leq 1-\delta_n} \abss{f_j(Q_j(y))\,\hat q_{\nu_j,n}(y) - u_{j,n}(y)}
    \leq c_8\, \frac{\log n}{\sqrt n}.
\end{align}

\paragraph{Assembly.}
Since $w_j \defeq 1/(f_j \circ Q_j) \leq 1/C_f$ by condition (ii), we can write, for every $y \in (0,1)$,
\begin{align}
    \hat q_{\nu_j,n}(y) - w_j(y)\,\mbb{B}_{j,n}(y)
    = w_j(y)\big[f_j(Q_j(y))\,\hat q_{\nu_j,n}(y) - u_{j,n}(y)\big] + w_j(y)\big[u_{j,n}(y) - \mbb{B}_{j,n}(y)\big].
\end{align}
Taking the supremum over $y$ and over $j \in [k]$, and combining \cref{eq:strip-coupling,eq:interior-coupling} for the first bracket with \cref{eq:unif-strong-union-rate} for the second, we conclude that almost surely for all large $n$,
\begin{align}\label{eq:brownian-bridge-uniform-bound}
    \sup_{j \in [k]} \sup_{0 < y < 1} \bigg \vert \hat{q}_{\nu_j,n}(y) - \frac{\mbb{B}_{j,n}(y)}{f_j(Q_j(y))}\bigg \vert \leq \frac{c_9}{C_f}\, \frac{\log n + \log k}{\sqrt n},
\end{align}
with constants that depend neither on the collection $\sbraces{\nu_j}$ nor on $k$. Under the side condition $\log k = \order(\log n)$, the right-hand side is $\order(n^{-1/2}\log n)$. Then, for any $y \in (0,1)$, almost surely,
\begin{align}
    \bigg\vert\frac{1}{k} \sum_{j=1}^k \bigg(\hat{q}_{\nu_j,n}(y) - \frac{\mbb{B}_{j,n}(y)}{f_j(Q_j(y))}\bigg)\bigg\vert
    \sless{(c)} \frac{1}{k} \sum_{j=1}^k \bigg\vert\hat{q}_{\nu_j,n}(y) - \frac{\mbb{B}_{j,n}(y)}{f_j(Q_j(y))}\bigg\vert
    \sless{\cref{eq:brownian-bridge-uniform-bound}} c_{10}\, \frac{\log n}{\sqrt n},
\end{align}
where $(c)$ follows from the triangle inequality. Next, since this holds for all $y$, we can take the $L^2(0,1)$ norm of both sides to get
\begin{align}
    \frac{1}{n}\lnorm{\frac{1}{k} \sum_{j=1}^k \bigg(\hat{q}_{\nu_j,n}(y) - \frac{\mbb{B}_{j,n}(y)}{f_j(Q_j(y))}\bigg)}^2
    \leq \frac{c_{10}^2\,\log^2 n}{n^2}
    = \order\bigg(\frac{\log^2 n}{n^2}\bigg).
\end{align}

\subsubsection{Proof of \protect\cref{thm:functional-clt}: Functional CLT for the empirical barycenter quantile process}\label{sec:proof-functional-clt}
By the barycenter quantile formula \citep[Eq.~8]{bigot2020}, the quantile functions of
$\overline{\nu}$ and $\what{\overline{\nu}}_n$ are the averages
$\quantile{\overline{\nu}} = \frac1k\sum_j \quantile{\nu_j}$ and
$\quantile{\what{\overline{\nu}}_n} = \frac1k\sum_j \quantile{\what\nu_{j,n}}$. Hence
\begin{align}
    \mbb{Q}_{n,k}(t)
    = \sqrt{nk}\cdot\frac1k\sum_{j=1}^k
        \big(\quantile{\what\nu_{j,n}}(t) - \quantile{\nu_j}(t)\big)
    = \frac{1}{\sqrt k}\sum_{j=1}^k \hat{q}_{\nu_j,n}(t),
\end{align}
where $\hat{q}_{\nu_j,n} \defeq \sqrt{n}\,(\quantile{\what\nu_{j,n}} - \quantile{\nu_j})$ is
the empirical quantile process of $\nu_j$, exactly as in the proof of
\cref{prop:barycenters}.

Let $\sbraces{\mbb{B}_{j,n}}_{j\in[k]}$ be the standard Brownian bridges supplied by
\cref{lemma:l2-strong-approx-Op1}; for each fixed $n$ these are $k$ independent standard
Brownian bridges, since the samples across the $k$ distributions are independent and the coupling of \cref{thm:csorgo-unif-strong} is constructed separately from each distribution's samples. With
$w_j = 1/(f_{\nu_j}\circ\quantile{\nu_j})$ and $\mbb{G}_k$ as in
\cref{eq:clt-gaussian-process}, the difference between the barycenter quantile process and
the Gaussian process is
\begin{align}\label{eq:clt-difference}
    \mbb{Q}_{n,k}(t) - \mbb{G}_k(t)
    = \frac{1}{\sqrt k}\sum_{j=1}^k
        \big(\hat{q}_{\nu_j,n}(t) - w_j(t)\,\mbb{B}_{j,n}(t)\big).
\end{align}

\paragraph{Coupling bound.}
The uniform strong approximation established in \cref{eq:brownian-bridge-uniform-bound}, in
the proof of \cref{lemma:l2-strong-approx-Op1}, gives almost surely
\begin{align}\label{eq:clt-strong-approx}
    \sup_{j\in[k]}\sup_{0<t<1}
        \abss{\hat{q}_{\nu_j,n}(t) - w_j(t)\,\mbb{B}_{j,n}(t)}
    = \order\!\parenth{\frac{\log n + \log k}{\sqrt n}},
\end{align}
with a constant depending neither on $\sbraces{\nu_j}$ nor on $k$. Under the hypothesis
$\log k = \order(\log n)$, the right-hand side is $\order((\log n)/\sqrt n)$. Applying the
triangle inequality to \cref{eq:clt-difference} and then \cref{eq:clt-strong-approx}, for
every $t \in (0,1)$ almost surely
\begin{align}
    \abss{\mbb{Q}_{n,k}(t) - \mbb{G}_k(t)}
    &\leq \frac{1}{\sqrt k}\sum_{j=1}^k
        \abss{\hat{q}_{\nu_j,n}(t) - w_j(t)\,\mbb{B}_{j,n}(t)} \notag \\
    &\leq \frac{k}{\sqrt k}\cdot\order\!\parenth{\frac{\log n}{\sqrt n}}
    = \order\!\parenth{\sqrt k\,\frac{\log n}{\sqrt n}}.
\end{align}
Since this bound is uniform in $t$, taking the supremum over $t \in (0,1)$ yields
\cref{eq:functional-clt}. The $L^2(0,1)$ statement \cref{eq:functional-clt-l2} follows
because the domain $(0,1)$ has unit length, so that $\lnorm{\cdot} \leq \sup_{0<t<1}\abss{\cdot}$.

\paragraph{Identification of $\mbb{G}_k$.}
For each fixed $n$, the process $\mbb{G}_k = k^{-1/2}\sum_j w_j\,\mbb{B}_{j,n}$ is a finite
linear combination of the independent standard Brownian bridges
$\mbb{B}_{1,n},\dots,\mbb{B}_{k,n}$ with deterministic coefficient functions $w_j$. It is
therefore a centered Gaussian process, with continuous sample paths because each
$w_j = (\quantile{\nu_j})'$ is Lipschitz by \cref{def:regular-measure}, condition (iv). Using
the independence of the bridges across $j$ and the standard Brownian-bridge covariance
$\Cov(\mbb{B}_{j,n}(s),\mbb{B}_{j,n}(t)) = \min(s,t)-st$, we obtain
\begin{align}
    \Cov\big(\mbb{G}_k(s),\mbb{G}_k(t)\big)
    &= \frac1k\sum_{j=1}^k\sum_{j'=1}^k
        w_j(s)\,w_{j'}(t)\,\Cov\big(\mbb{B}_{j,n}(s),\mbb{B}_{j',n}(t)\big) \\
    &= \frac1k\sum_{j=1}^k w_j(s)\,w_j(t)\,\big(\min(s,t)-st\big),
\end{align}
which is \cref{eq:clt-covariance}. In particular the law of $\mbb{G}_k$ does not depend on
$n$. \qed

\subsubsection{Proof of \protect\cref{lem:quantile-density-continuity}: Uniform continuity of the quantile density}\label{sec:proof-quantile-density-continuity}
By \cref{eq:wasserstein-1dim}, $W_2(\nu_m,\nu) = \norm{\quantile{\nu_m} - \quantile{\nu}}_{L^2(0,1)} \to 0$.
Each quantile function $\quantile{\nu_m}$ is continuously differentiable with derivative $(\quantile{\nu_m})' = w_{\nu_m}$, where $\norm{w_{\nu_m}}_{L^\infty} \leq 1/C$ from condition (ii) ($f_{\nu_m} \geq C$) and $w_{\nu_m}$ is $L'$-Lipschitz from condition (iv). The same holds for $\quantile{\nu}$.

First, we show that $\quantile{\nu_m} \to \quantile{\nu}$ uniformly. From condition (i), all of the distributions are supported in a common bounded interval, so the family $\sbraces{\quantile{\nu_m}}$ is uniformly bounded. Since each $\quantile{\nu_m}$ is $(1/C)$-Lipschitz, the family is also equicontinuous. So, by the Arzel\`a-Ascoli theorem \citep[Ch.~11]{carothers2000real}, every subsequence of $\sbraces{\quantile{\nu_m}}$ has a further subsequence that converges uniformly to some continuous limit. Uniform convergence on $(0,1)$ implies convergence in $L^2(0,1)$, and the $L^2$ limit of the full sequence is $\quantile{\nu}$, so each such uniform limit must equal $\quantile{\nu}$. Thus, every subsequence has a further subsequence converging uniformly to the same limit, and so the full sequence converges: $\quantile{\nu_m} \to \quantile{\nu}$ uniformly on $[0,1]$.

Next, we apply the same argument to the derivatives. The family $\sbraces{w_{\nu_m}}$ is uniformly bounded by $1/C$ and equi-Lipschitz with constant $L'$, so it is also equicontinuous. Thus, by the Arzel\`a-Ascoli theorem, every subsequence of $\sbraces{w_{\nu_m}}$ has a further subsequence $\sbraces{w_{\nu_{m'}}}$ that converges uniformly to some continuous limit $g$. Writing
\begin{align}
    \quantile{\nu_{m'}}(t) = \quantile{\nu_{m'}}(0) + \int_0^t w_{\nu_{m'}}(s)\, ds
\end{align}
and taking limits on both sides (using that $\quantile{\nu_{m'}} \to \quantile{\nu}$ uniformly and $w_{\nu_{m'}} \to g$ uniformly), we get
\begin{align}
    \quantile{\nu}(t) = \quantile{\nu}(0) + \int_0^t g(s)\, ds.
\end{align}
Thus, $g = (\quantile{\nu})' = w_{\nu}$. So, every subsequence of $\sbraces{w_{\nu_m}}$ has a further subsequence converging uniformly to the same limit $w_{\nu}$, and the full sequence converges: $\norm{w_{\nu_m} - w_{\nu}}_{L^\infty(0,1)} \to 0$. \qed

\subsubsection{Proof of \protect\cref{cor:nn-functional-clt}: Functional CLT for the nearest-neighbor estimator}\label{sec:proof-nn-functional-clt}
Write $k = \abss{\nn{\eta_\numcols}{i}}$ and $n = n_{j,\numcols}$, and decompose
\begin{align}
    \sqrt{nk}\,\big(\quantile{\Hat{\mu}_{ij}} - \quantile{\mu_{ij}}\big)
    = \underbrace{\sqrt{nk}\,\big(\quantile{\Hat{\mu}_{ij}} - \quantile{\Bar{\mu}_{ij}}\big)}_{\text{variance}}
    + \underbrace{\sqrt{nk}\,\big(\quantile{\Bar{\mu}_{ij}} - \quantile{\mu_{ij}}\big)}_{\text{bias}}.
\end{align}

\paragraph{Bias term.} The sup-norm bias bound established in the proof of
\cref{claim:gaussian-process-op1} (via the interpolation inequality of
\cref{lem:interpolation}) gives, under \cref{eq:bb-assum1},
$\sqrt{nk}\,\norm{\quantile{\Bar{\mu}_{ij}} - \quantile{\mu_{ij}}}_{L^\infty(0,1)} = o_p(1)$,
and a fortiori $\sqrt{nk}\,\lnorm{\quantile{\Bar{\mu}_{ij}} - \quantile{\mu_{ij}}} = o_p(1)$.

\paragraph{Variance term.} Condition on the latent row and column factors, so that the
neighborhood collection $\sbraces{\mu_{uj}}_{u\in\nn{\eta_\numcols}{i}}$ is a fixed family of
$k$ regular probability distributions, each with per-distribution sample size $n$; the side condition
$\log k = \order(\log n)$ holds as discussed after \cref{claim:barycenters-ok}. Applying
\cref{thm:functional-clt} to this collection (with $\Hat{\mu}_{ij}$ and $\Bar{\mu}_{ij}$ in
the roles of $\what{\overline{\nu}}_n$ and $\overline{\nu}$) furnishes independent standard
Brownian bridges $\sbraces{\mbb{B}_{u,n}}_{u\in\nn{\eta_\numcols}{i}}$ and the coupling
\begin{align}
    \sup_{0<t<1}\abss{\sqrt{nk}\,\big(\quantile{\Hat{\mu}_{ij}} - \quantile{\Bar{\mu}_{ij}}\big)(t)
        - \mbb{G}_k(t)}
    = \order\!\parenth{\sqrt k\,\frac{\log n}{\sqrt n}} = o_p(1)
\end{align}
under \cref{eq:bb-assum1}, where $\mbb{G}_k = k^{-1/2}\sum_{u\in\nn{\eta_\numcols}{i}} w_{uj}\,\mbb{B}_{u,n}$
and $w_{uj} \defeq 1/(f_{\mu_{uj}}\circ\quantile{\mu_{uj}})$.

\paragraph{Weight convergence.} Set $\mbb{B} \defeq k^{-1/2}\sum_{u\in\nn{\eta_\numcols}{i}} \mbb{B}_{u,n}$,
which, being a normalized sum of $k$ independent standard Brownian bridges, is itself a
standard Brownian bridge. Then $\mbb{G}_k - w_\infty\mbb{B} = k^{-1/2}\sum_{u}(w_{uj} - w_\infty)\mbb{B}_{u,n}$,
so by the independence of the bridges across $u$ and the Brownian-bridge variance
$\Var(\mbb{B}_{u,n}(t)) = t-t^2$,
\begin{align}
    \expect*{\lnorm{\mbb{G}_k - w_\infty\mbb{B}}^2}
    &= \frac1k\sum_{u\in\nn{\eta_\numcols}{i}} \int_0^1 \big(w_{uj}(t)-w_\infty(t)\big)^2 (t-t^2)\,dt \notag \\
    &\leq \frac16 \max_{u\in\nn{\eta_\numcols}{i}} \norm{w_{uj} - w_\infty}_{L^\infty}^2.
\end{align}
By the maximum-over-neighbors bound in the proof of \cref{claim:asym-bias-bound},
$\max_{u\in\nn{\eta_\numcols}{i}} W_2^2(\mu_{uj},\mu_{ij}) = \order_p\!\big(\eta_\numcols + \sqrt{\log(2\numrows_\numcols)/(\numcols p^2)}\big)$,
which tends to $0$ because $\eta_\numcols \to 0$ and $\log(2\numrows_\numcols)/(\numcols p^2) \to 0$
(the latter forced by \cref{eq:bb-assum1}). Thus, $\max_u W_2(\mu_{uj},\mu_{ij}) = o_p(1)$. Next, \cref{lem:quantile-density-continuity} yields a modulus of continuity at the fixed target $\mu_{ij}$: with $\omega(\delta) \defeq \sup\sbraces{\norm{w_\nu - w_\infty}_{L^\infty(0,1)} : \nu \text{ regular with the same uniform constants},\ W_2(\nu,\mu_{ij}) \leq \delta}$, we have $\omega(\delta) \to 0$ as $\delta \to 0$, for otherwise there would exist a sequence of regular distributions $\nu_m$ with $W_2(\nu_m,\mu_{ij}) \to 0$ and $\norm{w_{\nu_m} - w_\infty}_{L^\infty}$ bounded away from zero, contradicting the lemma. Since $\omega$ is nondecreasing, $\max_u \norm{w_{uj}-w_\infty}_{L^\infty} \leq \omega\parenth{\max_u W_2(\mu_{uj},\mu_{ij})} = o_p(1)$, and therefore
$\lnorm{\mbb{G}_k - w_\infty\mbb{B}} = o_p(1)$.

\paragraph{Combination.} The triangle inequality in $L^2(0,1)$ applied to the three displays gives
\begin{align}
    \lnorm{\sqrt{nk}\,\big(\quantile{\Hat{\mu}_{ij}} - \quantile{\mu_{ij}}\big) - w_\infty\mbb{B}}
    = o_p(1).
\end{align}
Since the limiting process $w_\infty\mbb{B}$ depends only on the target distribution $\mu_{ij}$ and
not on the realized neighborhood, the bound holds unconditionally, which is
\cref{eq:nn-functional-clt}. \qed

\subsubsection{Proof of \protect\cref{cor:confidence-band}: Simultaneous confidence band}\label{sec:proof-confidence-band}
Throughout, write $k = \abss{\nn{\eta_\numcols}{i}}$, $n = n_{j,\numcols}$,
$\sigma_k(t) = \sigma_{\nn{\eta_\numcols}{i}}(t)$, and $\widehat{\sigma}(t) = \widehat{\sigma}_{\nn{\eta_\numcols}{i}}(t)$.
By \cref{eq:neighbor-sigma}, $\sigma_k(t)^2 = (t-t^2)\,\tfrac1k\sum_{u\in\nn{\eta_\numcols}{i}} w_{uj}(t)^2 = \Sigma_k(t,t)$
is the pointwise variance of the process $\mbb{G}_k$ from the proof of \cref{cor:nn-functional-clt}.
As shown in the proof of \cref{claim:gaussian-process-op1}, the densities $f_{\mu_{uj}}$ are
uniformly bounded above on the common bounded support, so $w_{uj} \geq 1/C'$ for a universal
$C' < \infty$, and hence
\begin{align}\label{eq:sigma-lower-bound}
    \sigma_k(t)^2 \geq \frac{t(1-t)}{C'^2} \geq \frac{\delta(1-\delta)}{C'^2} > 0
    \qquad \text{for all } t\in[\delta,1-\delta].
\end{align}

\paragraph{Reduction to the studentized Gaussian process.}
From the variance coupling and the sup-norm bias bound in the proof of
\cref{cor:nn-functional-clt} (\cref{thm:functional-clt} and \cref{claim:gaussian-process-op1},
respectively), we have $\sqrt{nk}\,(\quantile{\Hat{\mu}_{ij}} - \quantile{\mu_{ij}}) = \mbb{G}_k + r_\numcols$
with $\sup_{0<t<1}\abss{r_\numcols(t)} = o_p(1)$. Dividing by $\sigma_k$, which is bounded below
on $[\delta,1-\delta]$ by \cref{eq:sigma-lower-bound},
\begin{align}\label{eq:studentized-reduction}
    \frac{\sqrt{nk}\,(\quantile{\Hat{\mu}_{ij}}(t) - \quantile{\mu_{ij}}(t))}{\sigma_k(t)}
    = \frac{\mbb{G}_k(t)}{\sigma_k(t)} + o_p(1) \qquad \text{uniformly on } [\delta,1-\delta].
\end{align}

\paragraph{Pivotal weak limit of the studentized Gaussian process.}
The process $Z_k \defeq \mbb{G}_k/\sigma_k$ is mean-zero, unit-variance, and Gaussian, with
correlation kernel $\rho_k(s,t) = \Sigma_k(s,t)/(\sigma_k(s)\sigma_k(t))$. By
\cref{lem:quantile-density-continuity} and $\eta_\numcols\to 0$ (exactly as in the weight step of
the proof of \cref{cor:nn-functional-clt}), $\max_{u}\norm{w_{uj}-w_\infty}_{L^\infty} = o_p(1)$,
so $\tfrac1k\sum_u w_{uj}(s)w_{uj}(t) \to w_\infty(s)w_\infty(t)$ uniformly on $[\delta,1-\delta]^2$,
and therefore
\begin{align}
    \rho_k(s,t)
    = \frac{(\min(s,t)-st)\,\tfrac1k\sum_u w_{uj}(s)w_{uj}(t)}{\sigma_k(s)\sigma_k(t)}
    \;\to\; \rho_\infty(s,t) \defeq \frac{\min(s,t)-st}{\sqrt{(s-s^2)(t-t^2)}}
\end{align}
uniformly on $[\delta,1-\delta]^2$, where $\rho_\infty$ is the correlation function of the
standardized Brownian bridge $Z_\infty(t)\defeq \mbb{B}(t)/\sqrt{t(1-t)}$. The uniform covariance convergence gives convergence of the finite-dimensional distributions. For tightness, the increments satisfy $\expect*{(Z_k(s)-Z_k(t))^2} \leq C\abss{s-t}$ on $[\delta,1-\delta]^2$ with $C$ independent of $k$: this follows from the increment bound \cref{claim:dudley-distance-bound} for $\mbb{G}_k$, together with $\sigma_k$ being Lipschitz and bounded below on $[\delta,1-\delta]$ (the $w_{uj}$ are uniformly bounded and uniformly Lipschitz by \cref{def:regular-measure}). Since $Z_k$ is a Gaussian process, $\expect*{(Z_k(s)-Z_k(t))^4} = 3\parenth{\expect*{(Z_k(s)-Z_k(t))^2}}^2 \leq 3C^2\abss{s-t}^2$. Thus, the Kolmogorov criterion gives tightness in $C[\delta,1-\delta]$, and $Z_k \Rightarrow Z_\infty$ weakly
\citep[Ch.~2-3]{shorack2009empirical}. By the continuous mapping theorem applied to the
supremum functional, $\sup_{t\in[\delta,1-\delta]}\abss{Z_k(t)} \Rightarrow \sup_{t\in[\delta,1-\delta]}\abss{Z_\infty(t)}$.

\paragraph{Studentization by $\widehat{\sigma}$.}
Let $T_\numcols \defeq \sup_{t\in[\delta,1-\delta]} \abss{\sqrt{nk}\,(\quantile{\Hat{\mu}_{ij}}(t) - \quantile{\mu_{ij}}(t))}/\widehat{\sigma}(t)$.
Writing $T_\numcols = \sup_t \abss{\sqrt{nk}\,(\quantile{\Hat{\mu}_{ij}}(t) - \quantile{\mu_{ij}}(t))/\sigma_k(t)}\cdot(\sigma_k(t)/\widehat{\sigma}(t))$
and using \cref{eq:studentized-reduction}, the uniform consistency
\cref{eq:uniform-sigma-consistency}, and that $\sup_t\abss{Z_k}=O_p(1)$, we obtain
$T_\numcols = \sup_{t\in[\delta,1-\delta]}\abss{Z_k(t)} + o_p(1) \Rightarrow \sup_{t\in[\delta,1-\delta]}\abss{Z_\infty(t)}$
by Slutsky's theorem and the pivotal weak limit above.

\paragraph{Coverage.}
The event $\{\quantile{\mu_{ij}}(t)\in\mc{C}_\numcols(t)\ \forall t\in[\delta,1-\delta]\}$ equals
$\{T_\numcols \leq c_\alpha(\delta)\}$. The limit $\sup_{t\in[\delta,1-\delta]}\abss{Z_\infty(t)}$
is the supremum of a non-degenerate separable Gaussian process and therefore has a continuous
cumulative distribution function on $(0,\infty)$, so $c_\alpha(\delta)$ is a continuity point. Thus, from the convergence of $T_\numcols$ above,
\begin{align}
    \probc*{\quantile{\mu_{ij}}(t)\in\mc{C}_\numcols(t)\ \forall t\in[\delta,1-\delta]}
    &= \probc*{T_\numcols \leq c_\alpha(\delta)} \notag \\
    &\to \probc*{\sup_{t\in[\delta,1-\delta]}\abss{Z_\infty(t)} \leq c_\alpha(\delta)}
    = 1-\alpha,
\end{align}
the last equality by the definition of $c_\alpha(\delta)$ as the $(1-\alpha)$-quantile of the
limit. \qed

\subsubsection{Proof of \protect\cref{lemma:l2-strong-approx-bounded}: Norm of average of Brownian bridges}\label{sec:proof-l2-strong-approx-bounded}
We need to show that \cref{prop:barycenters} holds with constants that do not depend on $k$. We have
\begin{align}
    \frac1n \lnorm{\frac1k \sum_{j=1}^k w_j \cdot \mathbb{B}_{j}}^2 = \frac{1}{nk} \lnorm{\underbrace{\frac{1}{\sqrt{k}} \sum_{j=1}^k w_j \cdot \mathbb{B}_{j}}_{\mbb{G}_k}}^2.
\end{align}
Clearly, $\mbb{G}_k$ is a Gaussian process with $\expect*{\mbb{G}_k(t)}=0$ and $\mbb{G}_k(0)=\mbb{G}_k(1)=0$. We also know that it has continuous sample paths and is thus bounded. However, we wish to show that it is uniformly bounded over all values of $k$. For $s,t \in [0,1]$, we have
\begin{align}
    \Cov(\mbb{G}_k(s),\mbb{G}_k(t))
    &= \frac1k \Cov\bigg(\sum_{j=1}^k w_j(s)\mbb{B}_j(s),\sum_{j=1}^k w_j(t)\mbb{B}_j(t)\bigg) \\
    &\seq{(a)} \frac1k \sum_{j=1}^k \Cov(w_j(s)\mbb{B}_j(s),w_j(t)\mbb{B}_j(t)) \\
    &= \frac1k \sum_{j=1}^k w_j(s) w_j(t) (\min(s,t) - st)
\end{align}
where $(a)$ follows from the independence of the Brownian bridges. Next, we have for all $j\in [k]$ and $t\in[0,1]$, $w_j(t) < C'$ for some universal $C'$. Thus,
\begin{align}\label{eq:gaussian-var-bound}
    \Var(\mbb{G}_k(t))
    = \Cov(\mbb{G}_k(t),\mbb{G}_k(t))
    = \frac{1}{k} \sum_{j=1}^k (w_j(t))^2 (t - t^2)
    \leq (t - t^2) C'^2
    \leq (0.5) C'^2.
\end{align}

Let $\sigma^2_{\mbb{G}_k} \defeq \sup_{t \in [0,1]} \expect*{(\mbb{G}_k(t))^2}$. Since we have that almost surely, the paths of $\mbb{G}_k$ are bounded, then by the Borell–TIS inequality \citep[Thm. 2.1.1]{adler2009random}, we have that for $u > 0$,
\begin{align}
    \mbb{P}\bigg(\sup_{t \in [0,1]} \mbb{G}_k(t) - \mbb{E}\bigg[\sup_{t \in [0,1]} \mbb{G}_k(t)\bigg] > u\bigg) 
    \leq \exp(-u^2 / (2\sigma^2_{\mbb{G}_k}))
    \sless{\cref{eq:gaussian-var-bound}} \exp(-u^2 / C'^2).
\end{align}
Rewriting, we get
\begin{align}
    \mbb{P}\bigg(\sup_{t \in [0,1]} \mbb{G}_k(t) \leq u + \mbb{E}\bigg[\sup_{t \in [0,1]} \mbb{G}_k(t)\bigg]\bigg)
    &\geq 1 - \exp(-u^2 / C'^2).
\end{align}
The final step is to provide a uniform upper bound on $\mbb{E}\big[\sup_{t \in [0,1]} \mbb{G}_k(t)\big]$. From Dudley's theorem \citep[Thm. 1.3.3]{adler2009random}, there exists a universal constant $K$ such that
\begin{align}\label{eq:dudley}
    \mbb{E}\bigg[\sup_{t \in [0,1]} \mbb{G}_k(t)\bigg] \leq K \int_0^{\mathrm{diam}([0,1])/2} \sqrt{\log(\mathcal{N}([0,1],d,\varepsilon))} d\varepsilon
\end{align}
where $d(s,t) = \big(\mbb{E}\big[(\mbb{G}_k(s) - \mbb{G}_k(t))^2\big]\big)^{1/2}$, $\mathrm{diam}([0,1])$ is the maximum distance under $d$ between two points in $[0,1]$, and $\mathcal{N}([0,1],d,\varepsilon)$ is the smallest number of balls of length $\varepsilon$ that cover $[0,1]$ under $d$. Next, we claim that there is a constant $\Tilde{C} > 0$ such that
\begin{align}\label{claim:dudley-distance-bound}
    (d(s,t))^2 \leq \Tilde{C} \abss{s - t}.
\end{align}
We prove this claim in \cref{sec:proof-dudley-distance-bound}. So, we can provide the following upper bounds:
\begin{align}\label{eq:diam-bound}
    \mathrm{diam}([0,1]) 
    = \max_{s,t \in [0,1]} d(s,t)
    \leq \sqrt{\Tilde{C}}, \quad \text{and}
\end{align}
\begin{align}\label{eq:covering-number-bound}
    \mathcal{N}([0,1],d,\varepsilon)
    \leq \mathcal{N}([0,1],\vert \cdot \vert,\varepsilon^2/\Tilde{C})
    \sless{(a)} \begin{cases}
        \frac{3\Tilde{C}}{\varepsilon^2} &\quad \text{if} \quad \varepsilon \leq \sqrt{\Tilde{C}} \\
        1 &\quad \text{if} \quad \epsilon > \sqrt{\Tilde{C}}
    \end{cases}
\end{align}
where $(a)$ follows from \citep[Eq. 4.10]{vershynin2020high}. Plugging this into \cref{eq:dudley}, we get
\begin{align}
    \mbb{E}\bigg[\sup_{t \in [0,1]} \mbb{G}_k(t)\bigg] \leq K \int_0^{\sqrt{\Tilde{C}}/2} \sqrt{\log\bigg(\frac{3\Tilde{C}}{\varepsilon^2}\bigg)} d\varepsilon \defeq \Tilde{K} < \infty.
\end{align}
Going back to the probability term, we get
\begin{align}
    \mbb{P}\bigg(\sup_{t \in [0,1]} \mbb{G}_k(t) \leq u + \Tilde{K} \bigg)
    &\geq 1 - \exp(-u^2 / C'^2).
\end{align}

Thus, we have that $\sup_{t \in [0,1]} \mbb{G}_k(t) = \order_p(1)$ with constants that do not depend on $k$ or the functions $w_j$. The process $-\mbb{G}_k$ is again a centered Gaussian process with the same covariance kernel, so the argument above applies verbatim to it and yields $\sup_{t \in [0,1]} \parenth{-\mbb{G}_k(t)} = \order_p(1)$ with the same constants. Since
\begin{align}\label{eq:two-sided-sup}
    \sup_{t \in [0,1]} \abss{\mbb{G}_k(t)}
    = \max\bigg\{\sup_{t \in [0,1]} \mbb{G}_k(t), \;\; \sup_{t \in [0,1]} \parenth{-\mbb{G}_k(t)}\bigg\},
\end{align}
we conclude that $\sup_{t \in [0,1]} \abss{\mbb{G}_k(t)} = \order_p(1)$, again with constants independent of $k$ and $\sbraces{w_j}$. As $(0,1)$ has unit length, it follows that
\begin{align}
    \Vert \mbb{G}_k\Vert^2_{L^2(0,1)}
    = \int_0^1 \mbb{G}_k^2(t) \, dt
    \leq \sup_{t \in [0,1]} \abss{\mbb{G}_k(t)}^2
    = \order_p(1).
\end{align}

\subsection{Proof of claim \protect\cref{claim:dudley-distance-bound}}\label{sec:proof-dudley-distance-bound}

First, let $t \geq s$. Since $\mbb{G}_k$ has mean 0 at any time,
\begin{align}
    (d(t,s))^2
    &= \mbb{E}[(\mbb{G}_k(t)-\mbb{G}_k(s))^2] \\
    &= \Var(\mbb{G}_k(t)-\mbb{G}_k(s)) \\
    &= \Var(\mbb{G}_k(t)) + \Var(\mbb{G}_k(s)) - 2\Cov(\mbb{G}_k(t),\mbb{G}_k(s)) \\
    &= \frac{1}{k} \sum_{j=1}^k \big[w_j(t)^2 (t -t^2) + w_j(s)^2 (s-s^2) - 2w_j(t)w_j(s)(s-st)\big]
\end{align}
Now, we just consider one summand since the bound will apply for all summands:
\begin{align}
    &\, w_j^2(t) (t -t^2) + w_j^2(s) (s-s^2) - 2w_j(t)w_j(s)(s-st) \\
    &= w_j^2(t) t - w_j^2(t) t^2 + w_j^2(s) s - w_j^2(s) s^2 - 2w_j(t)w_j(s) s + 2w_j(t)w_j(s)st \\
    &= w_j^2(s) s - 2w_j(t)w_j(s) s + w_j^2(t) s - w_j^2(t) s + w_j^2(t) t \\
    &\quad - w_j^2(t) t^2 - w_j^2(s) s^2 + 2w_j(t)w_j(s)st \\
    &= w_j^2(s) s - 2w_j(t)w_j(s) s + w_j^2(t) s - w_j^2(t) s + w_j^2(t) t - (w_j(t) t - w_j(s) s)^2 \\
    &\leq w_j^2(s) s - 2w_j(t)w_j(s) s + w_j^2(t) s - w_j^2(t) s + w_j^2(t) t \\
    &= s(w_j(t) - w_j(s))^2 + w_j^2(t) (t - s) \\
    &\leq (w_j(t) - w_j(s))^2 + w_j^2(t) (t - s) \\
    &\sless{(a)} L^2 (t-s)^2 + w_j^2(t) (t - s) \\
    &\sless{(b)} \Tilde{C} (t-s)
\end{align}
where $(a)$ follows from the $w_j$ functions being $L$-Lipschitz, and $(b)$ follows from the fact that $t,s\in[0,1]$ and there is a multiplicative constant that makes $(t - s)$ dominate $(t - s)^2$ within $[0,1]$.

\section{Proofs for the two-sided and doubly robust estimators}\label{sec:proof-two-sided-dr}

This appendix proves \cref{thm:two-sided-rate} (two-sided) and \cref{thm:dr-rate} (doubly robust). Both decompose the squared error into bias and variance as in \cref{eq:basic_decomp}, reusing the single-sided ingredients: the bias is controlled by the maximal within-neighborhood distance (proven in \cref{sec:proof-asym-bias-bound}, applied to the row and column neighborhoods separately), and the variance by the empirical-barycenter rate \cref{prop:barycenters}, which is available under the equal-sample-size simplification of \cref{sec:setup}.

\subsection{Proof of \cref{thm:two-sided-rate}: two-sided estimator}\label{sec:proof-two-sided}

Let $\bar\mu^{\textnormal{2s}}_{ij}$ be the population two-sided barycenter, with quantile function $\quantile{\bar\mu^{\textnormal{2s}}_{ij}} = (|\nn{\eta_r}{i}^r| + |\nn{\eta_c}{j}^c|)^{-1}(\sum_{u\in\nn{\eta_r}{i}^r}\quantile{\mu_{uj}} + \sum_{v\in\nn{\eta_c}{j}^c}\quantile{\mu_{iv}})$, obtained from \cref{eq:two-sided-est} by replacing each empirical quantile function with its population counterpart. As in \cref{eq:basic_decomp}, $W_2^2(\hat\mu^{\textnormal{2s}}_{ij},\mu_{ij}) \leq 2\,W_2^2(\hat\mu^{\textnormal{2s}}_{ij},\bar\mu^{\textnormal{2s}}_{ij}) + 2\,W_2^2(\bar\mu^{\textnormal{2s}}_{ij},\mu_{ij})$.

\paragraph{Bias.} Writing $K_r = |\nn{\eta_r}{i}^r|$ and $K_c = |\nn{\eta_c}{j}^c|$ and using $\quantile{\mu_{ij}} = (K_r+K_c)^{-1}(K_r+K_c)\quantile{\mu_{ij}}$, Minkowski's inequality gives
\begin{align}
    \norm{\quantile{\bar\mu^{\textnormal{2s}}_{ij}} - \quantile{\mu_{ij}}}_{L^2}
    &\leq \frac{1}{K_r+K_c}\bigg(\sum_{u\in\nn{\eta_r}{i}^r}\!\!\norm{\quantile{\mu_{uj}} - \quantile{\mu_{ij}}}_{L^2} + \sum_{v\in\nn{\eta_c}{j}^c}\!\!\norm{\quantile{\mu_{iv}} - \quantile{\mu_{ij}}}_{L^2}\bigg) \notag\\
    &\leq \max\Big(\max_{u\in\nn{\eta_r}{i}^r} W_2(\mu_{uj},\mu_{ij}),\ \max_{v\in\nn{\eta_c}{j}^c} W_2(\mu_{iv},\mu_{ij})\Big).
\end{align}
Squaring and applying the concentration bound from the proof of \cref{claim:asym-bias-bound} to the row neighborhood (in column $j$) and, by the symmetric argument under \cref{assum:latent-factors-lipschitz-col} with $c_f^{\mathrm{col}}$ in place of $c_f$, to the column neighborhood (in row $i$),
\begin{align}
    W_2^2(\bar\mu^{\textnormal{2s}}_{ij},\mu_{ij})
    = \order_p\!\Big(\eta_r + \sqrt{\tfrac{\log(2\numrows)}{\numcols p^2}} + \eta_c + \sqrt{\tfrac{\log(2\numcols)}{\numrows p^2}}\Big),
\end{align}
which gives the bias terms in \cref{eq:two-sided-rate}.

\paragraph{Variance.} The two neighborhoods are entangled with the pooled samples: the row distances \cref{eq:avg-distance-def} use the row-$i$ samples $\{Y_{iv}\}$ that the column half averages, and the column distances use the column-$j$ samples $\{Y_{uj}\}$ that the row half averages. So, we cannot treat the pooled collection as a single fixed family. Instead, we split the pooled error into its row and column halves. By Minkowski's inequality \citep[Thm.~198]{hardy1952inequalities},
\begin{align}\label{eq:two-sided-variance-split}
    \norm{\quantile{\hat\mu^{\textnormal{2s}}_{ij}} - \quantile{\bar\mu^{\textnormal{2s}}_{ij}}}_{L^2}
    \leq \frac{K_r}{K_r+K_c}\norm{\quantile{\hat\mu^{\textnormal{row}}_{ij}} - \quantile{\bar\mu^{\textnormal{row}}_{ij}}}_{L^2}
    + \frac{K_c}{K_r+K_c}\norm{\quantile{\hat\mu^{\textnormal{col}}_{ij}} - \quantile{\bar\mu^{\textnormal{col}}_{ij}}}_{L^2}.
\end{align}
Each half needs only its own neighborhood to be independent of its own samples. For the row half, the row distances \cref{eq:avg-distance-def} exclude column $j$, so $\nn{\eta_r}{i}^r$ does not depend on the samples $\{Y_{uj}\}$ being averaged. Thus, conditioning on the latent factors as in the proof of \cref{claim:barycenters-ok}, \cref{prop:barycenters} applies with $k = K_r$ and gives
\begin{align}
    \norm{\quantile{\hat\mu^{\textnormal{row}}_{ij}} - \quantile{\bar\mu^{\textnormal{row}}_{ij}}}_{L^2}^2 = \order_p\!\Big(\tfrac{1}{nK_r} + \tfrac{\log^2 n}{n^2}\Big).
\end{align}
For the column half, the column distances exclude row $i$ (\cref{subsec:two-sided-dr}), so $\nn{\eta_c}{j}^c$ does not depend on the samples $\{Y_{iv}\}$ being averaged, and the same argument with $k = K_c$ gives $\order_p(\tfrac{1}{nK_c} + \tfrac{\log^2 n}{n^2})$. The side condition $\log k = \order(\log n)$ of \cref{prop:barycenters} holds for both halves because $K_r + K_c \leq \numrows + \numcols + 2$ and $\log(\numrows\numcols) = \order(\log n)$ by assumption. The constants in both bounds depend neither on the collections nor on their sizes, so squaring \cref{eq:two-sided-variance-split} via the Cauchy-Schwarz inequality \citep[Thm.~7]{hardy1952inequalities} yields
\begin{align}
    W_2^2(\hat\mu^{\textnormal{2s}}_{ij},\bar\mu^{\textnormal{2s}}_{ij})
    &= \order_p\!\bigg(\frac{K_r^2}{(K_r+K_c)^2}\cdot\frac{1}{nK_r} + \frac{K_c^2}{(K_r+K_c)^2}\cdot\frac{1}{nK_c} + \frac{\log^2 n}{n^2}\bigg) \\
    &= \order_p\!\bigg(\frac{K_r + K_c}{n(K_r+K_c)^2} + \frac{\log^2 n}{n^2}\bigg)
    = \order_p\!\Big(\tfrac{1}{n(K_r+K_c)} + \tfrac{\log^2 n}{n^2}\Big),
\end{align}
the variance terms in \cref{eq:two-sided-rate}. Combining the bias and variance displays proves \cref{thm:two-sided-rate}. \qed

\subsection{Proof of \cref{thm:dr-rate}: doubly robust estimator}\label{sec:proof-asym-bias-bound-dr}

We bound the doubly robust error by its bias and variance, $W_2^2(\hat\mu^{\textnormal{DR}}_{ij},\mu_{ij}) \leq 2\,W_2^2(\hat\mu^{\textnormal{DR}}_{ij},\bar\mu^{\textnormal{DR}}_{ij}) + 2\,W_2^2(\bar\mu^{\textnormal{DR}}_{ij},\mu_{ij})$, where $\bar\mu^{\textnormal{DR}}_{ij} \defeq \mathcal{R}(\quantile{\bar\mu^{\textnormal{raw}}_{ij}})$ is the rearrangement of the population raw combination.

\subsubsection{Bias: product rate}

Write $\alpha_u \defeq \quantile{\mu_{uj}} - \quantile{\mu_{ij}}$ and $\beta_v \defeq \quantile{\mu_{iv}} - \quantile{\mu_{ij}}$. With $\gamma_{uv}$ as in \cref{eq:mixed-difference}, every entry satisfies the identity
\begin{align}\label{eq:dr-decomp}
    \quantile{\mu_{uv}} - \quantile{\mu_{ij}} = \alpha_u + \beta_v + \gamma_{uv},
\end{align}
where $\alpha_i = 0$, $\beta_j = 0$, and under \cref{assum:latent-factors-bilinear-lipschitz},
\begin{align}\label{eq:bilinear-bound}
    \norm{\gamma_{uv}}_{L^2(0,1)} \leq L_h \cdot W_2(\mu_{uj}, \mu_{ij}) \cdot W_2(\mu_{iv}, \mu_{ij}),
\end{align}
with $\gamma_{uj} = \gamma_{iv} = 0$ from the bilinear bound at $v = j$ or $u = i$.

\paragraph{Cancellation of main effects.}
Plugging \cref{eq:dr-decomp} into each barycenter:
\begin{align}
    \quantile{\bar\mu^{\textnormal{row}}_{ij}} - \quantile{\mu_{ij}} &= \frac{1}{|\nn{\eta_r}{i}^r|} \sum_{u \in \nn{\eta_r}{i}^r} (\alpha_u + \beta_j + \gamma_{uj}) = \frac{1}{|\nn{\eta_r}{i}^r|}\sum_u \alpha_u, \\
    \quantile{\bar\mu^{\textnormal{col}}_{ij}} - \quantile{\mu_{ij}} &= \frac{1}{|\nn{\eta_c}{j}^c|}\sum_v \beta_v, \\
    \quantile{\bar\mu^{\textnormal{both}}_{ij}} - \quantile{\mu_{ij}} &= \frac{1}{|\nn{\eta_r}{i}^r|}\sum_u \alpha_u + \frac{1}{|\nn{\eta_c}{j}^c|}\sum_v \beta_v + \frac{1}{|\nn{\eta_r}{i}^r||\nn{\eta_c}{j}^c|}\sum_{u,v} \gamma_{uv},
\end{align}
where the row barycenter simplifies via $\gamma_{uj} = 0$, the column barycenter via $\gamma_{iv} = 0$, and we used $\alpha_i = \beta_j = 0$.

Combining with the doubly robust definition $\quantile{\bar\mu^{\textnormal{raw}}_{ij}} \defeq \quantile{\bar\mu^{\textnormal{row}}_{ij}} + \quantile{\bar\mu^{\textnormal{col}}_{ij}} - \quantile{\bar\mu^{\textnormal{both}}_{ij}}$, the row and column main effects $\alpha_u$ and $\beta_v$ cancel exactly, leaving
\begin{align}\label{eq:dr-residual}
    \quantile{\bar\mu^{\textnormal{raw}}_{ij}} - \quantile{\mu_{ij}} = -\frac{1}{|\nn{\eta_r}{i}^r|\,|\nn{\eta_c}{j}^c|} \sum_{u,v} \gamma_{uv}.
\end{align}

\paragraph{Rearrangement projection.}
The function $\quantile{\bar\mu^{\textnormal{raw}}_{ij}}$ need not be non-decreasing, so we set
\begin{align}
    \quantile{\bar\mu^{\textnormal{DR}}_{ij}} \defeq \mathcal{R}(\quantile{\bar\mu^{\textnormal{raw}}_{ij}}),
\end{align}
where $\mathcal{R}$ is the monotone rearrangement operator. Since $\quantile{\mu_{ij}}$ is non-decreasing, by the improvement property of the rearrangement \citep[Proposition~4]{chernozhukov2010quantile},
\begin{align}\label{eq:rearrangement}
    \norm{\quantile{\bar\mu^{\textnormal{DR}}_{ij}} - \quantile{\mu_{ij}}}_{L^2} \leq \norm{\quantile{\bar\mu^{\textnormal{raw}}_{ij}} - \quantile{\mu_{ij}}}_{L^2}.
\end{align}

\paragraph{Bilinear bound on the residual.}
Combining \cref{eq:dr-residual,eq:rearrangement} with Minkowski's inequality \citep[Thm.~198]{hardy1952inequalities} and the bilinear bound \cref{eq:bilinear-bound},
\begin{align}
    \norm{\quantile{\bar\mu^{\textnormal{DR}}_{ij}} - \quantile{\mu_{ij}}}_{L^2}
    &\leq \frac{1}{|\nn{\eta_r}{i}^r|\,|\nn{\eta_c}{j}^c|} \sum_{u,v} \norm{\gamma_{uv}}_{L^2} \\
    &\leq L_h \cdot \frac{1}{|\nn{\eta_r}{i}^r|\,|\nn{\eta_c}{j}^c|}\sum_{u,v} W_2(\mu_{uj}, \mu_{ij}) \cdot W_2(\mu_{iv}, \mu_{ij}) \\
    &= L_h \cdot \bigg(\frac{1}{|\nn{\eta_r}{i}^r|}\sum_u W_2(\mu_{uj}, \mu_{ij})\bigg) \cdot \bigg(\frac{1}{|\nn{\eta_c}{j}^c|}\sum_v W_2(\mu_{iv}, \mu_{ij})\bigg) \\
    &\leq L_h \cdot \max_{u \in \nn{\eta_r}{i}^r} W_2(\mu_{uj}, \mu_{ij}) \cdot \max_{v \in \nn{\eta_c}{j}^c} W_2(\mu_{iv}, \mu_{ij}).
\end{align}
Squaring both sides,
\begin{align}\label{eq:dr-bias-pre}
    W_2^2(\bar\mu^{\textnormal{DR}}_{ij}, \mu_{ij}) \leq L_h^2 \cdot \max_u W_2^2(\mu_{uj}, \mu_{ij}) \cdot \max_v W_2^2(\mu_{iv}, \mu_{ij}).
\end{align}

\paragraph{Concentration.}
By the same concentration argument as in \cref{sec:proof-asym-bias-bound} applied to the row and column neighborhoods respectively, with $c_f$ and $c_f^{\mathrm{col}}$,
\begin{align}
    \max_{u \in \nn{\eta_r}{i}^r} W_2^2(\mu_{uj}, \mu_{ij}) &= \order_p\!\Big(\eta_r + \sqrt{\tfrac{\log(2\numrows)}{\numcols p^2}}\Big), \\
    \max_{v \in \nn{\eta_c}{j}^c} W_2^2(\mu_{iv}, \mu_{ij}) &= \order_p\!\Big(\eta_c + \sqrt{\tfrac{\log(2\numcols)}{\numrows p^2}}\Big),
\end{align}
as $\numcols, \numrows \to \infty$. Substituting into \cref{eq:dr-bias-pre} gives
\begin{align}\label{eq:dr-bias-final}
    W_2^2(\bar\mu^{\textnormal{DR}}_{ij}, \mu_{ij}) = \order_p\!\bigg(\!\Big(\eta_r + \sqrt{\tfrac{\log(2\numrows)}{\numcols p^2}}\Big) \cdot \Big(\eta_c + \sqrt{\tfrac{\log(2\numcols)}{\numrows p^2}}\Big)\bigg)
\end{align}
as $\numcols, \numrows \to \infty$, the product-bias term of \cref{eq:dr-rate}.

\subsubsection{Variance: non-expansiveness of the rearrangement}

By the non-expansiveness of the monotone rearrangement in $L^2$ \citep[Thm.~378]{hardy1952inequalities},
\begin{align}
    \norm{\quantile{\hat\mu^{DR}_{ij}} - \quantile{\bar\mu^{\textnormal{DR}}_{ij}}}_{L^2}
    = \norm{\mathcal{R}(\quantile{\hat\mu^{\textnormal{raw}}_{ij}}) - \mathcal{R}(\quantile{\bar\mu^{\textnormal{raw}}_{ij}})}_{L^2}
    \leq \norm{\quantile{\hat\mu^{\textnormal{raw}}_{ij}} - \quantile{\bar\mu^{\textnormal{raw}}_{ij}}}_{L^2}.
\end{align}
% This step requires no monotonicity of either argument, so it is valid even though $\quantile{\hat\mu^{\textnormal{raw}}_{ij}}$ need not be monotone, and is a different inequality from the contractivity toward a monotone target used for the bias above. 
Expanding the right-hand side via Minkowski's inequality \citep[Thm.~198]{hardy1952inequalities},
\begin{align}
    \norm{\quantile{\hat\mu^{\textnormal{raw}}_{ij}} - \quantile{\bar\mu^{\textnormal{raw}}_{ij}}}_{L^2}
    \leq \norm{\quantile{\hat\mu^{\textnormal{row}}_{ij}} - \quantile{\bar\mu^{\textnormal{row}}_{ij}}}_{L^2}
    + \norm{\quantile{\hat\mu^{\textnormal{col}}_{ij}} - \quantile{\bar\mu^{\textnormal{col}}_{ij}}}_{L^2}
    + \norm{\quantile{\hat\mu^{\textnormal{both}}_{ij}} - \quantile{\bar\mu^{\textnormal{both}}_{ij}}}_{L^2}.
\end{align}
Each term is the $L^2$ error between an empirical barycenter and its population barycenter, controlled by \cref{prop:barycenters} under the equal-sample-size simplification of \cref{sec:setup} (sample size $n$). The side condition $\log k = \order(\log n)$ of \cref{prop:barycenters} holds for all three collections, since their sizes are at most $(\numrows+1)(\numcols+1)$ and $\log(\numrows\numcols) = \order(\log n)$ by assumption. For the row term, the column-$j$ samples $\{Y_{uj}\}$ are not used to compute the row neighborhood $\nn{\eta_r}{i}^r$, and likewise the row-$i$ samples $\{Y_{iv}\}$ are not used to compute the column neighborhood $\nn{\eta_c}{j}^c$; conditioning on the latent factors as in \cref{claim:barycenters-ok} makes each a fixed collection, so \cref{prop:barycenters} applies with $k = |\nn{\eta_r}{i}^r|$ and $k = |\nn{\eta_c}{j}^c|$ respectively.

The ``both'' term requires more care: the samples $\{Y_{uv}\}_{u\in\nn{\eta_r}{i}^r, v\in\nn{\eta_c}{j}^c}$ being averaged can also enter the distance computations that select $\nn{\eta_r}{i}^r$ and $\nn{\eta_c}{j}^c$, breaking the independence that \cref{prop:barycenters} requires. Cross-fitting removes this dependence: we split the observed entries into two independent folds, use the first to compute the neighborhoods $\nn{\eta_r}{i}^r,\nn{\eta_c}{j}^c$ and the second to form the barycenters. Conditioned on the first fold the neighborhoods are fixed, the second-fold samples are independent of the selection, and \cref{prop:barycenters} applies to the collection $\{\mu_{uv}\}$ of $|\nn{\eta_r}{i}^r||\nn{\eta_c}{j}^c|$ distributions, giving an error of order $\tfrac{1}{n |\nn{\eta_r}{i}^r||\nn{\eta_c}{j}^c|} + \tfrac{\log^2 n}{n^2}$, which is dominated by the row and column terms. Combining the three terms,
\begin{align}\label{eq:dr-variance-final}
    W_2^2(\hat\mu^{DR}_{ij}, \bar\mu^{\textnormal{DR}}_{ij})
    = \order_p\!\bigg(\tfrac{1}{n\, |\nn{\eta_r}{i}^r|} + \tfrac{1}{n\, |\nn{\eta_c}{j}^c|} + \tfrac{\log^2 n}{n^2}\bigg),
\end{align}
the variance terms of \cref{eq:dr-rate}. Adding the bias \cref{eq:dr-bias-final} and the variance \cref{eq:dr-variance-final} via the decomposition above proves \cref{thm:dr-rate}.

\subsection{Proof of claim \protect\cref{claim:asym-bias-bound}: Asymptotic bias bound}\label{sec:proof-asym-bias-bound}

We start by breaking up the bias into its components:
\begin{align}
    W_2(\Bar{\mu}_{ij}, \mu_{ij})
    \seq{\cref{eq:wasserstein-1dim}} \left\lVert \quantile{\Bar{\mu}_{ij}} - \quantile{\mu_{ij}}\right\lVert_{L^2(0,1)}
    &\seq{\cref{eq:barycenter-quantile}} \left\lVert \frac{1}{\lvert \nn{\eta}{i} \rvert} \sum_{k \in \nn{\eta}{i}}\left(\quantile{\mu_{kj}}- \quantile{\mu_{ij}}\right)\right\lVert_{L^2(0,1)} \\
    &= \frac{1}{\lvert \nn{\eta}{i} \rvert} \left\lVert  \sum_{k \in \nn{\eta}{i}}\left(\quantile{\mu_{kj}}- \quantile{\mu_{ij}}\right)\right\lVert_{L^2(0,1)} \\
    &\sless{(a)} \frac{1}{\lvert \nn{\eta}{i} \rvert} \sum_{k \in \nn{\eta}{i}} \left\lVert  \quantile{\mu_{kj}}- \quantile{\mu_{ij}}\right\lVert_{L^2(0,1)} \\
    &\leq \frac{\lvert \nn{\eta}{i} \rvert}{\lvert \nn{\eta}{i} \rvert} \max_{k \in \nn{\eta}{i}} \left\lVert  \quantile{\mu_{kj}}- \quantile{\mu_{ij}}\right\lVert_{L^2(0,1)} \\
    &\seq{\cref{eq:wasserstein-1dim}} \max_{k\in \nn{\eta}{i}} W_2(\mu_{kj}, \mu_{ij}).
\end{align}
where $(a)$ follows from Minkowski's inequality \citep[Thm. 198]{hardy1952inequalities}. By nonnegativity and squaring, we have
\begin{align}
    W_2^2(\Bar{\mu}_{ij}, \mu_{ij})
    &\leq \max_{k\in \nn{\eta}{i}} W_2^2(\mu_{kj}, \mu_{ij}).
\end{align}
Next, we claim that
% \begin{claim}[Expected row-wise distance]
    \begin{align}
        \expect*{\rowdist{i}{u} | \latRow{i},\latRow{u}} = \expect*{W_2^2(Y_{ij}, Y_{uj})| \latRow{i},\latRow{u}}
        \label{claim:expected-dist}
    \end{align}
% \end{claim}
which we prove at the end of this section. Since the latent spaces are bounded and the latent function $f$ is Lipschitz, then the space of distributions is also bounded in Wasserstein distance. Since the space is bounded in Wasserstein distance and since each distribution has finite support, then there exists a universal constant $y_{\max}$ such that $W_2^2(Y_{ij},Y_{uj}) \leq y_{\max}$. Thus, we have\footnote{The Orlicz $\psi_2$ norm is defined as $\lVert X \rVert_{\psi_2} = \inf\sbraces{t>0 : \expect*{\exp(\lvert X\rvert^2/t^2)} \leq 2}$.}
\begin{align}
    \lVert W_2^2(Y_{ij},Y_{uj})\rVert_{\psi_2} \leq \frac{y_{\max}}{\sqrt{\ln(2)}} = K.
\end{align}
So, by the Hoeffding Inequality (Theorem 2.6.3 in \citep{vershynin2020high}), we have that for any fixed row $u$,
\begin{align}
    \probc*{\left\lvert \rowdist{i}{u} - \expect*{W_2^2(Y_{ij}, Y_{uj})}\right\rvert \geq t | \sharedcol{i}{u}, \latRow{i},\latRow{u}} 
    &\leq 2 \exp\parenth{-c \frac{t^2}{K^2} \abss{\sharedcol{i}{u}}}.
\end{align}
So, by total probability, we have
\begin{align}
    &\, \probc*{\abss{\rowdist{i}{u} - \expect*{W_2^2(Y_{ij}, Y_{uj})}} \geq t | \latRow{i},\latRow{u}} \\
    &\leq \probc*{\abss{\rowdist{i}{u} - \expect*{W_2^2(Y_{ij}, Y_{uj})}} \geq t | \lvert \sharedcol{i}{u} \rvert \geq \frac{1}{2}\numcols p^2,\latRow{i},\latRow{u}} \\
    &\quad + \probc*{\lvert \sharedcol{i}{u} \rvert \leq \frac{1}{2}\numcols p^2,\latRow{i},\latRow{u}} \\
    &\leq 2 \exp\parenth{-c \frac{t^2}{2K^2} \numcols p^2} + \exp\parenth{-\frac{1}{8} \numcols p^2}.
\end{align}
Taking a union bound, we can remove the conditioning to get
\begin{align}\label{eq:dist-concentration}
    &\, \probc*{\max_{k \neq i}\left\lvert \rowdist{i}{k} - \expect*{W_2^2(Y_{ij}, Y_{kj})}\right\rvert \leq t} \\
    &\geq 1 - 2 \numrows \exp\parenth{-c \frac{t^2}{2K^2} \numcols p^2} - \numrows \exp\parenth{-\frac{1}{8} \numcols p^2}.
\end{align}
Denote the event above as $\mathcal{E}$. Recall the ratio constant $c_f$ from \cref{eq:cf-def}. When $\numrows$ is fixed, as in \cref{thm:main-asymptotic}, we do not need $c_f$ to be finite. The maximum of the ratio over the realized row pairs is almost surely finite, because there are finitely many pairs and a pair with a vanishing denominator has a vanishing numerator almost surely, since $\latCol{j}$ is itself drawn from the column distribution. When $\numrows \to \infty$, the statements assume $c_f < \infty$. In both cases, writing $c_f$ for the relevant constant, we have
\begin{align}\label{eq:cf}
    W_2^2(\mu_{kj}, \mu_{ij}) \leq c_f \expect_{x \sim \Hcol}*{W_2^2(f(\latRow{k},x), f(\latRow{i},x))},
\end{align}
where $x \sim \Hcol$ means that $x$ is drawn from the distribution over the column latent space. Next, we have from \citep[Lem. 3]{papp2022bounds} that for any distributions $\mu$ and $\nu$ with respective empirical distributions $\hat{\mu}$ and $\hat{\nu}$ derived from $n$ samples each,
\begin{align}\label{eq:empirical-expected-bound}
    \expect*{W_2^2(\hat{\mu}, \hat{\nu})} \geq W_2^2(\mu,\nu)
\end{align}
where the expectation is taken over the randomness in the sampling. So, we have that on event $\mathcal{E}$,
\begin{align}
    \max_{k\in \nn{\eta}{i}} W_2^2(\mu_{kj}, \mu_{ij})
    &\sless{\cref{eq:cf}} \max_{k\in \nn{\eta}{i}} c_f \expect_{x \sim \Hcol}*{W_2^2(f(\latRow{k},x), f(\latRow{i},x))} \\
    &\sless{\cref{eq:empirical-expected-bound}} \max_{k\in \nn{\eta}{i}} c_f \expect_{\Hcol}*{\expect*{W_2^2(Y_{kj}, Y_{ij})}} \\
    &\sless{\cref{eq:dist-concentration}} \max_{k\in \nn{\eta}{i}} c_f(\rowdist{i}{k} + t) \\
    &\leq c_f (\eta + t).
\end{align}
Putting this together, we have
\begin{align}
    &\, \probc*{W_2^2\parenth{\Bar{\mu}_{ij}, \mu_{ij}} \leq c_f(\eta + t) | \lvert \nn{\eta}{i} \rvert \geq 1} \\
    &\geq 1 - 2 \numrows \exp\parenth{-c \frac{t^2}{2K^2} \numcols p^2} - \numrows \exp\parenth{-\frac{1}{8} \numcols p^2}.
\end{align}
Finally, it can be easily verified that by setting
\begin{align}
    t = K\sqrt{\frac{2}{c \numcols p^2}\log\parenth{\frac{2\numrows}{\delta}}}
\end{align}
we get
\begin{align}
    &\, \probc*{2W_2^2\parenth{\Bar{\mu}_{ij}, \mu_{ij}} \leq 2c_f\parenth{\eta + K\sqrt{\frac{2}{c \numcols p^2}\log\parenth{\frac{2\numrows}{\delta}}}} | \lvert \nn{\eta}{i} \rvert \geq 1} \\
    &\geq 1 - \delta - \numrows \exp\parenth{-\frac{1}{8} \numcols p^2}.
\end{align}
Now, we will prove that this quantity is bounded in probability as $\numcols \to \infty$. Let $0 < \epsilon < 1$ be arbitrary. Then, there exists a $\delta$ and large enough $\numcols_0$ such that for all $\numcols \geq \numcols_0$, $\epsilon \leq \delta + \numrows \exp\parenth{-\frac{1}{8} \numcols p^2}$. Plugging this into the inequality above, we get that
\begin{align}
    &\, \probc*{W_2^2\parenth{\Bar{\mu}_{ij}, \mu_{ij}} \leq c_f\parenth{\eta + \sqrt{\frac{2K^2}{c \numcols p^2}\log\parenth{\frac{2\numrows}{\delta}}}} | \lvert \nn{\eta}{i} \rvert \geq 1} \\
    &\geq 1 - \delta - \numrows \exp\parenth{-\frac{1}{8} \numcols p^2} \\
    &\geq 1 - \epsilon.
\end{align}
Rewriting the left-hand side, we get
\begin{align}
    W_2^2\parenth{\Bar{\mu}_{ij}, \mu_{ij}} &\leq c_f\parenth{\eta + K\sqrt{\frac{2}{c \numcols p^2}\log\parenth{\frac{2\numrows}{\delta}}}} \\
    W_2^2\parenth{\Bar{\mu}_{ij}, \mu_{ij}} - c_f \eta &\leq c_f K\sqrt{\frac{2}{c \numcols p^2}\log\parenth{\frac{2\numrows}{\delta}}} \\
    \frac{W_2^2\parenth{\Bar{\mu}_{ij}, \mu_{ij}} - c_f \eta}{\sqrt{2(c\numcols p^2)^{-1}}} &\leq c_f K \sqrt{\log(2\numrows / \delta)}.
\end{align}
Thus, since we keep $\numrows$ fixed, then by the definition of bounded in probability we have
\begin{align}
    W_2^2\parenth{\Bar{\mu}_{ij}, \mu_{ij}} = \order_p \big(\eta + (\numcols p^2)^{-1/2}\big) \quad \text{as} \quad \numcols \to \infty.
\end{align}

% Finally, since in the statement of \cref{thm:main-asymptotic}, we condition on $\abss{\sharedcol{i}{u}} \geq \frac{1}{2}\numcols p^2$ for all $u \in [\numrows + 1]$, then we can remove the third term in the probability bound to get:
% \begin{align}
%     \probc*{2W_2^2\parenth{\Bar{\mu}_{ij}, \mu_{ij}} \leq 2c_f\parenth{\eta + K\sqrt{\frac{2}{c \numcols p^2}\log\parenth{\frac{2\numrows}{\delta}}}} | \lvert \nn{\eta}{i} \rvert \geq 1} 
%     \geq 1 - \delta.
% \end{align}

\subsection{Proof of claim \protect\cref{claim:expected-dist}}\label{sec:proof-expected-dist}
% {Claim}: Expected row-wise distance}

Recall the definition of $\rowdist{i}{u}$ in \cref{eq:avg-distance-def}.
% :
% \begin{align}
%     \rowdist{i}{u} 
%     &\overset{\triangle}{=} \begin{cases}
%         \frac{1}{\lvert\sharedcol{i}{u}\rvert}\sum_{j \in \sharedcol{i}{u}} W_2^2(Y_{ij}, Y_{uj}) &\quad \text{if} \, \lvert\sharedcol{i}{u}\rvert \geq 1 \\
%         \infty &\quad \text{if}\, \lvert\sharedcol{i}{u}\rvert = 0.
%     \end{cases}
% \end{align}
Let $\mathbb{S}$ denote the randomness from sampling $n$ points from each distribution. Next, we only need to consider the case of $\abss{\sharedcol{i}{u}} \geq 1$ because we give a high-probability bound on this quantity being large. Under this, we have
\begin{align}
    \expect*{\rowdist{i}{u} | \sharedcol{i}{u}, \latRow{i}, x_{\mathrm{row}}^{(u)}}
    &= \expect*{\frac{1}{\lvert\sharedcol{i}{u}\rvert}\sum_{j \in \sharedcol{i}{u}} W_2^2(Y_{ij}, Y_{uj})| \sharedcol{i}{u}, \latRow{i}, x_{\mathrm{row}}^{(u)}} \\
    &=\frac{1}{\abss{\sharedcol{i}{u}}}\sum_{j \in \sharedcol{i}{u}} \expect*{W_2^2(Y_{ij}, Y_{uj}) | \sharedcol{i}{u}, \latRow{i}, x_{\mathrm{row}}^{(u)}} \\
    &= \frac{1}{\lvert\sharedcol{i}{u}\rvert}\sum_{j \in \sharedcol{i}{u}} \expect*{W_2^2(Y_{ij}, Y_{uj}) | \latRow{i}, x_{\mathrm{row}}^{(u)}}
\end{align}
where the last line follows from the independence of the missingness from the distributions. This expectation is taken over two sources of randomness: the distribution over $\Hcol$ and the distribution over the sampling from each distribution, $\mbb S$. Next, since each sample is drawn i.i.d.\ and since each column vector is also drawn i.i.d then we have that $\expect*{W_2^2(Y_{ij}, Y_{uj}) | \latRow{i}, x_{\mathrm{row}}^{(u)}}$ is constant across column latent vectors. Thus, we have that
\begin{align}
    \expect*{\rowdist{i}{u} | \sharedcol{i}{u}, \latRow{i}, \latRow{u}}
    &= \expect*{W_2^2(Y_{ij}, Y_{uj}) | \latRow{i}, \latRow{u}}.
\end{align}
Since this holds for all $\sharedcol{i}{u}$, we can remove the conditioning to get our claim:
\begin{align}
    \expect*{\rowdist{i}{u} | \latRow{i}, \latRow{u}}
    &= \expect*{W_2^2(Y_{ij}, Y_{uj}) | \latRow{i}, \latRow{u}}.
\end{align}

\section{Proof of \protect{\cref{thm:main-brownian-bridge-feasible}} - Asymptotic distribution with difference estimator}\label{sec:proof-main-brownian-bridge-feasible}

We first record the following auxiliary normality result, which corresponds to the same statistic normalized by the population variance term $\sigma_{\nn{\eta_\numcols}{i}}(t)$.

\begin{lemma}[Auxiliary asymptotic normality of $\quantile{\Hat{\mu}_{ij}}(t)$]\label{thm:main-brownian-bridge}
    Let \cref{assum:latent-factors-lipschitz}, \cref{assum:mcar}, and \cref{assum:regular} hold. Let the sequence ${\{n_{j,\numcols},\numrows_\numcols, \eta_\numcols,\nn{\eta_\numcols}{i}\}}_{\numcols=1}^\infty$ satisfy \cref{eq:bb-assum1}. Then, we have that for almost all $t \in (0,1)$,
    \begin{align}\label{eq:normality}
        \frac{\sqrt{n_{j,\numcols} \abss{\nn{\eta_\numcols}{i}}}}{\sigma_{\nn{\eta_\numcols}{i}}(t)}
        \big(\quantile{\Hat{\mu}_{ij}}(t) - \quantile{\mu_{ij}}(t)\big)
        \overset{d}{\to} \mathcal{N}(0,1) \quad \text{as} \quad \numcols\to\infty.
    \end{align}
\end{lemma}

\subsection{Proof of \protect\cref{thm:main-brownian-bridge}}\label{sec:proof-thm-brownian-bridge}
Let $t \in (0,1)$. We prove the convergence in \cref{eq:normality}. First, we can do a similar bias-variance decomposition as in \cref{eq:bias-variance}:
\begin{align}
    \quantile{\Hat{\mu}_{ij}}(t) - \quantile{\mu_{ij}}(t)
    &= \quantile{\Hat{\mu}_{ij}}(t) - \quantile{\Bar{\mu}_{ij}}(t) + \quantile{\Bar{\mu}_{ij}}(t) - \quantile{\mu_{ij}}(t).
\end{align}

We claim that 
\begin{align}\label{claim:gaussian-process-convergence}
    \frac{\sqrt{n_{j,\numcols}\abss{\nn{\eta_\numcols}{i}}}}{\sigma_{\nn{\eta_\numcols}{i}}(t)}\big(\quantile{\Hat{\mu}_{ij}}(t) - \quantile{\Bar{\mu}_{ij}}(t)\big) \overset{d}{\to} \mathcal{N}(0,1)
\end{align}
and 
\begin{align}\label{claim:gaussian-process-op1}
    \frac{\sqrt{n_{j,\numcols}\abss{\nn{\eta_\numcols}{i}}}}{\sigma_{\nn{\eta_\numcols}{i}}(t)}\big(\quantile{\Bar{\mu}_{ij}}(t) - \quantile{\mu_{ij}}(t)\big) = o_p(1).
\end{align}
We use the following fact to prove both claims for random variables $X_m$:
\begin{align}
    \text{if} \quad X_m = \order_p(a_m) \quad \text{and} \quad \lim_{m\to\infty} a_m = 0 \quad \text{then} \quad X_m = o_p(1). \notag
\end{align}
Putting together claims \cref{claim:gaussian-process-convergence,claim:gaussian-process-op1}, we obtain the convergence claimed in \cref{eq:normality}:
\begin{align}
    \frac{\sqrt{n_{j,\numcols}\abss{\nn{\eta_\numcols}{i}}}}{\sigma_{\nn{\eta_\numcols}{i}}(t)}\big(\quantile{\Hat{\mu}_{ij}}(t) - \quantile{\mu_{ij}}(t)\big) \overset{d}{\to} \mc{N}(0,1). \notag
\end{align}

\subsection{Proof of claim \protect\cref{claim:gaussian-process-convergence}}

First, define $\hat{q}_{uj}(t) \defeq \sqrt{n_{j,\numcols}} \big(\quantile{Y_{uj}}(t) - \quantile{\mu_{uj}}(t)\big)$. Recall from the proof of \cref{prop:barycenters} that we have that for each distribution $\mu_{uj}$ for $u \in \nn{\eta_\numcols}{i}$ and its respective approximation by a sequence of standard Brownian bridges, 
$\parenth{\mathbb{B}_{u,l}}_{l=1}^{n_{j,\numcols}}$, 
\begin{align}
    \sup_{t \in (0,1)} \bigg \vert \hat{q}_{uj}(t) - \frac{\mathbb{B}_{u,n_{j,\numcols}}(t)}{f_{\mu_{uj}}\big(\quantile{\mu_{uj}}(t)\big)} \bigg \vert \overset{a.s.}{=} \order\parenth{\frac{\log n_{j,\numcols}}{\sqrt{n_{j,\numcols}}}}.
\end{align}
Next, we have
\begin{align}
    &\sqrt{n_{j,\numcols}\abss{\nn{\eta_\numcols}{i}}}\big(\quantile{\Hat{\mu}_{ij}}(t) - \quantile{\Bar{\mu}_{ij}}(t)\big) \notag \\
    &\seq{\cref{eq:dist-nn-barycenter}} \sqrt{n_{j,\numcols}\abss{\nn{\eta_\numcols}{i}}} \bigg(\frac{1}{\abss{\nn{\eta_\numcols}{i}}} \sum_{u \in \nn{\eta_\numcols}{i}} \quantile{Y_{uj}}(t) - \quantile{\mu_{uj}}(t)\bigg) \\
    &= \frac{1}{\sqrt{\abss{\nn{\eta_\numcols}{i}}}} \sum_{u \in \nn{\eta_\numcols}{i}} \hat{q}_{uj}(t)
\end{align}
Now, we will proceed similar to the proof of \cref{prop:barycenters} by adding and subtracting the Brownian bridge approximation:
\begin{align}
    &\, \frac{1}{\sqrt{\abss{\nn{\eta_\numcols}{i}}}} \sum_{u \in \nn{\eta_\numcols}{i}} \hat{q}_{uj}(t) \\
    &= \frac{1}{\sqrt{\abss{\nn{\eta_\numcols}{i}}}} \sum_{u \in \nn{\eta_\numcols}{i}} \bigg(\hat{q}_{uj}(t) - \frac{\mathbb{B}_{u,n_{j,\numcols}}(t)}{f_{\mu_{uj}}\big(\quantile{\mu_{uj}}(t)\big)} + \frac{\mathbb{B}_{u,n_{j,\numcols}}(t)}{f_{\mu_{uj}}\big(\quantile{\mu_{uj}}(t)\big)}\bigg) \\
    &= \frac{1}{\sqrt{\abss{\nn{\eta_\numcols}{i}}}} \sum_{u \in \nn{\eta_\numcols}{i}} \bigg(\hat{q}_{uj}(t) - \frac{\mathbb{B}_{u,n_{j,\numcols}}(t)}{f_{\mu_{uj}}\big(\quantile{\mu_{uj}}(t)\big)}\bigg) + \frac{1}{\sqrt{\abss{\nn{\eta_\numcols}{i}}}} \sum_{u \in \nn{\eta_\numcols}{i}} \frac{\mathbb{B}_{u,n_{j,\numcols}}(t)}{f_{\mu_{uj}}\big(\quantile{\mu_{uj}}(t)\big)}
\end{align}
Analyzing the first sum, we have the following bound by applying the strong approximation from the proof of \cref{lemma:l2-strong-approx-Op1}:
\begin{align}
    &\quad \frac{1}{\sqrt{\abss{\nn{\eta_\numcols}{i}}}} \sum_{u \in \nn{\eta_\numcols}{i}} \bigg(\hat{q}_{uj}(t) - \frac{\mathbb{B}_{u,n_{j,\numcols}}(t)}{f_{\mu_{uj}}\big(\quantile{\mu_{uj}}(t)\big)}\bigg) \\
    &\leq \bigg \vert \frac{1}{\sqrt{\abss{\nn{\eta_\numcols}{i}}}} \sum_{u \in \nn{\eta_\numcols}{i}} \bigg(\hat{q}_{uj}(t) - \frac{\mathbb{B}_{u,n_{j,\numcols}}(t)}{f_{\mu_{uj}}\big(\quantile{\mu_{uj}}(t)\big)}\bigg)\bigg \vert \\
    &\sless{(a)} \frac{1}{\sqrt{\abss{\nn{\eta_\numcols}{i}}}} \sum_{u \in \nn{\eta_\numcols}{i}} \bigg \vert\hat{q}_{uj}(t) - \frac{\mathbb{B}_{u,n_{j,\numcols}}(t)}{f_{\mu_{uj}}\big(\quantile{\mu_{uj}}(t)\big)}\bigg \vert \\
    &\seq{\cref{eq:brownian-bridge-uniform-bound}} \order\bigg(\sqrt{\abss{\nn{\eta_{\numcols}}{i}}} \frac{\log n_{j,\numcols}}{\sqrt{n_{j,\numcols}}}\bigg) \\
    &\seq{\cref{eq:bb-assum1}} o_p(1)
\end{align}
where $(a)$ follows from the triangle inequality. Analyzing the second sum, we have
\begin{align}
    \frac{1}{\sqrt{\abss{\nn{\eta_\numcols}{i}}}} \sum_{u \in \nn{\eta_\numcols}{i}} \frac{\mathbb{B}_{u,n_{j,\numcols}}(t)}{f_{\mu_{uj}}\big(\quantile{\mu_{uj}}(t)\big)}
    &\overset{d}{=} \frac{1}{\sqrt{\abss{\nn{\eta_\numcols}{i}}}} \sum_{u \in \nn{\eta_\numcols}{i}} \frac{X_u}{f_{\mu_{uj}}\big(\quantile{\mu_{uj}}(t)\big)}
\end{align}
where $\{X_u\}_{u \in \nn{\eta_\numcols}{i}}$ are i.i.d.\ $\mc{N}(0,t-t^2)$ random variables because $\mbb{B}_{u,n_{j,\numcols}}(t) \sim \mc{N}(0,t-t^2)$ from the definition of a standard Brownian bridge \citep[Prop. 8.1.1]{ross1995stochastic}. Next, the sum of independent mean-zero Gaussian random variables is Gaussian with a variance equal to the sum of summand's variances and mean-zero. Thus, we have
\begin{align}
    \frac{1}{\sqrt{\abss{\nn{\eta_\numcols}{i}}}} \sum_{u \in \nn{\eta_\numcols}{i}} \frac{X_u}{f_{\mu_{uj}}\big(\quantile{\mu_{uj}}(t)\big)}
    &\sim \mc{N}\bigg(0, \frac{1}{\abss{\nn{\eta_\numcols}{i}}} \sum_{u \in \nn{\eta_\numcols}{i}} \frac{t-t^2}{f_{\mu_{uj}}^2\big(\quantile{\mu_{uj}}(t)\big)}\bigg)\\
    &\seq{\cref{eq:neighbor-sigma}} \mc{N}(0, \sigma^2_{\nn{\eta}{i}}(t))
\end{align}
Thus, we have
\begin{align}
    \frac{1}{\sigma_{\nn{\eta}{i}}(t)\sqrt{\abss{\nn{\eta_\numcols}{i}}}} \sum_{u \in \nn{\eta_\numcols}{i}} \frac{\mathbb{B}_{u,n_{j,\numcols}}(t)}{f_{\mu_{uj}}\big(\quantile{\mu_{uj}}(t)\big)}
    \sim \mc{N}(0,1)
\end{align}
which completes the proof of claim~\cref{claim:gaussian-process-convergence}.

\subsection{Proof of claim \protect\cref{claim:gaussian-process-op1}}

We consider the sequence
$\frac{\sqrt{n_{j,\numcols}\abss{\nn{\eta_\numcols}{i}}}}{\sigma_{\nn{\eta_\numcols}{i}}(t)}\big(\quantile{\Bar{\mu}_{ij}}(t) - \quantile{\mu_{ij}}(t)\big)$
and show that it is $o_p(1)$ for every $t \in (0,1)$.
From the proof of \cref{thm:main-asymptotic},
\begin{align}\label{eq:asym-bias}
    \big\Vert\quantile{\Bar{\mu}_{ij}} - \quantile{\mu_{ij}}\big\Vert_{L^2(0,1)}
    &\seq{\cref{eq:wasserstein-1dim}} W_2\parenth{\Bar{\mu}_{ij}, \mu_{ij}} \\
    &\seq{\cref{claim:asym-bias-bound}} \order_p\!\bigg(\!\bigg(\eta_\numcols + \sqrt{\tfrac{\log(2\numrows_\numcols)}{\numcols p^2}}\bigg)^{\!1/2}\bigg).
\end{align}
By \cref{assum:regular}, condition (ii), the density satisfies $f_{\mu_{uj}} \geq C > 0$ uniformly in $u$, and hence $\norm{(\quantile{\mu_{uj}})'}_{L^\infty} \leq 1/C \defeq L$ uniformly in $u$. Defining
\begin{align}
    g_\numcols \defeq \quantile{\Bar{\mu}_{ij}} - \quantile{\mu_{ij}},
\end{align}
we therefore have $g_\numcols \in W^{1,\infty}(0,1)$ with $\norm{g_\numcols'}_{L^\infty} \leq 2L$ uniformly in $\numcols$.

The delicate point is that we must convert the $L^2$ bias bound \cref{eq:asym-bias} into a \emph{pointwise} statement after multiplication by the diverging normalization $a_\numcols(t) \asymp \sqrt{n_{j,\numcols}\abss{\nn{\eta_\numcols}{i}}}$. We do not argue through equicontinuity of the rescaled functions $a_\numcols g_\numcols$, since rescaling by a diverging factor need not preserve a uniform modulus of continuity. Instead we first establish a \emph{deterministic} sup-norm bound on the unscaled $g_\numcols$ via the following elementary interpolation inequality, a one-dimensional instance of the Gagliardo-Nirenberg inequalities \citep[Eq.~2.2]{nirenberg1959elliptic}, and only afterwards multiply by the scalar $a_\numcols(t)$. We include a short proof to keep the appendix self-contained.

\begin{lemma}[One-dimensional interpolation]\label{lem:interpolation}
    For every $g \in W^{1,\infty}(0,1)$, writing $E \defeq \norm{g}_{L^2(0,1)}$ and $M \defeq \norm{g'}_{L^\infty(0,1)}$,
    \begin{align}\label{eq:interp-ineq}
        \norm{g}_{L^\infty(0,1)} \leq 2\,\max\big\{\,E,\; M^{1/3} E^{2/3}\,\big\}.
    \end{align}
\end{lemma}
\begin{proof}[Proof of \cref{lem:interpolation}]
    If $M = 0$ then $g$ is constant and $\norm{g}_{L^\infty} = E$, so the bound holds; assume $M > 0$. Let $t^\star \in [0,1]$ attain $m \defeq \norm{g}_{L^\infty}$. Since $g$ is $M$-Lipschitz, $\abss{g(s)} \geq m - M\abss{s - t^\star} \geq m/2$ for $s$ in the interval $I \defeq [0,1] \cap [t^\star - \tfrac{m}{2M},\, t^\star + \tfrac{m}{2M}]$, whose length satisfies $\abss{I} \geq \min\{1,\, m/(2M)\}$. Hence
    \begin{align}
        E^2 = \int_0^1 g^2 \,\geq\, \int_I g^2 \,\geq\, \frac{m^2}{4}\,\abss{I} \,\geq\, \frac{m^2}{4}\,\min\Big\{1,\, \frac{m}{2M}\Big\}.
    \end{align}
    If the minimum equals $1$, then $E^2 \geq m^2/4$, so $m \leq 2E$. If it equals $m/(2M)$, then $E^2 \geq m^3/(8M)$, so $m \leq 2\,M^{1/3} E^{2/3}$. In both cases $m \leq 2\max\{E,\, M^{1/3}E^{2/3}\}$. \qed
\end{proof}

Note that \cref{eq:interp-ineq} is a deterministic inequality applied to the fixed function $g_\numcols$. It involves no limiting or equicontinuity argument, so the subsequent multiplication by the scalar $a_\numcols(t)$ poses no difficulty. Applying \cref{eq:interp-ineq} with $M = \norm{g_\numcols'}_{L^\infty} \leq 2L$ and using $E = \norm{g_\numcols}_{L^2} = o_p(1)$ by \cref{eq:asym-bias} (so that for $E \le 1$ the maximum is the second term, and $E^{2/3} \le E^{1/2}$), we obtain
\begin{align}\label{eq:sup-norm-rate}
    \norm{g_\numcols}_{L^\infty(0,1)}
    = \order_p\!\big(\norm{g_\numcols}_{L^2(0,1)}^{1/2}\big)
    = \order_p\!\bigg(\!\bigg(\eta_\numcols + \sqrt{\tfrac{\log(2\numrows_\numcols)}{\numcols p^2}}\bigg)^{\!1/4}\bigg).
\end{align}

Now fix $t \in (0,1)$. The density $f_{\mu_{uj}}$ is uniformly bounded above on its support: by \cref{assum:regular}, condition (v), it is non-decreasing in a right-neighborhood of $a$ and non-increasing in a left-neighborhood of $b$, which caps it near each endpoint, while on the remaining compact subinterval it is bounded by continuity (condition (i)); the support itself is bounded by condition (i). Hence $1/f_{\mu_{uj}}^2$ is uniformly bounded below by a positive constant. Therefore
\begin{align}
    \sigma^2_{\nn{\eta_\numcols}{i}}(t) \geq c\, t(1-t) > 0 \quad \text{for all } t \in (0,1),
\end{align}
for some $c > 0$ independent of $\numcols$, and consequently
\begin{align}
    a_\numcols(t)
    \defeq \frac{\sqrt{n_{j,\numcols}\abss{\nn{\eta_\numcols}{i}}}}{\sigma_{\nn{\eta_\numcols}{i}}(t)}
    = \order\!\big(\!\sqrt{n_{j,\numcols}\abss{\nn{\eta_\numcols}{i}}}\big).
\end{align}
Combining with \cref{eq:sup-norm-rate},
\begin{align}
    \abss{a_\numcols(t)\, g_\numcols(t)}
    &\leq a_\numcols(t)\, \norm{g_\numcols}_{L^\infty} \\
    &= \order_p\!\bigg(\!\sqrt{n_{j,\numcols}\abss{\nn{\eta_\numcols}{i}}}\, \bigg(\eta_\numcols + \sqrt{\tfrac{\log(2\numrows_\numcols)}{\numcols p^2}}\bigg)^{\!1/4}\bigg) \\
    &\seq{\cref{eq:bb-assum1}} o_p(1).
\end{align}
Thus, for every $t \in (0,1)$,
\begin{align}
    \frac{\sqrt{n_{j,\numcols}\abss{\nn{\eta_\numcols}{i}}}}{\sigma_{\nn{\eta_\numcols}{i}}(t)}\, \big(\quantile{\Bar{\mu}_{ij}}(t) - \quantile{\mu_{ij}}(t)\big) = o_p(1),
\end{align}
as claimed.

\subsection{Proof of \protect\cref{thm:main-brownian-bridge-feasible}}

By \cref{thm:main-brownian-bridge} and Slutsky's theorem, it suffices to show
\begin{align}\label{eq:slutsky-target}
    \frac{\widehat{\sigma}_{\nn{\eta_\numcols}{i}}(t)}{\sigma_{\nn{\eta_\numcols}{i}}(t)} \xrightarrow{p} 1.
\end{align}
Since $\sigma^2_{\nn{\eta_\numcols}{i}}(t)$ is bounded above and below by positive constants for each fixed $t \in (0,1)$ (from \cref{assum:regular} and the lower-bound argument in the proof of \cref{thm:main-brownian-bridge}), it is enough to show
\begin{align}\label{eq:sigma-diff}
    \widehat{\sigma}^2_{\nn{\eta_\numcols}{i}}(t) - \sigma^2_{\nn{\eta_\numcols}{i}}(t) \xrightarrow{p} 0.
\end{align}
Writing $Q_u \defeq \quantile{\mu_{uj}}$ and $\widehat{Q}_u \defeq \quantile{Y_{uj}}$ for the observed empirical quantile function of entry $(u,j)$, to lighten notation, we have
\begin{align}
    \widehat{\sigma}^2_{\nn{\eta_\numcols}{i}}(t) - \sigma^2_{\nn{\eta_\numcols}{i}}(t)
    = \frac{t(1-t)}{\abss{\nn{\eta_\numcols}{i}}} \sum_{u \in \nn{\eta_\numcols}{i}} \big(D_u(t)^2 - (Q_u'(t))^2\big),
\end{align}
where
\begin{align}
    D_u(t) \defeq \frac{\widehat{Q}_u(t + h_\numcols) - \widehat{Q}_u(t - h_\numcols)}{2 h_\numcols}.
\end{align}
Using the identity $a^2 - b^2 = (a-b)(a+b)$ and the fact that $\norm{Q_u'}_{L^\infty} \leq L$ uniformly in $u$ (by \cref{assum:regular}, condition (ii)), it suffices to show
\begin{align}\label{eq:Du-target}
    \max_{u \in \nn{\eta_\numcols}{i}} \abss{D_u(t) - Q_u'(t)} \xrightarrow{p} 0.
\end{align}

Decompose
\begin{align}
    D_u(t) - Q_u'(t) = A_u + B_u,
\end{align}
where
\begin{align}
    A_u &\defeq \frac{Q_u(t + h_\numcols) - Q_u(t - h_\numcols)}{2 h_\numcols} - Q_u'(t), \\
    B_u &\defeq \frac{[\widehat{Q}_u(t + h_\numcols) - Q_u(t + h_\numcols)] - [\widehat{Q}_u(t - h_\numcols) - Q_u(t - h_\numcols)]}{2 h_\numcols}.
\end{align}

\paragraph{Bias term $A_u$.}
By \cref{assum:regular} (condition (iv)), $Q_u'$ is Lipschitz with universal constant $L'$, so by a second-order Taylor expansion of $Q_u$ around $t$,
\begin{align}
    \abss{A_u} \leq L' h_\numcols \to 0
\end{align}
uniformly in $u$ and deterministically, by \cref{eq:bb-assum2}.

\paragraph{Plug-in term $B_u$.}
By the triangle inequality, $\abss{B_u} \leq \norm{\widehat{Q}_u - Q_u}_{L^\infty(0,1)} / h_\numcols$. Write $k = \abss{\nn{\eta_\numcols}{i}}$ and $n = n_{j,\numcols}$. From the uniform strong approximation \cref{eq:brownian-bridge-uniform-bound} applied to the neighborhood collection (the side condition $\log k = \order(\log n)$ holds as discussed after \cref{claim:barycenters-ok}), we have almost surely
\begin{align}
    \max_{u \in \nn{\eta_\numcols}{i}} \norm{\widehat{Q}_u - Q_u}_{L^\infty(0,1)}
    \leq \frac{1}{\sqrt{n}} \max_{u \in \nn{\eta_\numcols}{i}} \sup_{0<t<1} \abss{\frac{\mbb{B}_{u,n}(t)}{f_{\mu_{uj}}(Q_u(t))}} + \order\parenth{\frac{\log n}{n}}.
\end{align}
From condition (ii) of \cref{def:regular-measure}, each weight $1/(f_{\mu_{uj}} \circ Q_u)$ is bounded by $1/C$. Next, since $\probc*{\sup_{0<t<1}\abss{\mbb{B}(t)} > z} \leq 2e^{-2z^2}$ for a standard Brownian bridge, a union bound over the $k$ bridges gives $\max_{u} \sup_{t} \abss{\mbb{B}_{u,n}(t)} = \order_p(\sqrt{\log k}) = \order_p(\sqrt{\log n})$. Thus, $\max_{u \in \nn{\eta_\numcols}{i}} \abss{B_u} = \order_p\big(\tfrac{1}{h_\numcols}\sqrt{\log n / n}\big) = o_p(1)$ by \cref{eq:bb-assum2}. Both the bound on $A_u$ and this bound on $B_u$ are uniform in $t \in (0,1)$.

\paragraph{Conclusion.}
Combining the bounds for $A_u$ and $B_u$,
\begin{align}
    \max_{u \in \nn{\eta_\numcols}{i}} \abss{D_u(t) - Q_u'(t)} \xrightarrow{p} 0,
\end{align}
which establishes \cref{eq:Du-target} and thus \cref{eq:sigma-diff}. The conclusion \cref{eq:normality-feasible} then follows from \cref{thm:main-brownian-bridge} and Slutsky's theorem.

Finally, we verify the uniform consistency \cref{eq:uniform-sigma-consistency} used in \cref{cor:confidence-band}. The bounds on $A_u$ and $B_u$ above are uniform in $t \in (0,1)$, and $\abss{D_u(t)} \leq \abss{Q_u'(t)} + \abss{D_u(t) - Q_u'(t)}$ is bounded uniformly in $t$ and $u$, so $\sup_{t \in [\delta, 1-\delta]} \abss{\widehat{\sigma}^2_{\nn{\eta_\numcols}{i}}(t) - \sigma^2_{\nn{\eta_\numcols}{i}}(t)} = o_p(1)$. Since $\sigma^2_{\nn{\eta_\numcols}{i}}(t) \geq \delta(1-\delta)/C'^2$ on $[\delta,1-\delta]$ by \cref{eq:sigma-lower-bound}, dividing gives $\sup_{t\in[\delta,1-\delta]} \abss{\widehat{\sigma}_{\nn{\eta_\numcols}{i}}(t)/\sigma_{\nn{\eta_\numcols}{i}}(t) - 1} = o_p(1)$, which is \cref{eq:uniform-sigma-consistency}.

\section{Corollaries}\label{sec:corr-proofs}

Let $\phi(x,r) = \probc_{\Tilde{x}\sim\Hrow}*{\expect_{v\sim\Hcol}*{W_2^2(f(x,v),f(\Tilde{x},v))} \leq r}$. We require the following lemma to remove the conditioning. It gives a high-probability lower bound on the number of neighbors (proven in \cref{sec:proof-lemma-nn-lower-bound}):
\begin{lemma}[Lower bound on number of neighbors]\label{lemma:nn-lower-bound}
    Let $n_{-j}$ be the number of samples in each matrix entry not in column $j$. Let there exist constants $c_1$ and $K$ such that $\eta' \geq \frac{6c_1}{n_{-j}}$ and $\eta \geq \eta' + K \sqrt{\frac{4 \log(\numrows)}{c\numcols p^2}}$. Let $\nn{\eta}{i}$ be the nearest neighbors for row $i$ and write $\phi_i \defeq \phi\parenth{\latRow{i}, \frac{\eta'}{3}-\frac{6c_1}{n_{-j}}}$. Then we have
    \begin{align}
        \probc*{\abss{\nn{\eta}{i}} \geq \frac{1}{4}\numrows p\, \phi_i  |  \latRow{i}}
        &\geq 1 - \frac{2}{\numrows^2} - 2\exp\parenth{-\frac{\numcols p^2}{8}} - \exp\parenth{-\frac{\numrows p\, \phi_i}{16}}.
    \end{align}
\end{lemma}
\subsection{Latent factors drawn from uniform hypercube}

Here, we provide a general corollary where the latent factors are drawn from a uniform hypercube. This case covers \cref{cor:location-scale}.

\begin{corollary}[Uniform distribution on hypercube]\label{cor:unif-hypercube}
    Let $\Hrow={[0,1]}^d$ for some $d\geq 1$. Let \cref{assum:latent-factors-lipschitz}, \cref{assum:mcar}, and \cref{assum:regular} hold. Let $d_\mathrm{row}$ and $d_\mathrm{col}$ be the Euclidean distance and $\mu_{\mathrm{row}}$ be the uniform distribution. Let $p$ be fixed and $n_{-j}=n_v$ for $v\neq j$. Assume that the number of rows grows at most polynomially in the target-column sample size, i.e., $\log \numrows = \order(\log n_j)$, so that the side condition of \cref{prop:barycenters} is met for the growing neighborhood $\nn{\eta}{i}$. Conditioned on $\mathcal{E}=\{\abss{\nn{\eta}{i}} \geq \frac{1}{4} {(\numrows p)}^{\frac{2}{d+2}}\}$, we have
    \begin{align}
        W_2^2(\Hat{\mu}_{ij},\mu_{ij})
        &= \Tilde{\order}_p\parenth{\frac{1}{{(\numrows p)}^{\frac{2}{d+2}}} + \frac{1}{n_j{(\numrows p)}^{\frac{2}{d+2}}} + \frac{\log^2 n_j}{n_j^2} + \frac{1}{p\sqrt{\numcols}} + \frac{1}{n_{-j}}} \quad \text{as} \quad n_j, \numcols \to \infty, \quad \text{and}\\
        \probc*{\mc{E}} &\geq 1 - \frac{2}{\numrows^2} - 2\exp\parenth{-\frac{\numcols p^2}{8}} - \exp\parenth{-{(\numrows p)}^{\frac{2}{d+2}}/16}.
    \end{align}
\end{corollary}
\begin{proof}[Proof of \cref{cor:unif-hypercube}]
    Similar to Corollary 2 in \citep{li2019}, we define $B(x,r) \overset{\triangle}{=} \{x' \in \Hrow : d_{\mathrm{row}}(x,x') \leq r\}$ for $r > 0$. Then, we have if $d_{\mathrm{row}}(x,x') \leq \frac{1}{L}\parenth{\frac{\eta'}{3}-\frac{6c_1}{n_{-j}}}$, then by Lipschitzness of $f$, we have $\expect_{v \sim \mu_{\mathrm{col}}}*{W_2^2(f(x,v), f(x',v))} \leq \frac{\eta'}{3}-\frac{6c_1}{n_{-j}}$. Thus, we have 
    \begin{align}
        \phi\parenth{x,\frac{\eta'}{3}-\frac{6c_1}{n_{-j}}} &\geq \mu_{\mathrm{row}}\parenth{B\parenth{x,\frac{1}{L}\parenth{\frac{\eta'}{3}-\frac{6c_1}{n_{-j}}}}}
        = \mathrm{Vol}\parenth{B\parenth{x,\frac{1}{L}\parenth{\frac{\eta'}{3}-\frac{6c_1}{n_{-j}}}}}.
    \end{align}
    There are positive universal constants $\alpha$ and $\beta$ such that for any $d \geq 1$, $x \in {[0,1]}^d$, $r > 0$
    \begin{align}
        \mathrm{Vol}(B(x,r)) \geq \min(1, \alpha \beta^d r^d).
    \end{align}
    Plugging this into the inequality above, we have
    \begin{align}
        \phi\parenth{x,\frac{\eta'}{3}-\frac{6c_1}{n_{-j}}}
        &\geq \min\sbraces{1,\alpha \beta^d \parenth{\frac{1}{L}\parenth{\frac{\eta'}{3}-\frac{6c_1}{n_{-j}}}}^d}
    \end{align}
    Next, $\forall \, v \in \Hcol, x,x' \in \Hrow$, $W_2(f(x,v),f(x',v)) \leq L\sqrt{d}$. Let 
    \begin{align}
        \frac{\eta'}{3} = \frac{6c_1}{n_{-j}}+ \alpha^{-2/d}\beta^{-2} L^2 {(\numrows p)}^{-2/(d+2)}.
    \end{align}
    So, we have
    \begin{align}
        \numrows p \cdot \phi\parenth{x, \frac{\eta'}{3}-\frac{6c_1}{n_{-j}}}
        \geq \numrows p \cdot \frac{\parenth{\alpha^{-2/d}\beta^{-2} L^2 {(\numrows p)}^{-2/(d+2)}}^{d/2}}{\alpha \beta^d L^d}
        = {(\numrows p)}^{\frac{2}{d+2}}.
    \end{align}
    Putting this into the bound in \cref{lemma:nn-lower-bound}, and letting $\eta'$ be equal to its lower bound, we get our result.
\end{proof}

\subsection{Proof of \protect\cref{lemma:nn-lower-bound}: Lower bound on number of neighbors}\label{sec:proof-lemma-nn-lower-bound}
We place a high-probability lower bound on the number of neighbors in order to remove the conditioning on $\lvert \nn{\eta}{i} \rvert$. We do this by finding a high-probability bound on another set which can be more easily analyzed:
\begin{align}
    \Omega \overset{\triangle}{=} \left\{u \in [\numrows+1] \setminus \{i\} : A_{uj} = 1, \expect_{\Hcol}*{\expect*{W_2^2(Y_{iv}, Y_{uv})}} \leq \eta' \right\}.
\end{align}
Next, consider the set
\begin{align}
    \Tilde{\Omega} \overset{\triangle}{=} \left\{u \in \Omega : \rowdist{i}{u} - \expect_{\Hcol}*{\expect*{W_2^2(Y_{iv}, Y_{uv})}} \leq \eta - \eta'\right\}.
\end{align}
Since for $u \in \Omega$, $\expect_{\Hcol}*{\expect*{W_2^2(Y_{iv}, Y_{uv})}} \leq \eta'$, then $\Tilde{\Omega} \subseteq \nn{\eta}{i}$. Thus, if we can provide a lower bound on $\lvert \Tilde{\Omega}\rvert$, then this provides an upper bound on $\frac{1}{\lvert \nn{\eta}{i} \rvert}$.

Next, we claim (proven at the end of this section) 
\begin{align}\label{claim:lower-bound-omega}
    \probc*{u \in \Omega | \latRow{i}} &\geq p \cdot \phi\parenth{\latRow{i}, \frac{\eta'}{3}-\frac{6c_1}{n_{-j}}}\quad \text{and} \\
    \probc*{u \in \Tilde{\Omega} | u \in \Omega} &\geq 1 - \frac{1}{\numrows^2} - \exp\left(-\frac{\numcols p^2}{8}\right).
\end{align}
Putting these together, the marginal probability satisfies
\begin{align}
    \probc*{u \in \Tilde{\Omega} | \latRow{i}}
    &= \probc*{u \in \Tilde{\Omega} | u \in \Omega, \latRow{i}}\probc*{u \in \Omega | \latRow{i}}
    \geq \left(1 - \frac{1}{\numrows^2} - \exp\left(-\frac{\numcols p^2}{8}\right)\right) \cdot p \cdot \phi_i.
\end{align}
The indicators $\sbraces{\mathbf{1}\{u \in \Tilde{\Omega}\}}_{u \neq i}$ are not independent, because every $\rowdist{i}{u}$ uses row $i$'s mask and samples. Thus, we condition on $\mc{D}_i \defeq \parenth{\latRow{i}, (A_{iv}, Y_{iv})_{v \in [\numcols+1]}}$. Given $\mc{D}_i$, the event $\{u \in \Tilde{\Omega}\}$ depends only on $(\latRow{u}, (A_{uv}, Y_{uv})_{v})$, so the indicators are i.i.d.\ across $u$ with common conditional probability $q(\mc{D}_i) \defeq \probc*{u \in \Tilde{\Omega} | \mc{D}_i}$. Writing $r(\mc{D}_i) \defeq \probc*{u \notin \Tilde{\Omega} | u \in \Omega, \mc{D}_i}$, and noting that the event $\{u \in \Omega\}$ depends only on $(\latRow{u}, A_{uj})$ and is therefore independent of the rest of $\mc{D}_i$ given $\latRow{i}$, we have $q(\mc{D}_i) \geq p\, \phi_i \parenth{1 - r(\mc{D}_i)}$. Averaging over $\mc{D}_i$, which is valid because $\{u \in \Omega\}$ is independent of $\mc{D}_i$ given $\latRow{i}$, and using the second part of claim~\cref{claim:lower-bound-omega},
\begin{align}
    \expect*{r(\mc{D}_i) | \latRow{i}} = \probc*{u \notin \Tilde{\Omega} | u \in \Omega, \latRow{i}} \leq \frac{1}{\numrows^2} + \exp\parenth{-\frac{\numcols p^2}{8}}.
\end{align}
So, by Markov's inequality, $r(\mc{D}_i) \leq \tfrac{1}{2}$ with probability at least $1 - \tfrac{2}{\numrows^2} - 2\exp(-\numcols p^2/8)$, and on this event $q(\mc{D}_i) \geq \tfrac{1}{2} p\, \phi_i$. On the same event, the binomial Chernoff bound applied conditionally on $\mc{D}_i$ gives
\begin{align}
    \probc*{\lvert \Tilde{\Omega}\rvert \leq \tfrac{\numrows}{2}\, q(\mc{D}_i)  |  \mc{D}_i}
    \leq \exp\parenth{-\frac{\numrows\, q(\mc{D}_i)}{8}}
    \leq \exp\parenth{-\frac{\numrows p\, \phi_i}{16}},
\end{align}
and $\tfrac{\numrows}{2} q(\mc{D}_i) \geq \tfrac{1}{4}\numrows p\, \phi_i$. Since $\Tilde{\Omega} \subseteq \nn{\eta}{i}$, combining the two failure probabilities gives our result.

\subsection{Proof of claim \protect\cref{claim:lower-bound-omega}}

We have that a row $u$ is in $\Omega$ with probability $p \cdot \psi(\latRow{i},\eta')$, where
\begin{align}
    \psi(\latRow{i},\eta') = \probc*{\expect*{W_2^2(Y_{ij}, Y_{uj})} \leq \eta' | \latRow{i}}.
\end{align}

By the Binomial Chernoff Bound and conditioning on the $i$-th latent row vector, we have
\begin{align}
    \probc*{\lvert \Omega \rvert = 0 | \latRow{i}}
    &\leq \probc*{\lvert \Omega \rvert \leq \frac{1}{2}\numrows p \cdot\psi (\latRow{i}, \eta')} \\
    &\leq \exp\parenth{-\frac{\numrows p}{8}\psi (\latRow{i}, \eta')}.
\end{align}
Next, there exists a universal constant $c_1$ such that for two empirical distributions $\mu_n$ and $\nu_n$ with corresponding true distributions $\mu$ and $\nu$, we have
\begin{align}
    \expect*{W_2^2(\mu_n, \nu_n)}
    &\leq 3 W_2^2(\mu, \nu) + \frac{6c_1}{n_{-j}}.
\end{align}
This follows from this line of reasoning: Let $X^{(k)}$ and $Y^{(k)}$ denote the $k$-th order statistics of the samples from $\mu$ and $\nu$ respectively. Let $\quantile\mu$ and $\quantile\nu$ denote the quantile functions of $\mu$ and $\nu$, respectively. Then, we have
\begin{align}
    &\, \expect*{W_2^2(\mu_n, \nu_n)} \\
    &= \frac{1}{n}\sum_{k=1}^n \expect*{\parenth{X^{(k)} - Y^{(k)}}^2} \\
    &= \sum_{k=1}^n \int_{(k-1)/n}^{k/n} \expect*{\parenth{X^{(k)} - \quantile\mu(t) + \quantile\mu(t) - \quantile\nu + \quantile\nu - Y^{(k)}}^2} dt \\
    &\leq 3 \sum_{k=1}^n \int_{(k-1)/n}^{k/n} \expect*{\parenth{X^{(k)} - \quantile\mu(t)}^2 + \parenth{\quantile\mu(t) - \quantile\nu}^2 + \parenth{\quantile\nu - Y^{(k)}}^2} dt \\
    &= 3 W_2^2(\mu,\nu) + 3\expect*{W_2^2(\mu_n,\mu)}+ 3\expect*{W_2^2(\nu_n,\nu)} \\ 
    &\leq 3 W_2^2(\mu,\nu) + \frac{6c_1}{n_{-j}},
\end{align}
where the last inequality uses $\expect*{W_2^2(\mu_n, \mu)} \leq c_1/n_{-j}$ for regular distributions \citep[Sec.~5]{bobkov2019one}, which applies because condition (ii) of \cref{def:regular-measure} makes the functional $J_2(\mu) = \int F(1-F)/f$ finite.
Then, let $\expect*{W_2^2(\mu_{ij}, \mu_{uj})} \leq \frac{\eta'}{3} - \frac{6c_1}{n_{-j}}$. So, we get
\begin{align}
    \expect*{W_2^2(Y_{ij}, Y_{uj})} \leq \eta'.
\end{align}
Thus, we have that the bound on $\expect*{W_2^2(\mu_{ij}, \mu_{uj})}$ implies the bound on $\expect*{W_2^2(Y_{ij}, Y_{uj})}$. So, we have
\begin{align}
    \probc*{\expect*{W_2^2(\mu_{ij}, \mu_{uj})} \leq \frac{\eta'}{3} - \frac{6c_1}{n_{-j}}}
    \leq \probc*{\expect*{W_2^2(Y_{ij}, Y_{uj})} \leq \eta'}.
\end{align}
Rewriting this in the $\phi$ and $\psi$ notation, we have that
\begin{align}
    \phi\parenth{\latRow{i}, \frac{\eta'}{3} - \frac{6c_1}{n_{-j}}} \leq \psi(\latRow{i}, \eta').
\end{align}
Thus, we obtain the first part of our claim. For the second part, we have
\begin{align}
    &\, \probc*{\rowdist{i}{u} - \expect_{\Hcol}*{\expect*{W_2^2(Y_{iv}, Y_{u v})}} > \eta - \eta'| \latRow{i}, \lvert \Omega \rvert \geq 1} \\
    &\leq \probc*{\rowdist{i}{u} - \expect_{\Hcol}*{\expect*{W_2^2(Y_{iv}, Y_{u v})}} > \eta - \eta'| \latRow{i}, \lvert \sharedcol{i}{u}\rvert \geq \frac{1}{2}\numcols p^2} \\
    &\quad + \probc*{\lvert \sharedcol{i}{u}\rvert < \frac{1}{2}\numcols p^2 | \latRow{i}} \\
    &\leq \exp\parenth{-c \frac{{(\eta - \eta')}^2}{K^2}\frac{1}{2}\numcols p^2} + \exp\parenth{-\frac{\numcols p^2}{8}} \\
    &\leq \exp\parenth{-2\log(\numrows)} + \exp\parenth{-\frac{\numcols p^2}{8}} \\
    &= \frac{1}{\numrows^2} + \exp\parenth{-\frac{\numcols p^2}{8}}
\end{align}
which completes the proof of our claim.

%% file: appendix-estimators.tex
\section{Comparing the single-sided, two-sided, and doubly robust estimators}\label{sec:sim-estimators}

\begin{figure}[t!]
    \centering
    \includegraphics[width=0.7\linewidth]{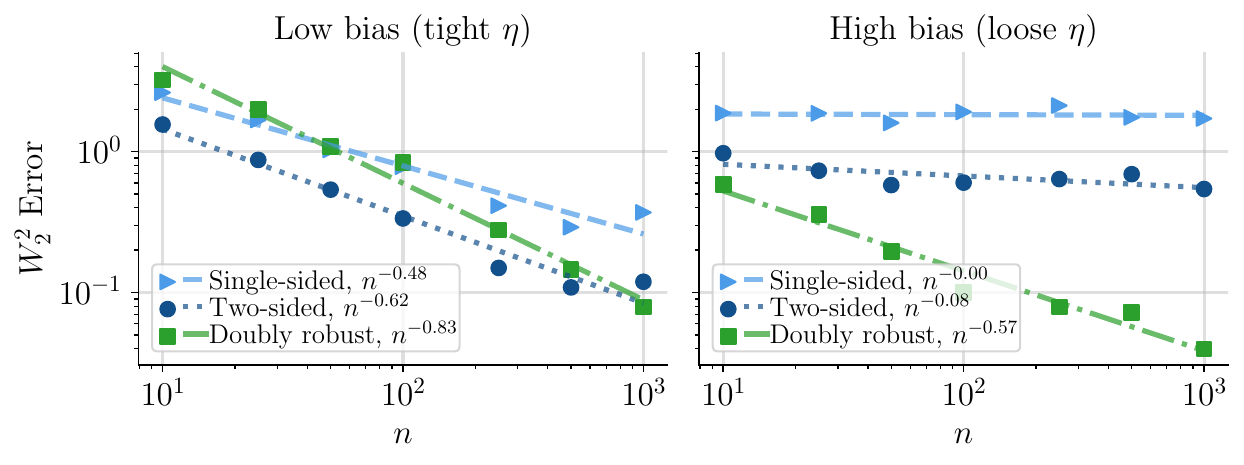}
    \caption{\tbf{The bias-variance tradeoff between the single-sided, two-sided, and doubly robust estimators, in a bilinear-interaction location-scale model.} Each panel plots the squared 2-Wasserstein error at the held-out entry against the per-cell sample size $n$, averaged over 50 runs on a $40 \times 40$ matrix with interaction strength $c = 6$. We report in the legends the least-squares power-law fits $f(n) = a\,n^b$. The two panels use the same data-generating model and differ only in the neighborhood threshold. \tbf{Left (low bias, lower $\eta$):} variance dominates, and the two-sided estimator is most accurate over the practical range. Its pooling halves the variance, and the doubly robust estimator's steeper decay (from its product-order bias) overtakes only at the largest $n$, where it finally overcomes its higher variance constant. \tbf{Right (high bias, higher $\eta$):} bias dominates. The single-sided and two-sided estimators plateau at an additive bias floor, while the doubly robust estimator's main-effect cancellation lets its error keep falling, making it best by an order of magnitude. This is the empirical counterpart of the variance reduction in \cref{thm:two-sided-rate} and the product-rate bias in \cref{thm:dr-rate}.}\label{fig:estimator-comparison}
\end{figure}

We next compare the three estimators of \cref{subsec:two-sided-dr} and confirm that their empirical behavior matches the bias-variance rates predicted by \cref{thm:main-asymptotic,thm:two-sided-rate,thm:dr-rate}. The two estimators improve on single-sided \distnn\ in complementary ways: the two-sided estimator reduces \emph{variance} by pooling the row and column neighborhoods (its effective neighbor count is $|\nn{\eta_r}{i}^r| + |\nn{\eta_c}{j}^c|$ rather than $|\nn{\eta_r}{i}^r|$), while the doubly robust estimator reduces \emph{bias} to a product of the row and column rates by cancelling the main effects, at the cost of a larger variance constant ($\tfrac{1}{n|\nn{\eta_r}{i}^r|} + \tfrac{1}{n|\nn{\eta_c}{j}^c|}$ versus the pooled $\tfrac{1}{n(|\nn{\eta_r}{i}^r| + |\nn{\eta_c}{j}^c|)}$). Which estimator is best is thus a bias-variance tradeoff that depends on how much bias the neighborhoods carry.

To demonstrate the doubly robust mechanism we need a model in which the \cref{eq:mixed-difference} effect is nonzero. The pure location-scale model of \cref{ex:location-scale-hetero} has $\gamma_{uv} \equiv 0$, so we add a bilinear interaction to the scale: entry $(i,j)$ is Gaussian with mean $\alpha_i$ and standard deviation $\sigma_j + c\,\rho_i\,\kappa_j$, where $\rho_i, \kappa_j \in [0,1]$ are latent row and column interaction factors and $c \geq 0$ controls the interaction strength. This makes $\gamma_{uv}(t) = c\,(\rho_u - \rho_i)(\kappa_v - \kappa_j)\,\Phi^{-1}(t)$, the bilinear form of \cref{assum:latent-factors-bilinear-lipschitz}. We show the bias-variance tradeoff clearly in \cref{fig:estimator-comparison}.
%  plots the squared 2-Wasserstein error of each estimator against the per-cell sample size $n$ on a $40 \times 40$ matrix with $c = 6$, in two settings that differ only in the neighborhood threshold. Lowering the threshold trades neighbors (variance) for locality (less bias). In the \emph{low-bias} setting (left panel, small threshold), variance dominates: the two-sided estimator is the most accurate across the practical range of $n$, since pooling lowers its variance and the bias it would need to cancel is already small. The doubly robust estimator has the steepest slope (its bias is suppressed to the product order) and so eventually catches up at the largest sample sizes, but it is not the best estimator over most of the range because it pays the larger variance constant. In the \emph{high-bias} setting (right panel, higher threshold), bias dominates: the single-sided and two-sided estimators both plateau at an additive bias floor that does not vanish with $n$, while the doubly robust estimator continues to decay steeply and is best by an order of magnitude. The two panels together are the empirical signature of \cref{thm:two-sided-rate,thm:dr-rate}: two-sided achieves variance reduction, doubly robust achieves bias reduction, and the better choice depends on the setting.

%% file: appendix-confidence-bands.tex
\section{Empirical validation of the simultaneous confidence bands}\label{sec:appendix-bands}

We validate the simultaneous confidence band of \cref{cor:confidence-band} numerically, in the same location-scale setting as \cref{sec:sim} (\cref{ex:location-scale-hetero}). The band improves on a Bonferroni band by replacing the grid-dependent correction with a single critical value for the supremum of the studentized barycenter quantile process. The band has half-width $c_\alpha(\delta)\,\widehat\sigma_{\nn{\eta}{i}}(t)/\sqrt{n\abss{\nn{\eta}{i}}}$, where $\widehat\sigma_{\nn{\eta}{i}}(t)$ is the difference-quotient scale of \cref{eq:sigma-hat} computed from the empirical samples alone (no access to the true neighbor distributions), and $c_\alpha(\delta)$ is a critical value for the simultaneous supremum over the interior $[\delta, 1-\delta]$. We compute $c_\alpha(\delta)$ three ways, all $\order(1)$ in the grid resolution: (i) the \emph{pivotal} universal constant from the standardized Brownian bridge $\sup_t \abss{\mbb{B}(t)}/\sqrt{t(1-t)}$, the $\eta\to 0$ limit; (ii) a \emph{plug-in kernel} value that simulates the studentized Gaussian process under the estimated finite-neighborhood covariance \cref{eq:clt-covariance}; and (iii) a \emph{bootstrap} value that resamples the neighbor samples and studentizes the centered bootstrap process. We use $\delta = 0.05$ and $\alpha = 0.05$.

\cref{fig:confidence-regions} shows the pivotal band against the Bonferroni band in the Gaussian and continuous-uniform location-scale cases; the band covers the true quantile function across the interior and is visibly narrower than Bonferroni. We plot only the pivotal band because, in this location-scale setting, the pivotal, plug-in, and bootstrap critical values nearly coincide (all near $3.0$-$3.5$), so overlaying all three would be indistinguishable; the quantitative comparison across critical-value methods and sample-size regimes is reported in \cref{tab:band-coverage}. Two points emerge: every critical value is smaller than the Bonferroni factor $z_{\alpha/(2m)} \asymp \sqrt{2\log m}$ (so the bands are uniformly narrower), and the resulting bands attain close to nominal simultaneous coverage in the larger-sample regime while under-covering when the sample size is small.

\begin{figure}[t!]
    \begin{center}
    \resizebox{0.82\linewidth}{!}{
        \begin{tabular}{cc}
            \includegraphics[width=0.45\linewidth]{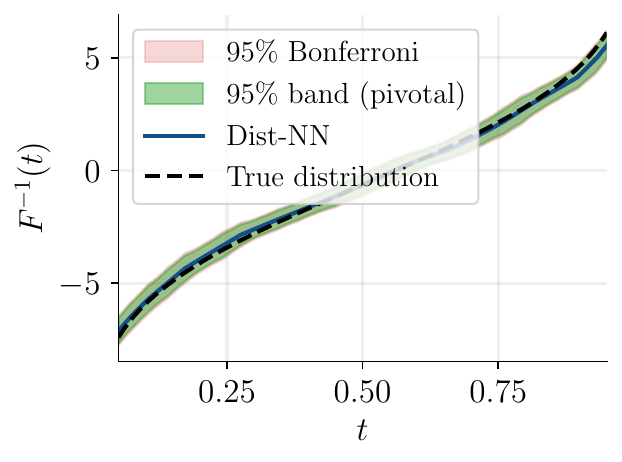} &
            \includegraphics[width=0.45\linewidth]{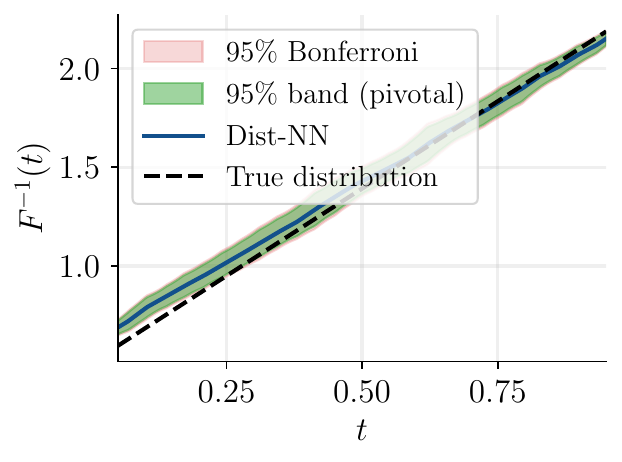}
            \\
            \parbox[c]{0.33\textwidth}{\centering (a) Gaussian} &
            \parbox[c]{0.33\textwidth}{\centering (b) Continuous uniform}
        \end{tabular}
    }
    \end{center}
    \caption{\tbf{Simultaneous confidence band for the barycenter quantile function in the Gaussian and continuous-uniform location-scale cases.} Each panel shows the Dist-NN quantile estimate, the true quantile function, the pivotal simultaneous band, and the looser Bonferroni band, all constructed as $\quantile{\Hat{\mu}_{ij}}(t) \pm c\,\widehat\sigma_{\nn{\eta}{i}}(t)/\sqrt{n\abss{\nn{\eta}{i}}}$ over the interior $[\delta, 1-\delta]$ with $\delta=\alpha=0.05$. The pivotal critical value is $\order(1)$ in the grid resolution, unlike the Bonferroni factor whose width grows like $\sqrt{2\log m}$, so the pivotal band sits inside the Bonferroni band. The plug-in and bootstrap bands (omitted here for legibility, since they nearly coincide with the pivotal band) and the across-regime coverage comparison are reported in \cref{tab:band-coverage}. The continuous-uniform distribution satisfies our regularity conditions throughout; the Gaussian density is unbounded at the boundary, which is why the band is restricted to the interior.}\label{fig:confidence-regions}
\end{figure}

\begin{table}[ht!]
    \centering
    \caption{\tbf{Empirical simultaneous coverage and mean critical value for the confidence bands}, over $300$ Monte Carlo repetitions at nominal level $1-\alpha = 0.95$, on the interior $[\delta, 1-\delta]$ with $\delta = 0.05$. Coverage is the fraction of repetitions in which the band contains the true barycenter quantile function at \emph{every} grid point; $\bar c$ is the mean critical value, which determines the band width (the per-point scale $\widehat\sigma/\sqrt{n\abss{\nn{\eta}{i}}}$ is shared across methods). The pivotal, plug-in, and bootstrap bands are all narrower than Bonferroni ($\bar c$ smaller). Coverage approaches the nominal level as $\abss{\nn{\eta}{i}}$ and $n$ grow; at small samples the bands under-cover, most for the uniform, due to empirical-quantile boundary bias (discussed below).}\label{tab:band-coverage}
    \begin{tabular}{ll cc cc cc cc}
        \toprule
        & & \multicolumn{2}{c}{Pivotal} & \multicolumn{2}{c}{Plug-in} & \multicolumn{2}{c}{Bootstrap} & \multicolumn{2}{c}{Bonferroni} \\
        \cmidrule(lr){3-4}\cmidrule(lr){5-6}\cmidrule(lr){7-8}\cmidrule(lr){9-10}
        Distribution & $(\abss{\nn{\eta}{i}}, n)$ & cov. & $\bar c$ & cov. & $\bar c$ & cov. & $\bar c$ & cov. & $\bar c$ \\
        \midrule
        Gaussian & $(20, 200)$  & $0.923$ & $3.08$ & $0.957$ & $3.18$ & $0.957$ & $3.27$ & $0.967$ & $3.34$ \\
        Gaussian & $(50, 500)$  & $0.970$ & $3.08$ & $0.987$ & $3.19$ & $0.947$ & $2.88$ & $0.990$ & $3.34$ \\
        Uniform  & $(20, 200)$  & $0.887$ & $3.08$ & $0.907$ & $3.18$ & $0.933$ & $3.43$ & $0.930$ & $3.34$ \\
        Uniform  & $(50, 500)$  & $0.933$ & $3.08$ & $0.947$ & $3.20$ & $0.927$ & $3.02$ & $0.957$ & $3.34$ \\
        \bottomrule
    \end{tabular}
\end{table}

The under-coverage at the smaller sample size is sharper for the uniform than for the Gaussian (\cref{tab:band-coverage}: the pivotal band reaches only $0.887$ for the uniform at $(\abss{\nn{\eta}{i}}, n) = (20, 200)$, versus $0.923$ for the Gaussian). Decomposing the error $\quantile{\Hat{\mu}_{ij}}(t) - \quantile{\mu_{ij}}(t)$ into bias and standard deviation, we find the studentizer is accurate ($\widehat\sigma_{\nn{\eta}{i}}(t)$ tracks the true standard deviation to within $15$-$25\%$ at every level), so the band is not too narrow on variance grounds. The miss comes entirely from \emph{bias near the support edge}: the barycenter averages empirical order statistics, and the empirical quantile is biased inward near the boundary by an $\order(1/n)$ amount. For the uniform this bias is large \emph{relative to the standard deviation} at the extreme levels, because the uniform's quantile density $1/f$ is constant and its tail standard deviation is small (an order of magnitude smaller than the Gaussian's at $t=0.05$); at $t \in \{0.05, 0.95\}$ the bias exceeds one standard deviation and does not shrink as $\abss{\nn{\eta}{i}}, n$ grow at this $\delta$, so a single near-edge level escapes the band and the \emph{simultaneous} statement fails. The Gaussian, whose quantile density grows toward the tails and inflates the tail standard deviation, has a bias-to-standard-deviation ratio of only $\approx 0.7$ there and covers correctly. This is precisely the boundary effect the theory excludes through the interior restriction $[\delta, 1-\delta]$ and the rate conditions \cref{eq:bb-assum1} that make the bias asymptotically negligible relative to the standardized scale; $\delta = 0.05$ is simply not far enough into the interior for the uniform at these sample sizes, and either enlarging $\delta$ or increasing $n$ restores nominal coverage. The bands are drawn on $[\delta, 1-\delta]$ for this reason; for the Gaussian the restriction is additionally necessary because its quantile density is unbounded at the boundary, violating the regularity conditions, whereas the continuous-uniform case satisfies them on all of $(0,1)$.

%% file: arxiv-main.bbl
\begin{thebibliography}{44}
\providecommand{\natexlab}[1]{#1}
\providecommand{\url}[1]{\texttt{#1}}
\expandafter\ifx\csname urlstyle\endcsname\relax
  \providecommand{\doi}[1]{doi: #1}\else
  \providecommand{\doi}{doi: \begingroup \urlstyle{rm}\Url}\fi

\bibitem[Adler and Taylor(2009)]{adler2009random}
Robert~J Adler and Jonathan~E Taylor.
\newblock \emph{Random fields and geometry}.
\newblock Springer Science \& Business Media, New York, NY, 2009.

\bibitem[Agarwal et~al.(2023)Agarwal, Dahleh, Shah, and
  Shen]{agarwal2023causal}
Anish Agarwal, Munther Dahleh, Devavrat Shah, and Dennis Shen.
\newblock Causal matrix completion.
\newblock In \emph{Proceedings of Thirty Sixth Conference on Learning Theory},
  volume 195, pages 3821--3826. PMLR, 2023.

\bibitem[Ahidar-Coutrix et~al.(2020)Ahidar-Coutrix, Le~Gouic, and
  Paris]{ahidarcoutrix2020}
Adil Ahidar-Coutrix, Thibaut Le~Gouic, and Quentin Paris.
\newblock Convergence rates for empirical barycenters in metric spaces:
  curvature, convexity and extendable geodesics.
\newblock \emph{Probability Theory and Related Fields}, 177\penalty0
  (1--2):\penalty0 323--368, 2020.

\bibitem[Altschuler and Boix-Adsera(2022)]{altschuler2022wasserstein}
Jason~M Altschuler and Enric Boix-Adsera.
\newblock Wasserstein barycenters are np-hard to compute.
\newblock \emph{SIAM Journal on Mathematics of Data Science}, 4\penalty0
  (1):\penalty0 179--203, 2022.

\bibitem[Bakhsh et~al.(2026)Bakhsh, Bahamdan, Ghazi, Bokhary, Binmahfooz,
  Hegazy, Alsufyani, and Boker]{bakhsh2026impact}
Abdullah Bakhsh, Sulafah Bahamdan, Rafah Ghazi, Hanaa Bokhary, Saleh
  Binmahfooz, Mohamed Hegazy, Hadeel Alsufyani, and Abdulaziz Boker.
\newblock Impact of hypotension burden on survival after in-hospital cardiac
  arrest: a retrospective observational study.
\newblock \emph{BMC Anesthesiology}, 2026.

\bibitem[Bergstra et~al.(2011)Bergstra, Bardenet, Bengio, and
  K{\'e}gl]{bergstra2011algorithms}
James Bergstra, R{\'e}mi Bardenet, Yoshua Bengio, and Bal{\'a}zs K{\'e}gl.
\newblock Algorithms for hyper-parameter optimization.
\newblock \emph{Advances in neural information processing systems},
  24:\penalty0 2546--2554, 2011.

\bibitem[Bigot and Klein(2018)]{bigot2013consistent}
J{\'e}r{\'e}mie Bigot and Thierry Klein.
\newblock Characterization of barycenters in the {W}asserstein space by
  averaging optimal transport maps.
\newblock \emph{ESAIM: Probability and Statistics}, 22:\penalty0 35--57, 2018.
\newblock \doi{10.1051/ps/2017020}.

\bibitem[Bigot et~al.(2017)Bigot, Gouet, Klein, and L{\'o}pez]{Bigot2013}
J{\'e}r{\'e}mie Bigot, Ra{\'u}l Gouet, Thierry Klein, and Alfredo L{\'o}pez.
\newblock {Geodesic PCA in the Wasserstein space by convex PCA}.
\newblock \emph{Annales de l'Institut Henri Poincar{\'e}, Probabilit{\'e}s et
  Statistiques}, 53\penalty0 (1):\penalty0 1 -- 26, 2017.

\bibitem[Bigot et~al.(2018)Bigot, Gouet, Klein, and L{\'o}pez]{bigot2018upper}
J{\'e}r{\'e}mie Bigot, Ra{\'u}l Gouet, Thierry Klein, and Alfredo L{\'o}pez.
\newblock {Upper and lower risk bounds for estimating the Wasserstein
  barycenter of random measures on the real line}.
\newblock \emph{Electronic Journal of Statistics}, 12\penalty0 (2):\penalty0
  2253 -- 2289, 2018.

\bibitem[{Bigot, J{\'e}r{\'e}mie}(2020)]{bigot2020}
{Bigot, J{\'e}r{\'e}mie}.
\newblock Statistical data analysis in the wasserstein space*.
\newblock \emph{ESAIM: ProcS}, 68:\penalty0 1--19, 2020.

\bibitem[Bobkov and Ledoux(2019)]{bobkov2019one}
Sergey Bobkov and Michel Ledoux.
\newblock \emph{One-dimensional empirical measures, order statistics, and
  Kantorovich transport distances}, volume 261 of \emph{Memoirs of the American
  Mathematical Society}.
\newblock American Mathematical Society, Providence, RI, 2019.

\bibitem[Cai and Zhou(2016)]{cai2013matrix}
T~Tony Cai and Wen-Xin Zhou.
\newblock Matrix completion via max-norm constrained optimization.
\newblock \emph{Electronic Journal of Statistics}, 10\penalty0 (1):\penalty0
  1493--1525, 2016.
\newblock \doi{10.1214/16-EJS1147}.

\bibitem[Candes and Plan(2010)]{candes2010}
Emmanuel~J. Candes and Yaniv Plan.
\newblock Matrix completion with noise.
\newblock \emph{Proceedings of the IEEE}, 98\penalty0 (6):\penalty0 925--936,
  2010.
\newblock \doi{10.1109/JPROC.2009.2035722}.

\bibitem[Carothers(2000)]{carothers2000real}
Neal~L Carothers.
\newblock \emph{Real analysis}.
\newblock Cambridge University Press, 2000.

\bibitem[Casella and Berger(2002)]{casella2002statistical}
George Casella and Roger~L. Berger.
\newblock \emph{Statistical Inference}.
\newblock Duxbury, Pacific Grove, CA, 2 edition, 2002.

\bibitem[Chatterjee(2015)]{chatterjee2015}
Sourav Chatterjee.
\newblock Matrix estimation by universal singular value thresholding.
\newblock \emph{The Annals of Statistics}, 43\penalty0 (1):\penalty0 177--214,
  2015.
\newblock ISSN 00905364.
\newblock URL \url{http://www.jstor.org/stable/43556512}.

\bibitem[Chernozhukov et~al.(2010)Chernozhukov, Fern{\'a}ndez-Val, and
  Galichon]{chernozhukov2010quantile}
Victor Chernozhukov, Iv{\'a}n Fern{\'a}ndez-Val, and Alfred Galichon.
\newblock Quantile and probability curves without crossing.
\newblock \emph{Econometrica}, 78\penalty0 (3):\penalty0 1093--1125, 2010.

\bibitem[Chi(2018)]{chi2018low}
Yuejie Chi.
\newblock Low-rank matrix completion [lecture notes].
\newblock \emph{IEEE Signal Processing Magazine}, 35\penalty0 (5):\penalty0
  178--181, 2018.

\bibitem[Cs{\'a}ki(1977)]{csaki1977law}
E~Cs{\'a}ki.
\newblock The law of the iterated logarithm for normalized empirical
  distribution function.
\newblock \emph{Zeitschrift f{\"u}r Wahrscheinlichkeitstheorie und Verwandte
  Gebiete}, 38\penalty0 (2):\penalty0 147--167, 1977.

\bibitem[Cs{\"o}rg{\H{o}} and R{\'e}v{\'e}sz(1978)]{csorgo1978strong}
Mikl{\'o}s Cs{\"o}rg{\H{o}} and P{\'a}l R{\'e}v{\'e}sz.
\newblock Strong approximations of the quantile process.
\newblock \emph{The Annals of Statistics}, 6\penalty0 (4):\penalty0 882--894,
  1978.

\bibitem[Cuturi and Doucet(2014)]{cuturi14}
Marco Cuturi and Arnaud Doucet.
\newblock Fast computation of wasserstein barycenters.
\newblock In Eric~P. Xing and Tony Jebara, editors, \emph{Proceedings of the
  31st International Conference on Machine Learning}, volume~32 of
  \emph{Proceedings of Machine Learning Research}, pages 685--693, Bejing,
  China, 22--24 Jun 2014. PMLR.
\newblock URL \url{https://proceedings.mlr.press/v32/cuturi14.html}.

\bibitem[del Barrio et~al.(2019)del Barrio, Gordaliza, Lescornel, and
  Loubes]{delbarrio2019clt}
Eustasio del Barrio, Paula Gordaliza, H{\'e}l{\`e}ne Lescornel, and Jean-Michel
  Loubes.
\newblock Central limit theorem and bootstrap procedure for {W}asserstein's
  variations with an application to structural relationships between
  distributions.
\newblock \emph{Journal of Multivariate Analysis}, 169:\penalty0 341--362,
  2019.

\bibitem[Dou et~al.(2016)Dou, Yang, and Deng]{dou2016survey}
Yingtong Dou, Hao Yang, and Xiaolong Deng.
\newblock A survey of collaborative filtering algorithms for social recommender
  systems.
\newblock In \emph{2016 12th International conference on semantics, knowledge
  and grids (SKG)}, pages 40--46. IEEE, 2016.

\bibitem[Dwivedi et~al.(2022)Dwivedi, Tian, Tomkins, Klasnja, Murphy, and
  Shah]{dwivedi2022doubly}
Raaz Dwivedi, Katherine Tian, Sabina Tomkins, Predrag Klasnja, Susan Murphy,
  and Devavrat Shah.
\newblock Doubly robust nearest neighbors in factor models.
\newblock \emph{arXiv preprint arXiv:2211.14297}, 2022.

\bibitem[Dwivedi et~al.(2025)Dwivedi, Tian, Tomkins, Klasnja, Murphy, and
  Shah]{dwivedi2022counterfactual}
Raaz Dwivedi, Katherine Tian, Sabina Tomkins, Predrag Klasnja, Susan Murphy,
  and Devavrat Shah.
\newblock Counterfactual inference in sequential experiments.
\newblock \emph{The Annals of Statistics}, 53\penalty0 (6):\penalty0
  2380--2408, 2025.
\newblock \doi{10.1214/25-AOS2519}.

\bibitem[Faltys et~al.(2021)Faltys, Zimmermann, Lyu, H{\"u}ser, Hyland,
  R{\"a}tsch, and Merz]{faltys2021hirid}
Martin Faltys, Marc Zimmermann, Xinrui Lyu, Matthias H{\"u}ser, Stephanie~L
  Hyland, Gunnar R{\"a}tsch, and Tobias~M Merz.
\newblock {HiRID}, a high time-resolution {ICU} dataset (version 1.1.1), 2021.
\newblock \url{https://physionet.org/content/hirid/1.1.1/}.

\bibitem[Gunsilius(2023)]{gunsilius2023distributional}
Florian~F Gunsilius.
\newblock Distributional synthetic controls.
\newblock \emph{Econometrica}, 91\penalty0 (3):\penalty0 1105--1117, 2023.

\bibitem[Hardy et~al.(1952)Hardy, Littlewood, and
  P{\'o}lya]{hardy1952inequalities}
Godfrey~Harold Hardy, John~Edensor Littlewood, and George P{\'o}lya.
\newblock \emph{Inequalities}.
\newblock Cambridge university press, Cambridge, UK, 1952.

\bibitem[Hyland et~al.(2020)Hyland, Faltys, H{\"u}ser, Lyu, Gumbsch, Esteban,
  Bock, Horn, Moor, Rieck, Zimmermann, Bodenham, Borgwardt, R{\"a}tsch, and
  Merz]{hyland2020early}
Stephanie~L Hyland, Martin Faltys, Matthias H{\"u}ser, Xinrui Lyu, Thomas
  Gumbsch, Crist{\'o}bal Esteban, Christian Bock, Max Horn, Michael Moor,
  Bastian Rieck, Marc Zimmermann, Dean Bodenham, Karsten Borgwardt, Gunnar
  R{\"a}tsch, and Tobias~M Merz.
\newblock Early prediction of circulatory failure in the intensive care unit
  using machine learning.
\newblock \emph{Nature Medicine}, 26\penalty0 (3):\penalty0 364--373, 2020.

\bibitem[Khan et~al.(2017)Khan, Mansha, Khan, and
  Bashir]{khan2017collaborative}
Basit~Mehmood Khan, Asim Mansha, Farhan~Hassan Khan, and Saba Bashir.
\newblock Collaborative filtering based online recommendation systems: A
  survey.
\newblock In \emph{2017 International Conference on Information and
  Communication Technologies (ICICT)}, pages 125--130. IEEE, 2017.

\bibitem[Kolouri et~al.(2017)Kolouri, Park, Thorpe, Slepcev, and
  Rohde]{kolouri2017optimal}
Soheil Kolouri, Se~Rim Park, Matthew Thorpe, Dejan Slepcev, and Gustavo~K
  Rohde.
\newblock Optimal mass transport: Signal processing and machine-learning
  applications.
\newblock \emph{IEEE signal processing magazine}, 34\penalty0 (4):\penalty0
  43--59, 2017.

\bibitem[Kolouri et~al.(2019)Kolouri, Nadjahi, Simsekli, Badeau, and
  Rohde]{kolouri2019generalized}
Soheil Kolouri, Kimia Nadjahi, Umut Simsekli, Roland Badeau, and Gustavo Rohde.
\newblock Generalized sliced wasserstein distances.
\newblock \emph{Advances in neural information processing systems}, 32, 2019.

\bibitem[Le~Gouic et~al.(2023)Le~Gouic, Paris, Rigollet, and
  Stromme]{le2022fast}
Thibaut Le~Gouic, Quentin Paris, Philippe Rigollet, and Austin~J Stromme.
\newblock Fast convergence of empirical barycenters in alexandrov spaces and
  the wasserstein space.
\newblock \emph{Journal of the European Mathematical Society}, 25\penalty0
  (6):\penalty0 2229--2250, 2023.

\bibitem[Li et~al.(2019)Li, Shah, Song, and Yu]{li2019}
Yihua Li, Devavrat Shah, Dogyoon Song, and Christina~Lee Yu.
\newblock Nearest neighbors for matrix estimation interpreted as blind
  regression for latent variable model.
\newblock \emph{IEEE Transactions on Information Theory}, 66\penalty0
  (3):\penalty0 1760--1784, 2019.

\bibitem[Marino and Sutin(2012)]{marino2012icu}
Paul~L Marino and Kenneth~M Sutin.
\newblock \emph{The ICU book}.
\newblock Lippincott williams \& wilkins, 2012.

\bibitem[Martinet et~al.(2022)Martinet, Strzalkowski, and
  Engelhardt]{martinet2022variance}
Guillaume~G Martinet, Alexander Strzalkowski, and Barbara Engelhardt.
\newblock Variance minimization in the wasserstein space for invariant causal
  prediction.
\newblock In \emph{International Conference on Artificial Intelligence and
  Statistics}, pages 8803--8851. PMLR, 2022.

\bibitem[Monge(1781)]{monge1781memoire}
Gaspard Monge.
\newblock \emph{M{\'e}moire sur la th{\'e}orie des d{\'e}blais et des
  remblais}.
\newblock De l'Imprimerie Royale, Paris, France, 1781.

\bibitem[Nirenberg(1959)]{nirenberg1959elliptic}
Louis Nirenberg.
\newblock On elliptic partial differential equations.
\newblock \emph{Annali della Scuola Normale Superiore di Pisa-Scienze Fisiche e
  Matematiche}, 13\penalty0 (2):\penalty0 115--162, 1959.

\bibitem[Papp and Sherlock(2022)]{papp2022bounds}
Tam{\'a}s~P. Papp and Chris Sherlock.
\newblock Centered plug-in estimation of {W}asserstein distances.
\newblock \emph{arXiv preprint arXiv:2203.11627}, 2022.

\bibitem[Ross(1995)]{ross1995stochastic}
Sheldon~M Ross.
\newblock \emph{Stochastic processes}.
\newblock John Wiley \& Sons, New York, NY, 2 edition, 1995.

\bibitem[Shorack and Wellner(2009)]{shorack2009empirical}
Galen~R Shorack and Jon~A Wellner.
\newblock \emph{Empirical processes with applications to statistics}.
\newblock SIAM, Philadelphia, PA, 2009.

\bibitem[Thorpe(2018)]{thorpe2018introduction}
Matthew Thorpe.
\newblock Introduction to optimal transport.
\newblock \emph{Notes of Course at University of Cambridge}, 2018.

\bibitem[Vershynin(2018)]{vershynin2020high}
Roman Vershynin.
\newblock \emph{High-Dimensional Probability: An Introduction with Applications
  in Data Science}, volume~47 of \emph{Cambridge Series in Statistical and
  Probabilistic Mathematics}.
\newblock Cambridge University Press, Cambridge, UK, 2018.

\bibitem[Wasserman(2017)]{wasserman2017optimal}
L~Wasserman.
\newblock Optimal transport and wasserstein distance, 2017.

\end{thebibliography}
